%% file: root.tex
\documentclass[lettersize,journal]{IEEEtran}
\usepackage{amsmath,amsfonts}
\usepackage{amssymb}
\usepackage{amsthm}
\usepackage{algorithm,algpseudocodex}
\usepackage{array}
\usepackage[caption=false, font=footnotesize,labelfont=sf,textfont=sf, farskip=0pt]{subfig}
\usepackage{textcomp}
\usepackage{stfloats}
\usepackage{url}
\usepackage{verbatim}
\usepackage{graphicx}
\usepackage{cite}
\usepackage[nocomma]{optidef}
\usepackage{setspace}
\usepackage{tcolorbox}
\usepackage{xcolor}
\usepackage{soul}
\usepackage{enumitem}
\usepackage[acronym]{glossaries}
\usepackage[framemethod=tikz]{mdframed}
\usepackage{pifont}
\usepackage{tabularx}
\usepackage[normalem]{ulem}
\usepackage{cancel}
\usepackage{capt-of}
\usepackage[T1]{fontenc}
\usepackage[utf8]{inputenc}
\usepackage{microtype}

\input{preambles}

\makeglossaries

\AddToHook{env/figure/begin}{%
	\DeclareRobustCommand\footnote[1]{%
		\footnotemark
		\expanded{\AddToHookNext{env/figure/after}%
			{\noexpand\setcounter{footnote}{\thefootnote}%
				\noexpand\footnotetext\unexpanded{{#1}}}}}}

\begin{document}

	\title{Semi-Infinite Programming for Collision-Avoidance in Optimal and Model Predictive Control}

	\author{Yunfan Gao$^{1, 2*}$, Florian Messerer$^{2}$, Niels van Duijkeren$^{1, 3}$, Rashmi Dabir$^{1, 4}$, Moritz Diehl$^{2,5}$
		\thanks{$^{1}$~Robert Bosch GmbH, Corporate Research, Stuttgart, Germany.
		$^{2}$~Department of Microsystems Engineering (IMTEK), University of Freiburg, Germany.
		$^{3}$~Presently with Fernride GmbH, Munich, Germany.
        $^{4}$~Presently with Institute of Automatic Control, RWTH Aachen University, Germany.
        $^{5}$~Department of Mathematics, University of Freiburg, Germany.}
        \thanks{* Corresponding author. Email: yunfan.gao@imtek.uni-freiburg.de}
		\thanks{The research that led to this paper was funded by Robert Bosch GmbH.
		This work was also supported by DFG via projects 504452366 (SPP 2364) and 525018088, by BMWK via 03EI4057A and 03EN3054B, and by the EU via ELO-X 953348.}
		\thanks{Source code is available at https://doi.org/10.5281/zenodo.19116921.}
	}

	% \markboth{IEEE Transactions on Robotics}%

	% FIXME:

	\maketitle

	\begin{abstract}
		% 1200 characters
		This paper presents a novel approach for collision avoidance in optimal and model predictive control,
			in which the environment is represented by a large number of points
			and the robot as a union of padded polygons.
		The conditions that none of the points shall collide with the robot can be written in terms of an infinite number of constraints per obstacle point.
		We show that the resulting semi-infinite programming (SIP) optimal control problem (OCP) can be efficiently tackled through a combination of two methods:
			local reduction and an external active-set method.
		Specifically, this involves iteratively identifying the closest point obstacles,
		determining the lower-level distance minimizer among all feasible robot shape parameters,
		and solving the upper-level finitely-constrained subproblems.

		In addition,
		    this paper addresses robust collision avoidance in the presence of ellipsoidal state uncertainties.
		Enforcing constraint satisfaction over all possible uncertainty realizations extends the dimension of constraint infiniteness.
		The infinitely many constraints arising from translational uncertainty are handled by local reduction together with the robot shape parameterization,
		    while rotational uncertainty is addressed via a backoff reformulation.

		A controller implemented based on the proposed method
		    is demonstrated on a real-world robot running at 20\,Hz,
		    enabling fast and collision-free navigation in tight spaces.
		An application to 3D collision avoidance is also demonstrated in simulation.
	\end{abstract}


	\input{sections/1-introduction.tex}
    \input{sections/1-2-relatedWork.tex}
	\input{sections/2-problem.tex}
	\input{sections/3-preliminary.tex}
	\input{sections/4-methods.tex}

	\input{sections/5-mobileRobot.tex}

	\input{sections/6-results.tex}

	\input{sections/7-carSeat.tex}

    \section{Conclusions}
    This paper presents an \gls{sip} formulation for collision-avoidance \glspl{ocp}.
    The nominal \gls{ocp} is efficiently solved without any approximations using the local reduction and external active-set method.
    The robust \gls{ocp}, through the approximate reformulation of the robust collision-avoidance constraints and the zero-order update of the uncertainty trajectories,
        is solved in approximation with only a slight increase in the computational cost compared to the nominal OCP.
    The nominal and robust controllers implemented based on the proposed method enable a real-world differential-drive robot to navigate tight spaces,
	    and the enhanced safety guarantee provided by the robust controller is confirmed in the experiments.
    The proposed method is also demonstrated on a 3D car-seat placement task.
    An important direction for future work is to extend the approach to continuous-time collision avoidance
		and to develop a benchmark for collision-avoidance methods with standardized interfaces and well-designed experimental settings.

	% \section*{Acknowledgments}
	% This should be a simple paragraph before the References to thank those individuals and institutions who have supported your work on this article.

    \vspace{-0.3\baselineskip}
	\bibliographystyle{IEEEtran}
	\bibliography{refs.bib} % Entries are in the refs.bib file

	\input{sections/appendix-parameters.tex}

	\begin{IEEEbiography}
	[{\includegraphics[width=1in,height=1.25in,clip,keepaspectratio]{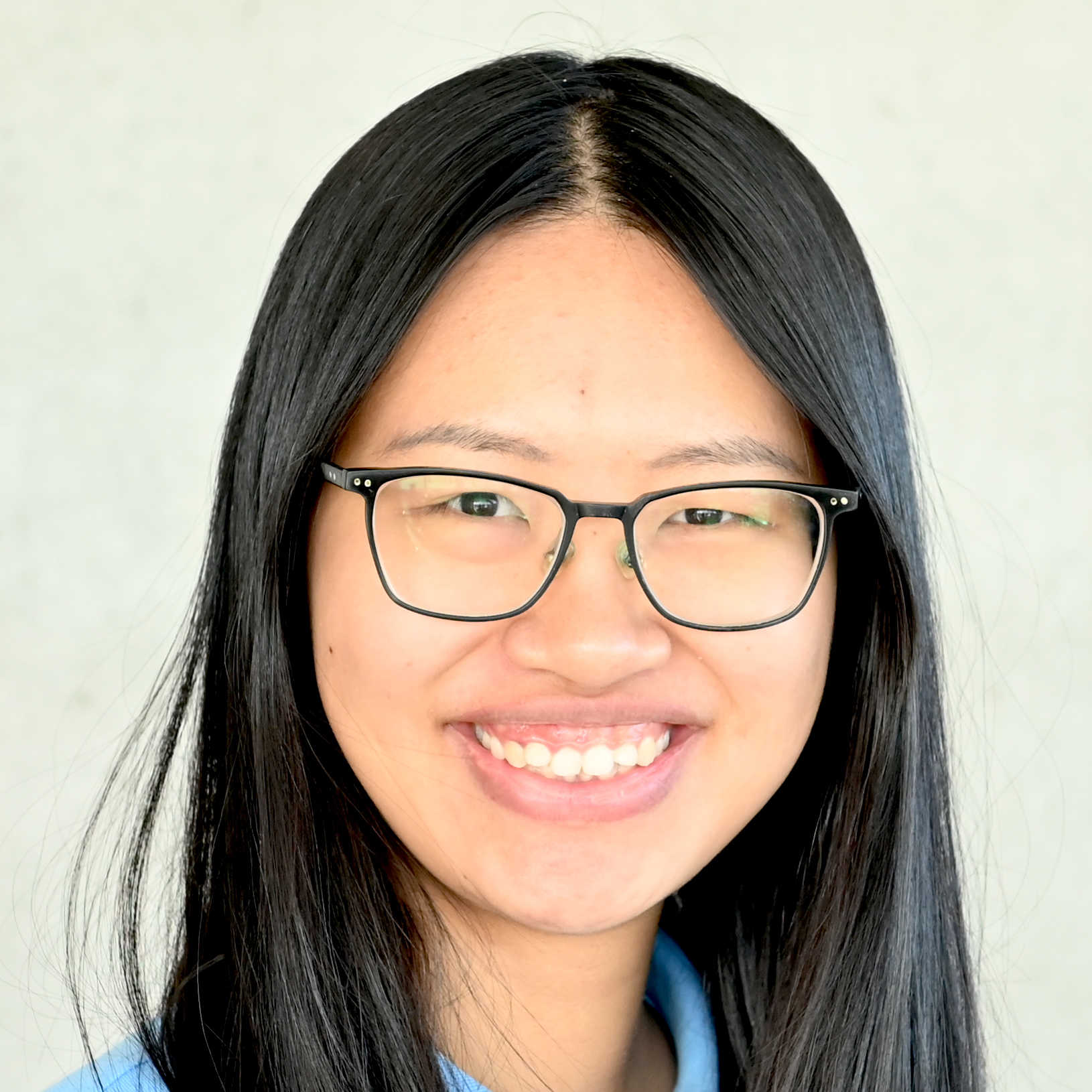}}]
	{Yunfan Gao} completed her master's degree in Robotics, Systems, and Control at ETH Zurich, Switzerland, in 2022.
	Since March 2022, she is a PhD student at the Systems Control and Optimization Laboratory, University of Freiburg, under the academic supervision of Prof. Moritz Diehl.
	Simultaneously, until September 2025, she was an industrial PhD student at Bosch Corporate Research, under the industrial supervision of Dr. Niels van Duijkeren.
	Her research interests include motion planning and optimal control, with a current focus on robust collision avoidance for mobile robots operating under uncertainty.
	\end{IEEEbiography}

    \vspace{-1.5\baselineskip}
	\begin{IEEEbiography}
		[{\includegraphics[width=1in, height=1.25in, clip,keepaspectratio]{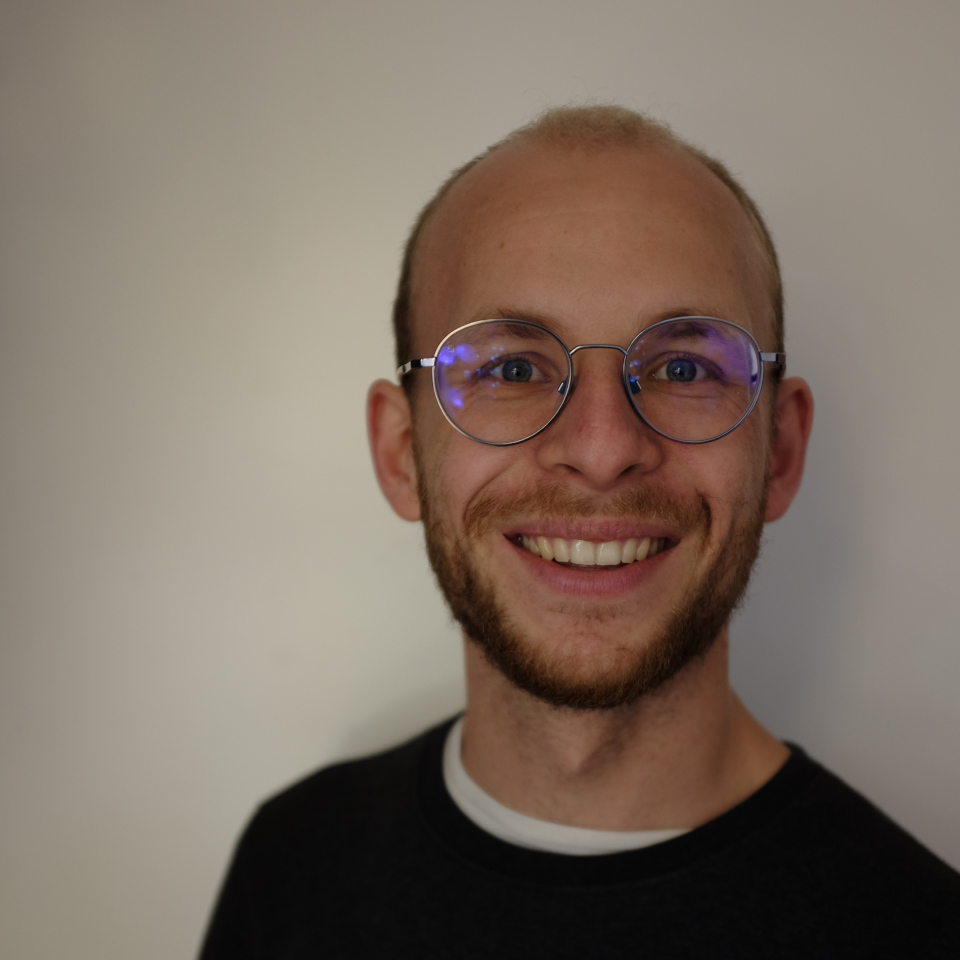}}]
		{Florian Messerer} received a B.Sc. degree in Microsystems Engineering from the University of Freiburg, Germany, in 2016 and an M.Sc. degree in Mathematical Engineering from KU Leuven, Belgium, in 2018.
		He is currently pursuing a PhD degree at the Systems Control and Optimization Laboratory, Department of Microsystems Engineering, University of Freiburg, under the supervision of Prof. Moritz Diehl.
		His research interests include numerical optimization and model predictive control, with a focus on optimal control under uncertainty.
	\end{IEEEbiography}

    \vspace{-1.5\baselineskip}
	\begin{IEEEbiography}
		[{\includegraphics[width=1in,height=1.25in,clip,keepaspectratio]{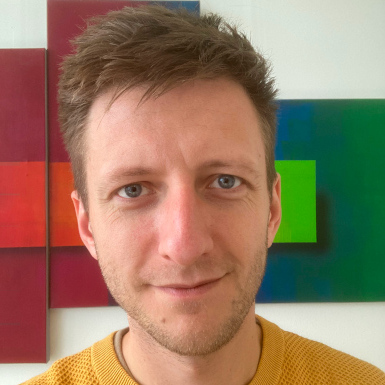}}]{Niels van Duijkeren}
	received his M.Sc. degree in Systems and Control from Delft University of Technology, The Netherlands and his Ph.D. degree from the Motion Estimation Control and Optimization group, Mechanical Engineering, KU Leuven, Belgium.
	His former research primarily focuses on collission-free geometric motion control for robot manipulators and mobile robots.
	He currently works on data-driven methods for decision-making and control, estimation, robust motion planning in dynamic environments, and system identification for adaptive robot control.
	He has co-authored papers in e.g., T-RO, TAC, CDC, IROS, and RSS on topics spanning model predictive control, optimization methods, and machine learning.
	\end{IEEEbiography}

    \vspace{-1.5\baselineskip}
	\begin{IEEEbiography}
	[{\includegraphics[width=1in,height=1.25in,clip,keepaspectratio]{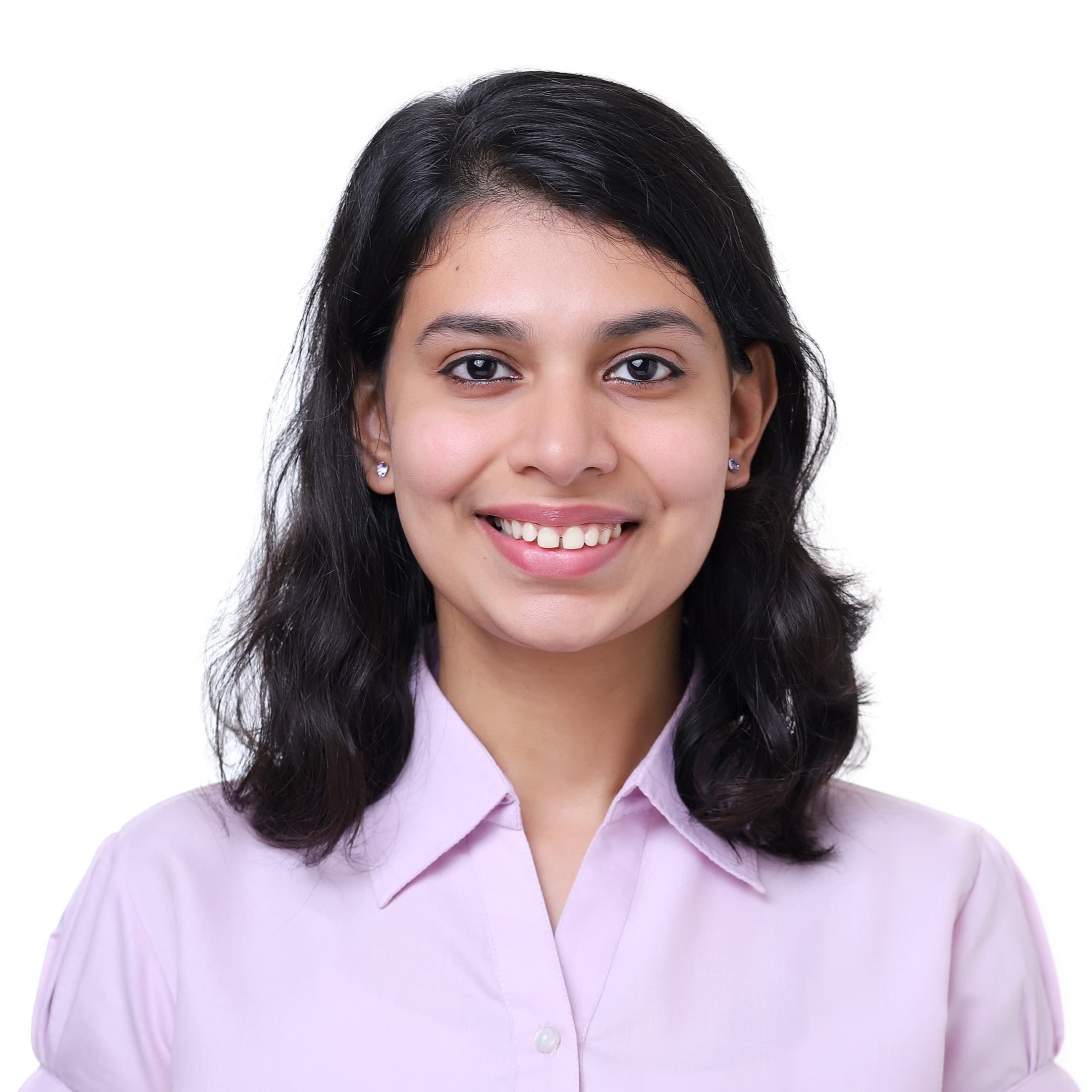}}]{Rashmi Dabir}
	is a research associate at the Institute of Automatic Control, RWTH Aachen University, Germany.
	She completed her M.Sc. degree in Embedded Systems Engineering from University of Freiburg, Germany.
	Her research interests are optimization-based uncertainty-aware trajectory planning and control.
    \end{IEEEbiography}

    \vspace{-1.5\baselineskip}
	\begin{IEEEbiography}
		[{\includegraphics[width=1in,height=1.25in,clip,keepaspectratio]{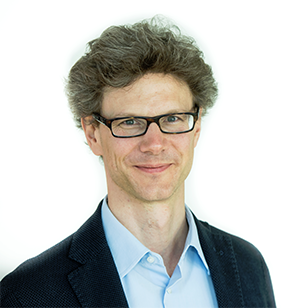}}]
		{Moritz Diehl} is Professor of Systems Control and Optimization at the University of Freiburg, Germany, where he serves as director of the Department of Microsystems Engineering (IMTEK) and as director of the university's Center for Renewable Energy (ZEE).
		He studied physics and mathematics at Heidelberg University, Germany, and Cambridge University, U.K., in 1993-1999, and received the Ph.D. degree from Heidelberg University in 2001.
		From 2006 to 2013, he was a Professor at the Department of Electrical Engineering, KU Leuven, Belgium.
		Since 2013, he is full professor at the Department of Microsystems Engineering in Freiburg, where he is also affiliated with the Department of Mathematics.
		His research interests are in optimization and control, ranging from numerical method development to applications in various branches of engineering, with a focus on embedded real-time implementations and renewable energy systems.
	\end{IEEEbiography}

	\vfill

\end{document}

%% file: preambles.tex
\newcommand{\bbR}{\mathbb{R}}
\newcommand{\bbN}{\mathbb{N}}
\newcommand{\gammaShape}{\gamma_{\mathrm{shp}}}
\newcommand{\gammaSetShape}{\Gamma_{\mathrm{shp}}}
\newcommand{\gammaNoise}{\gamma_{\mathrm{n}}}

\newcommand{\gammaSetSphere}[1]{\mathbb{B}^{#1}}
\newcommand{\gammaRotNoise}{\gamma_{\mathrm{rn}}}
\newcommand{\gammaTransNoise}{\gamma_{\mathrm{tn}}}
\newcommand{\nominalOptVar}{z}
\newcommand{\uncOptVar}{\Sigma}
\newcommand{\llConstrIndex}{l}
\newcommand{\finiteIndex}{i}

\newcommand{\staticObsSet}{\mathcal{O}}
\newcommand{\bigO}{\mathcal{O}}

\newcommand{\closestObs}{\mathrm{CO}}

\newcommand{\mapRToW}{\rho}

\newcommand{\nWorld}{n_{\mathrm{w}}}
\newcommand{\rotMat}{R}

\newcommand{\Compactcdots}{\mathinner{\cdotp\mkern-2mu\cdotp\mkern-2mu\cdotp}}

\newcommand{\ellipsoidE}[2]{\mathcal{E}\left({#1}, {#2}\right)}

\DeclareMathSymbol{\shortminus}{\mathbin}{AMSa}{"39}
\DeclareMathOperator*{\argmin}{arg\,min}

\DeclareMathOperator{\bdSet}{bd}
\DeclareMathOperator{\discSet}{disc}
\DeclareMathOperator{\signedDist}{SD}
\DeclareMathOperator{\vecMat}{vec}
\DeclareMathOperator{\so3Exp}{exp}

\newtheorem{theorem}{Theorem}
\newtheorem{remark}[theorem]{Remark}
\newtheorem{lemma}[theorem]{Lemma}
\newtheorem{definition}[theorem]{Definition}
\newtheorem{corollary}[theorem]{Corollary}

\newacronym{sip}{SIP}{semi-infinite programming}
\newacronym{nlp}{NLP}{nonlinear programming}
\newacronym{mpc}{MPC}{model predictive control}
\newacronym{qp}{QP}{quadratic programming}
\newacronym{sqp}{SQP}{sequential quadratic programming}
\newacronym{zoro}{zoRO}{zero-order robust optimization}
\newacronym{ocp}{OCP}{optimal control problem}
\newacronym{ode}{ODE}{ordinary differential equation}
\newacronym{qcqp}{QCQP}{quadratically constrained quadratic programming}
\newacronym{esdt}{ESDT}{Euclidean signed-distance transform}
\newacronym{lqr}{LQR}{linear quadratic regulator}
\newacronym{mppi}{MPPI}{model predictive path integral}
\newacronym{mip}{MIP}{mixed integer programming}
\newacronym{dwa}{DWA}{dynamic window approach}
\newacronym{teb}{TEB}{timed elastic band}

\algnewcommand{\LeftComment}[1]{\Statex \({\color{gray}\triangleright}\) \textcolor{gray}{\textit{#1}}}

%% file: sections/1-introduction.tex
\section{Introduction}

Optimization-based collision avoidance for robots has been a highly active area of research and engineering.
This paper presents an approach using \gls{sip},
    which is characterized by a finite number of decision variables and an infinite number of constraints,
	typically parameterized by variables ranging in a compact set~\cite{Hettich1993, Gramlich1995, Stein2003, Lopez2007}.

In the context of robotic applications,
    safety conditions require that every point of a robot body should maintain at least a predefined distance from every point within a specified obstacle set.
Formulating these conditions as infinitely many constraints provides great flexibility in robot shape description, enabling a precise modeling of robot shapes.
Compared to collision penalty methods~\cite{Ratliff2009, Kalakrishnan2011},
    the constraint formulation allows the robot to operate in closer proximity to obstacles while guaranteeing satisfaction of the safety conditions.
While the resulting \gls{sip} problem may appear challenging to solve,
    \gls{sip} problems can, under certain regularity assumptions, be efficiently solved through a sequence of \gls{nlp} problems~\cite{Hettich1993}.

\begin{figure}[t]
	\centering
	\includegraphics[ width=0.88\linewidth]{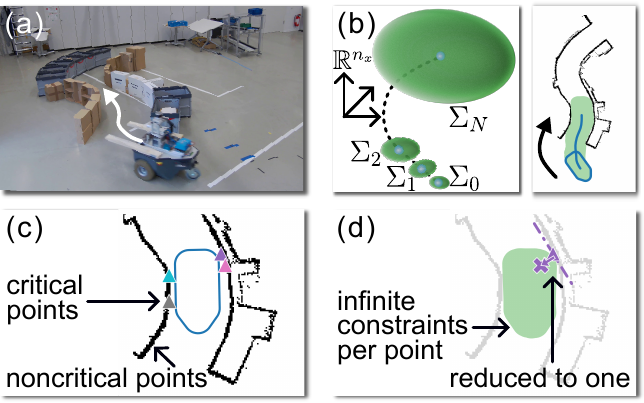}
	\caption{
	Method overview:
	(a)~The robot is driven by an SIP-based robust MPC controller running at 20\,Hz.
	(b)~State uncertainty sets are modeled as ellipsoids and are updated in a zero-order scheme.
	The green tube on the right shows the union of the occupied space of the uncertain robot over all discrete time steps.
	(c)~The environment is modeled with sampled points over the obstacle surfaces (black points).
	For each time step~$k$, a subset of critical points is maintained (colored triangles).
	(d)~The uncertainties and the robot-shape parameterization lead to infinitely many constraints, which are locally reduced to one linearized constraint per critical point, based on the lower-level maximizer (colored cross).
	The algorithm iterates through the steps in (b)--(d).}
	\label{fig:method-overview}
\end{figure}

With respect to obstacles,
    commonly-used shape descriptions include unions of convex shapes~\cite{Schulman2014, Zimmermann2022, Zhang2021, Dietz2023, Thirugnanam2022a, Tracy2023} and the \gls{esdt}~\cite{Ratliff2009, Oleynikova2016}.
\mbox{TrajOpt}~\cite{Schulman2014} identifies the signed distance minimizer between convex shapes
    and uses the normalized vector between the closest points to approximate the gradient of the collision-avoidance constraints.
To address the numerical issue that in general the signed distance minimizer is not a continuous function of the robot pose,
    the authors of~\cite{Zimmermann2022} add an additional regularization term to the objective function of the distance minimization problem.
Alternatively, collision avoidance between convex objects can be enforced by separating-plane constraints\cite{Zhang2021, Dietz2023},~\cite[Chapter 8]{Boyd2004},
    which require additional optimization variables and constraints that scale with the number of obstacles.
The \gls{esdt} encodes the minimum distance to all obstacles in one scalar function,
    allowing collision-avoidance constraints to be imposed by a single inequality against a prescribed threshold.
However, it suffers from nondifferentiability at the Voronoi diagram, where multiple obstacles are equidistant.

This paper models the robot as a (union of) padded polygon(s) and represents the quasi-static environment as a set of point obstacles.
Unlike the convex-shape environment representation,
    this environment representation introduces minimal approximation, especially for environments with many nonconvex obstacles.
Additionally, the representation can be obtained directly from LiDAR point clouds, thereby simplifying the perception module for environment segmentation into shape primitives and facilitating faster online feedback.
More importantly,
    this representation ensures the satisfaction of certain regularity conditions required for solving \gls{sip} problems by local reduction~\cite{Gramlich1995},
    which are not satisfied when representing environments by unions of general convex shapes or using the \gls{esdt}.
The infinitely constrained problem can thereby be efficiently solved
    by iteratively solving finitely constrained subproblems,
    in which the constraint gradients are computed exactly in a straightforward manner.
Although modeling the environment by point clouds results in a considerable number of constraints on individual obstacles,
    our proposed approach employs an external active set method~\cite{Chung2009} and iteratively solves subproblems that only consider selected critical point obstacles.
Similar methods are also explored \mbox{in~\cite{Hauser2021, Zhang2025, Liang2025}}.

In practical applications,
    obstacle positions may contain uncertainties and change over time.
The realization of these uncertainties is often only measured when the robot reaches a position where its sensors have a suitable point-of-view.
Feedback control provides an appropriate solution for such challenges through continuous update of the control policy based on real-time updated point clouds.
While the previously discussed papers mostly focus on path planning and trajectory planning,
this paper addresses collision avoidance in the setting of optimal control and \gls{mpc}.

While the planning and control problems share the objective of collision-free navigation through narrow spaces,
    the control problems have distinct features.
In our context, incorporating an accurate dynamics model in the \gls{ocp}, capturing effects that kinematic models cannot represent, such as overshoot, facilitates safe high-speed navigation.
Moreover, \gls{mpc} requires solving \glspl{ocp} within a short time frame---often on the order of tens of milliseconds.
Frequent updates of the control inputs allow the robot to adapt to the changes of the environment and to the disturbances.
Although this paper enforces collision avoidance only at discrete time grids,
   the \glspl{ocp} employ a comparatively finer discretization than typical planning problems, providing a good approximation to continuous-time collision avoidance.

Up to this point, the focus has been on nominal collision avoidance under the assumption of perfect modeling and ideal state estimation.
Real-world applications, however, are subject to disturbances and uncertainties.
As optimal solutions are often at the edge of constraints,
   even small disturbances can cause constraint violation.
This impairs safety, which is a particularly critical concern in robotic applications.

In addition to nominal constraint satisfaction,
    this paper addresses constraint robustification in the presence of ellipsoidal state uncertainties.
Compared to precomputed robust invariant sets~\cite{Gao2014} and scalar tubes~\cite{Koehler2021},
    modeling the uncertainties as ellipsoidal sets and solving the uncertainty trajectories jointly with the nominal trajectories in the \gls{ocp} provides a more accurate and less conservative modeling of the state uncertainties, especially for nonlinear systems.
The introduced computational burden is relieved by adopting the \gls{zoro} method~\cite{Feng2020Adjoint, Zanelli2021zoRO} and updating the uncertainty trajectories in a zero-order fashion.

To guarantee constraint satisfaction across all possible uncertainty realizations within the propagated ellipsoidal uncertainty sets, we extend the \gls{sip} formulation to include state uncertainties.
While the extension to robust constraint satisfaction is straightforward for our robot shape representation and constraint formulation,
    it can be challenging for other robot shape representations, such as those using the \gls{esdt}~\cite{Hauser2021}.
Unlike prior SIP-based methods that address either shape or state uncertainty alone, our approach accounts for both simultaneously.
Translational uncertainty is addressed via local reduction together with the shape parameterization, and rotational uncertainty is incorporated through a backoff bounding its greatest influence on the constraint values.

An implementation of the proposed method achieves real-time feasibility on compute hardware typically found in industrial mobile robots.
Experiments conducted on a real mobile robot demonstrate the efficacy of the approach in several illustrative settings.
The results show that the new algorithm and implementations lead to a controller that successfully navigates through corridors formerly very troublesome to achieve.
A demonstration of a 3D car-seat placement task in simulation illustrates the applicability of the proposed approach to 3D collision avoidance,
    and to complex shapes.

\vspace{-3pt}
\subsection{Contributions}
\label{sec:intro-contributions}
The main contributions of this paper are as follows:
\begin{enumerate}[label=(\alph*)]
	\item An \gls{sip}-based formulation for collision-free optimal control and \gls{mpc}.
	The formulation features accurate robot modeling and an easily-adaptable point-cloud environment representation, and can be applied to various tasks such as navigation and manipulation.
	\item An efficient method for solving the nominal \gls{sip}-\gls{ocp} without further approximations.
	\item An extension of (a) and (b) to a robust setting in which the state uncertainty leads to an additional dimension of constraint infiniteness.
	\item An efficient open-source implementation of the algorithm enabling an autonomous mobile robot to navigate through narrow spaces in real-world experiments.
\end{enumerate}

\subsection{Notation}
Let $\mathcal{B}$ and $\mathcal{D}$ be two sets.
Their Minkowski sum is given by $\mathcal{B} \oplus \mathcal{D} := \left\lbrace b + d \mid b\in \mathcal{B}, d \in  \mathcal{D}  \right\rbrace $.
Let
\begin{equation}
	\gammaSetSphere{n} := \left\lbrace \tau \left\vert \tau \in \bbR^{n}, \left\| \tau\right\|_2 \leq 1 \right.\right\rbrace
	\label{eq:gamma-set-sphere-def}
\end{equation}
denote a unit ball of dimension $n$.
An ellipsoidal set in $\bbR^{n}$ is defined by
\begin{equation}
	\ellipsoidE{b}{M} := \left\lbrace \left.M^{\frac{1}{2}}\tau + b \right\vert \tau \in \gammaSetSphere{n} \right\rbrace,
	\label{eq:ellipsoid-def}
\end{equation}
where $b \in \bbR^{n} $ is the center of the ellipsoid and $M \in\mathbb{S}^{n}_{+}$ a positive semi-definite matrix.
The matrix $M^{\frac{1}{2}}\in\mathbb{S}^{n}_{+}$ denotes the symmetric square-root of $M$.
For a matrix $M$,
the stacking of its columns, i.e., the vectorization, is given by $ \vecMat(M).$
Let $\left(\tau_1, \Compactcdots, \tau_N \right)$ denote the concatenation of
$N$ vectors $\tau_1, \Compactcdots, \tau_N$.
The block diagonal matrix of $N$ matrices $M_1, \Compactcdots, M_N$ is denoted by
$\mathrm{blk\hbox{-}diag} \left( M_1, \Compactcdots, M_N \right)$.
The Cartesian product of the two sets is given by
$\mathcal{B} \times \mathcal{D} :=
\left\lbrace \left( b, d \right) \mid b\in \mathcal{B}, d \in  \mathcal{D}  \right\rbrace $.

The $l_2$ norm and the $l_{\infty}$ norm of a vector $\tau\in\bbR^{n_{\tau}}$ are denoted by $\left\| \tau \right\|_2 $ and $\left\| \tau \right\|_{\infty}$, respectively.
The quadratic norm of a vector $\tau$ with respect to a positive semidefinite matrix $Q$ is denoted by $\left\| \tau \right\|_Q^2 := \tau^{\top} Q \tau $.
Consider a function $g:\bbR^{n_{\tau}}\to\bbR^{n_m}, \tau \mapsto g\left( \tau\right) $.
We define the Jacobian matrix $\frac{\partial g }{\partial \tau} \!\left( \tau\right) $ as a matrix in $\bbR^{n_m \times n_{\tau}}$.
For scalar functions ($n_m=1$), we denote the gradient vector by $\nabla g(\tau) \in\bbR^{n_{\tau}}$.
A set of natural numbers possibly containing 0 in the interval
$[b, \tau]$ is denoted by $\bbN_{\left[ b, \tau\right] }$.

{
	Let $\text{SO}\!\left(n\right)$ denote the special orthogonal group in $n$ dimensions.
	For a vector $\tau\in\bbR^{3}$, $[\tau]_{\times}$ denotes the corresponding skew-symmetric matrix
	\begingroup	\setlength{\arraycolsep}{2pt} \renewcommand*{\arraystretch}{0.9}
	$[\tau]_{\times} = -[\tau]_{\times}^\top = \begin{bmatrix}
		0 & -\tau_z & \tau_y \\
		\tau_z & 0 & -\tau_x \\
		-\tau_y & \tau_x & 0
	\end{bmatrix}$\endgroup;
    for $\theta \in \bbR$, $[\theta]_{\times}$ is defined as $\begin{bmatrix} 0 & -\theta \\ \theta & 0 \end{bmatrix}$.
	The exponential map that transfers elements of the Lie algebra to the Lie group is denoted by $\so3Exp([\cdot]_{\times})$
	with $\so3Exp(\cdot)$ being the matrix exponential.
}
The key symbols used in the paper are listed in Table~\ref{tab:symbols}.

\begin{table}[ht]
	\vspace{-10pt}
	\centering
	\caption{Key symbols used in this paper.}
	\label{tab:symbols}
	\vspace{-2pt}
	\begin{tabular}{@{}l l r@{}}
		\hline \hline
		Symbol & Description & Introd. in \\
		\hline \hline
		$x$ & Nominal system state &  \eqref{eq:intro-robot-system-ode}\\
		$\tilde{x}$ & Disturbed system state &  \eqref{eq:disturbed-system-discrete-time-dynamics}\\
		$u$ & Control input &  \eqref{eq:intro-robot-system-ode}\\
		$z$ & Upper-level opt. var. & \eqref{eq:general-SIP-def} \\
		$\gamma$ & Lower-level opt. var. & \eqref{eq:general-SIP-def}\\
		$\Gamma$ & Feasible set of $\gamma$ & \eqref{eq:general-SIP-def}\\
		$\gammaShape$ & Robot shape param. & \eqref{eq:polygon-def}\\
		$h(z, \gamma^*(z))$ & Upper-level constr. & \eqref{eq:general-reduction}\\
		$\check{h}(z)$ & Linearized constr. & \eqref{eq:general-linerized-constr-simplified}\\
		$\mathcal{I}$ & Finite index set & \eqref{eq:general-NLP-many-constr-def} \\
		$\finiteIndex$ & Index of finite index set & \eqref{eq:general-NLP-many-constr-def}\\
		$j$ & Iteration index & \\
		\hline \hline
	\end{tabular}
\end{table}

%% file: sections/1-2-relatedWork.tex
\subsection{Related Work}
\label{sec:related-work}
This subsection reviews optimization-based collision avoidance, outlines sampling-based and learning-based methods for robot control, as well as robust collision avoidance and \gls{sip}-based methods for safety-critical systems.

\paragraph{Constrained-Based Collision Avoidance}
Collision-free conditions between convex shape primitives can be formulated as separating-plane constraints, derived using the strong duality theorem~\cite{Zhang2021},~\cite[Chapter 8]{Boyd2004}.
For objects described as polyhedra,
    the authors of~\cite{Richards2002} present a \gls{mip} formulation.
When the shape primitives are ellipsoids,
    the collision-free condition can be formulated using a parametric over-approximation of the Minkowski sum of the ellipsoids~\cite{Gao2024}.
Methods such as~\cite{Marcucci2023, Li2024, Deits2015} construct polytopic safety corridors and require robots to stay in these regions.
Although such safety corridors result in convex feasible regions, these inner approximations may eliminate feasible solutions.
A method for computing differentiable collision-free parametric corridors is presented in~\cite{Arrizabalaga2024}.
For environments with obstacles of unstructured shapes,
   the \gls{esdt}, which provides minimum distance information to all obstacles,
   is widely used for collision avoidance~\cite{Hauser2021, Oleynikova2016, Li2020}.

Collision-free trajectory (or path) optimization can be formulated as an \gls{sip} problem.
This formulation is especially useful for continuous-time collision avoidance as the time (or path) parameterization can be treated as one dimension of constraint infiniteness~\cite{Wang2022, Zhang2024a, Hauser2021}.
Representing the environment as a point cloud and the robot using the \gls{esdt},
    the authors of~\cite{Hauser2021} detect the points and time instants with the deepest penetration into the robot and solve subproblems with constraints for these detected, finitely many, cases.
In~\cite{Wang2022}, the authors enforce constraints to guarantee that the point at the minimum distance along the entire path is collision-free.
The minimum-distance condition is formulated by embedding the first-order necessary and second-order sufficient conditions of the lower-level problem into the upper-level problem.
The paper~\cite{Zhang2024a} presents a subdivision-based method.
Iterative subdivision progressively tightens the motion bounds for each interval, producing a locally optimal solution.

\paragraph{Penalty-Based Collision Avoidance}
Including collision penalties in cost functions is also widely used for collision avoidance.
The artificial potential field is proposed in~\cite{Khatib1985}.
In CHOMP~\cite{Ratliff2009} and STOMP~\cite{Kalakrishnan2011},
    the objective function combines an obstacle cost and performance terms that, for example, encourage smooth trajectories.
CHOMP uses a covariant gradient descent method while STOMP employs a derivative-free stochastic optimization to minimize the cost functions.
The \gls{teb} method~\cite{Roesmann2017} optimizes the trajectory with respect to execution time and separation from obstacles, penalizing constraint violations quadratically.
The authors of~\cite{Liang2024} use separating hyperplanes to keep robot links and obstacles apart.
Barrier functions that penalize penetration relative to these planes are added to the objective function.
By leveraging an efficient matrix factorization,
    the Hessian matrix can be computed efficiently.
While unconstrained optimization problems are typically easier to solve,
    setting a large penalty weight can lead to numerical ill-conditioning~\cite[Section 17.1]{Nocedal2006}.
Moreover, smooth penalty functions provide only (soft) approximations of the constraints.
Quadratic penalties, for example, would allow constraint violations~\cite{Han1979},
    and log-barrier functions often lead to unnecessary additional clearances.

\paragraph{Sampling-Based and Learning-Based Control}
Sampling-based methods generate candidate trajectories,
evaluate their feasibility,
and select the optimal sample(s).
A predictive sampling approach based on MuJoCo physics for real-time \gls{mpc} is presented in~\cite{Howell2022}.
The \gls{dwa}~\cite{Fox1997} evaluates feasible velocity samples
and chooses the one maximizing a weighted score based on path alignment, proximity to the goal, and clearance from obstacles.
The \gls{mppi} approach solves finite-horizon \glspl{ocp} using model-based random rollouts~\cite{Williams2016, Williams2017}.
In highly cluttered environments,
the sampling-based methods often require a very large number of random samples to obtain a near-optimal collision-free trajectory~\cite{Williams2017}.

Many learning-based approaches, particularly end-to-end controllers, demonstrate promising performance~\cite{Chib2024, Chen2024, Yan2024, Bojarski2016, Liu2021}.
In~\cite{Bojarski2016}, a convolutional neural network (CNN) is trained to map raw pixels from a single front-facing camera directly to steering commands.
The authors of~\cite{Liu2021} propose an end-to-end deep learning model for autonomous vehicle control based on raw 3D LiDAR measurements.

\paragraph{Robust Collision Avoidance in Optimal Control}
Tube-based MPC ensures robust constraint satisfaction against disturbances by surrounding the nominal trajectory with a tube that accounts for all possible disturbances within a predefined set.
In\mbox{\cite{Gao2014, Lee2022, Majumdar2017, Singh2017}}, a robust invariant set is precomputed offline to tighten the collision avoidance constraints with the linearization error treated as an additional disturbance source.
Computing disturbance tubes online can reduce the impact of the linearization error,
but it increases the online computational burden.
To address this,
the authors of~\cite{Garimella2018, Gao2023, Zhang2024} alternate between finding the robust tubes and optimizing over the nominal trajectory while ensuring that the states within these robust tubes remain collision-free.

Scenario-based methods sample a finite number of disturbance trajectory realizations, referred to as scenarios,
and impose constraint satisfaction for all sampled scenarios.
In the context of robust collision avoidance,
the authors of~\cite{Batkovic2021} use a scenario-MPC approach to manage obstacle position uncertainties modeled by multi-modal distributions.
Achieving a high confidence of constraint satisfaction often requires a large number of samples~\cite{Calafiore2006}, leading to a substantial computational burden.
To reduce this cost, the authors of~\cite{Groot2021} prune scenarios based on geometric metrics.
Inspired by the methods used in \gls{sip}, the authors of\mbox{\cite{Zagorowska2023}} generate scenarios by iteratively adding interim worst-case scenarios.
Although scenario-based methods can capture diverse distributions, ellipsoidal tubes provide an effective model for the primary uncertainties in this paper (process noise and initial state uncertainties), while avoiding the sampling-related drawbacks.

\paragraph{SIP in Safety-Critical Systems}
\Gls{sip} naturally arises in robust \glspl{ocp}, where constraints must hold for all possible realizations of the modeled uncertainties.
When the disturbance set and the feasible set are both ellipsoidal,
    the infinite constraints can be reformulated as linear matrix inequalities (LMIs)\mbox{\cite{Jia2005}}.
Formulations with state-dependent uncertainties yield generalized \glspl{sip}\mbox{\cite{Wehbeh2025}}.
Enforcing control barrier functions (CBFs) under uncertainties in the control matrices of control-affine systems leads to convex \glspl{sip}, which can be solved efficiently using cutting-plane methods\mbox{\cite{Wei2023}}.
Enforcing safety over a continuous-time horizon also yields infinite constraints\mbox{\cite{Hauser2021}}.
In\mbox{\cite{Zhang2024b}}, a semi-infinitely constrained Markov decision process (MDP) is formulated as a linear \gls{sip} problem.

\subsection{Outline}

This paper is organized as follows.
Section~\ref{sec:prob-statement} presents the problem statement.
In Section~\ref{sec:numerical-methods}, we introduce the numerical foundations of our approach.
Section~\ref{sec:methods-solve-robust-OCP} presents the method for solving the nominal and robust~\glspl{ocp},
    including treatments for infinite constraints, numerous obstacles, and ellipsoidal state uncertainties in robust constraint satisfaction.
The specific method and implementation for optimal control and \gls{mpc} of a mobile robot are presented in Section~\ref{sec:mobile-robot-ocp-mpc},
    followed by numerical evaluations and real-world \gls{mpc} experiments in Section~\ref{sec:results}, with comparison to the state-of-the-art methods.
Section~\ref{sec:car-seat-placement} demonstrates the approach on a car-seat placement task.
The parameter settings and experiment configurations are provided in the Appendix.

%% file: sections/2-problem.tex
\section{Collision-Avoidance Optimal Control}
\label{sec:prob-statement}

This section first describes the robot shape parameterization utilized in this paper.
We then present the collision-avoidance constraint and the \textit{nominal} \gls{ocp} formulation.
Subsequently, we detail the disturbance modeling considered in this paper.
This section concludes with the \textit{robust} \gls{ocp} formulation.

\begin{figure}[t]
	\centering
	\subfloat[Padded polygon in 2D]{\includegraphics[width=0.45\linewidth, trim={0.5cm 1.0cm 0 2.0cm}, clip]{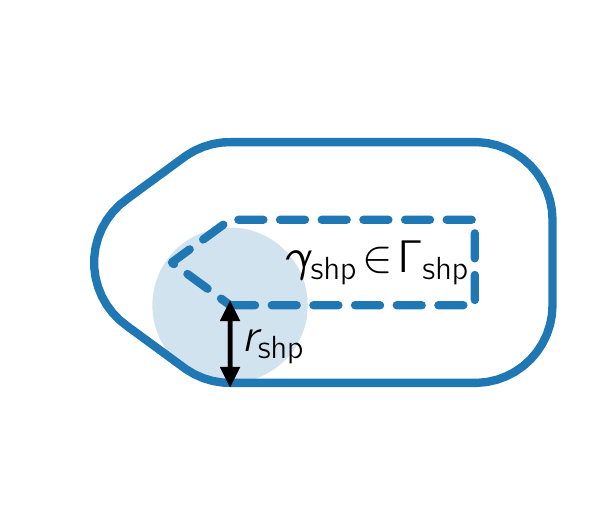}
		\label{fig:robot-dilated-polygon} }
	\hfill
	\subfloat[Padded polyhedron in 3D]{\includegraphics[width=0.45\linewidth, trim={4.6cm 2.6cm 1.7cm 2.8cm}, clip]{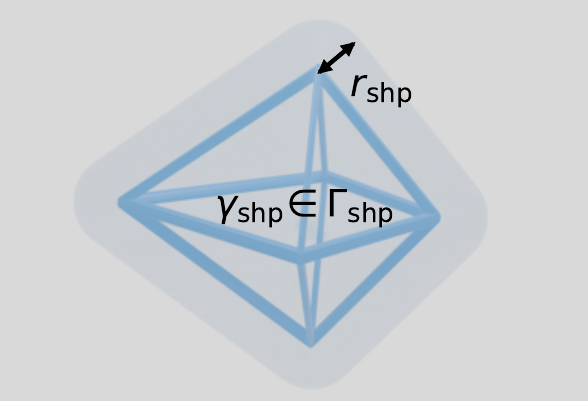}
		\label{fig:robot-capsule}}
	\caption{Robot shapes considered in this paper.}
	\label{fig:robot-shape}
	\vspace{0.3\baselineskip}
\end{figure}

\vspace{-0.2\baselineskip}
\subsection{Robot Shape Parameterization}
\label{sec:robot-shape-description}
\vspace{-0.1\baselineskip}
We model the robot shape as a polygon (or a polyhedron) that is padded by a circle (or a sphere), see Fig.~\ref{fig:robot-dilated-polygon}.
Define the polygon as
\begin{equation}
	\gammaSetShape := \left\lbrace \gammaShape \in \bbR^{\nWorld} \mid A \gammaShape + b \leq 0  \right\rbrace,
	\label{eq:polygon-def}
\end{equation}
where $A \in \bbR^{n_{\mathrm{h}}\times \nWorld}$ and $b \in \bbR^{n_{\mathrm{h}}}$.
Here, $n_{\mathrm{h}}$ denotes the number of half planes,
  and $\nWorld\in\{2, 3\}$ is the dimensionality of the workspace.
The region occupied by the robot is a function of the robot state $x \in \mathbb{R}^{n_x}$.
Let $ p_{\mathrm{c}} : \bbR^{n_x}\rightarrow \bbR^{\nWorld}$ and $\rotMat:\bbR^{n_x} \rightarrow \text{SO}\!\left(\nWorld\right)$ denote the mapping from the state to a translation vector and a rotation matrix, respectively.
The occupied space can then be parameterized by
\begin{equation}
	\left\lbrace p_{\mathrm{c}}(x)  + \rotMat(x) \gammaShape \mid \gammaShape \in \gammaSetShape \right\rbrace  \oplus \left\lbrace p \mid \|p\|_2 \leq r_{\mathrm{shp}}  \right\rbrace,
	\label{eq:polygon-occupied-space}
\end{equation}
where $r_{\mathrm{shp}} > 0$ is the radius of the padding.
Although the formulation in~\eqref{eq:polygon-occupied-space} considers the shape model of one padded polygon,
    the method presented in this paper can be straightforwardly extended to the objects that are a union of multiple padded polygons, as will be demonstrated in Section~\ref{sec:car-seat-placement}.

\subsection{Environment Modeling and Collision Avoidance}
\label{sec:env-modeling}
We represent the environment using a set of points on the surface of larger obstacles.
Let the coordinates of a point obstacle be denoted by $p_{\mathrm{o}} \in \mathbb{R}^{\nWorld}$, and the set of all point obstacles by~$\mathcal{O}$.
The collision-free condition is formulated such that no point obstacle is within a Euclidean distance of $r_{\mathrm{shp}}$ from the polygon (or the polyhedron):
\begin{equation}
	\left\|  p_{\mathrm{c}}(x) + \rotMat(x) \gammaShape - p_{\mathrm{o}} \right\|_2 \geq  r_{\mathrm{shp}}, \forall \gammaShape\! \in \!\gammaSetShape, p_{\mathrm{o}} \!\in \!\mathcal{O}.
\end{equation}

Directly representing obstacles as a set of points leads to a large number of constraints,
    but it also offers several advantages.
First, aside from the discretization errors in sampling the obstacle surface,
    we achieve a precise representation of the obstacle shape.
Second, the points are easily accessible in real-world applications, for example, through LiDAR measurements, and can be easily adjusted to changes in the environment.
Third, as will be presented in Section~\ref{sec:single-state-single-obstacle},
    the resulting lower-level optimization problems can be efficiently solved
    and have the numerical properties that enable us to reduce the infinite-constraint formulations into finite ones.

\subsection{Nominal OCP Formulation}
Consider a \textit{nominal} robot system that is described by an explicit \gls{ode}:
\begin{equation}
	\frac{\mathrm{d}x(t) }{\mathrm{d}t}= f(x(t), u(t)),
	\label{eq:intro-robot-system-ode}
\end{equation}
with state $x(t)\colon\bbR \to \bbR^{n_x}$ and control $u(t)\colon\bbR \to \bbR^{n_u}$.
Both the differential kinematics of the robot and the system dynamics can be modeled by \gls{ode}.
As will be demonstrated in the experiment results,
    incorporating an accurate dynamic model within the~\gls{ocp} facilitates safe trajectory execution for control of robots operating at high speeds.

Consider a prediction horizon of length $T > 0$.
The prediction horizon $\left[0, T\right] $ is split into $N$ (possibly non-uniform) discretization intervals.
Assuming that the control input is a zero-order-hold signal for each discretization interval,
the continuous-time system can be discretized as~\cite{Bock1984}:
\begin{equation}
	{x}_{k+1} = \psi_k({x}_k, u_k), k\in\mathbb{N}_{\left[0 , N\shortminus 1 \right]},
	\label{eq:intro-robot-system-discrete-time}
\end{equation}
where $x_k$ denotes the robot state at time $t_k$.
To model the limitations of the system actuators and states,
state-input constraints in the form of
$ h_{\mathcal{XU}}\!\left(x_k, u_k \right) \leq 0$ and terminal constraints in form of  $h_{\mathcal{X}_N}\left(x_N \right) \leq 0$  are imposed.
For the purpose of our discussion, these constraints do not include the collision-avoidance conditions.

Let ${L}_k$ and ${L}_N$ be the stage cost and terminal cost functions, both of which are twice continuously differentiable.
Let ${\bar{x}}_{0}$ denote the robot state at the beginning of the prediction horizon.
The nominal \gls{ocp} formulation, concerned with \textit{discrete-time} collision-avoidance, is given by
\begin{mdframed}[
	linecolor=black,linewidth=0.8pt,
	frametitlerule=false,
	innertopmargin=-1pt
	]
\begin{mini!}|s|
    {\substack{{x}_0, \Compactcdots, {x}_N,\\
            {u}_0, \Compactcdots, {u}_{N\!\shortminus \!1}}}
    {\sum_{k=0}^{N\!-\!1} {L}_k\left({x}_k, {u}_k\right) + {L}_N\left({x}_N\right)}
    {\label{eq:nominal-ocp-def}}{}
    \addConstraint{{x}_{0}}{={\bar{x}}_{0},}{}
    \addConstraint{{x}_{k+1}}{=\psi_k({x}_k,\! {u}_k),}{  k \!\in \! \mathbb{N}_{\left[0, N\!\shortminus\! 1 \right]}, \label{eq:ocp-state-input-constraints}}{}
    \addConstraint{0}{\geq h_{\mathcal{XU}}\!\left(x_k, u_k \right),}{ k \!\in \! \mathbb{N}_{\left[0, N\!\shortminus\! 1 \right]},
        \label{eq:nominal-ocp-state-input-constraints}}{}
    \addConstraint{0}{\geq h_{\mathcal{X}_N}\!\left(x_N \right), \label{eq:ocp-state-terminal-constraints}  }
    \addConstraint{r_{\mathrm{shp}}}{ \leq \left\| p_{\mathrm{c}}(x_k)  +  \rotMat(x_k) \gammaShape - p_{\mathrm{o}} \right\|_2, \nonumber}{}
    \addConstraint{}{\forall \gammaShape \! \in \! \gammaSetShape,\, p_{\mathrm{o}} \! \in \!\mathcal{O}, \, }{k \! \in \!\mathbb{N}_{\left[1, N\right]}. \label{eq:nominal-ocp-collision-constraints}}
\end{mini!}
\end{mdframed}

Treating each sampled point on the obstacle surface as a separate obstacle leads to a large cardinality of the obstacle set $\staticObsSet$ and, consequently, a large number of constraints on individual point obstacles.
Moreover, resulting from the robot shape parameter $\gammaShape$ ranging continuously within the set $\gammaSetShape$,
    the \gls{ocp}~\eqref{eq:nominal-ocp-def} has an infinite number of constraints, see~\eqref{eq:nominal-ocp-collision-constraints}.
{The \gls{ocp} formulation can be extended to robot shapes as a union of padded polygons (or polyhedra) by imposing~\eqref{eq:nominal-ocp-collision-constraints} for each polygon (or polyhedron).}

\subsection{Disturbance Modeling and Robust \gls{ocp} Formulation}
Suppose the robot system dynamics is subject to disturbance $w_k \in \bbR^{n_x}$:
\begin{equation}
	\tilde{x}_{k+1} = \tilde{\psi}_k(\tilde{x}_k, u_k, w_k), \ k\in\mathbb{N}_{\left[0 , N\shortminus 1 \right]},
	\label{eq:disturbed-system-discrete-time-dynamics}
\end{equation}
where $\tilde{x}_{k}$ denotes the disturbed state at time grid $t_k$.
Note that when ${w_k = 0}$, the formulation of the disturbed dynamics~\eqref{eq:disturbed-system-discrete-time-dynamics} is equivalent to the nominal system dynamics~\eqref{eq:intro-robot-system-discrete-time}.

The disturbance sequence is considered unknown but assumed to be contained in a compact set $(w_0, \Compactcdots, w_{N\shortminus 1}) \in \mathcal{W}$,
leading to a set of possible state trajectories.
In this paper, we are concerned with the case that, at each time point $t_k$, the set is an ellipsoid centered at the nominal state $x_k$:
\begin{equation}
	\tilde{x}_k \in \mathcal{E}\left(x_k, \Sigma_k\right), \ k\in\mathbb{N}_{\left[0 , N\right]},
\end{equation}
where ${\Sigma_k \in\mathbb{S}^{n_x}_{+}}$ denotes the ellipsoid shape matrix.
Consider dynamics of the ellipsoidal uncertainty sets of the form
\begin{subequations}
	\label{eq:intro_dynamics_ellipsoid}
	\begin{alignat}{2}
		x_{k+1} &= \psi_k(x_k, u_k), \\
		\Sigma_{k+1} &= \Phi_k( x_k,  u_k, \Sigma_k), \ && k\in\mathbb{N}_{\left[0 , N\shortminus 1 \right]}, \label{eq:intro_dynamics_sigma}
	\end{alignat}
\end{subequations}
where the dynamics of the ellipsoid center are those of the nominal system.
The method proposed in this paper is, in principle, compatible with all differentiable ellipsoid dynamics of the form~\eqref{eq:intro_dynamics_ellipsoid}.
In the following, we present a specific linearization-based variant~\cite{Nagy2003a, Diehl2006c, Messerer2021} of the shape dynamics~\eqref{eq:intro_dynamics_sigma}, which we use throughout the remainder of this paper. %, Gillis2013
Assuming that the disturbance sequence is contained within a high-dimensional ellipsoid
\begin{equation}
	\mathcal{W} = \mathcal{E}\left(0,  \mathrm{blk\hbox{-}diag} \left( W_0, W_1, \cdots, W_{N\shortminus 1} \right) \right),
	\label{eq:disturbance-seq-ellipsoidal-tube}
\end{equation}
the ellipsoid shape matrices $\Sigma_k$ can be propagated based on a linearization of the dynamics at the  nominal trajectory as
\begin{equation}
	\begin{split}
	\Sigma_0  =  {\bar{\Sigma}}_{0}, \
	\Sigma_{k+1} & = A_k  \Sigma_k {A_k}^{\! \top} + C_k W_k C_k^\top\\
	& =: \Phi_k({x}_k, {u}_k, \Sigma_k),  \ k\in\mathbb{N}_{\left[0 , N\shortminus 1 \right]},
\end{split}
\label{eq:dynamics_ellipsoid_lin}
\end{equation}
where ${\bar{\Sigma}}_{0}$ denotes the state uncertainty matrix at the beginning of the prediction horizon,
    and $A_k$ and $C_k$ are obtained from the sensitivities of the discrete-time system simulation:
\begin{subequations}
	\begin{align}
		A_k & \coloneqq\frac{\partial \tilde{\psi}_k(\tilde{x}_k, u_k, w_k)}{\partial \tilde{x}_k}\left|{}_{(\tilde{x}_k, w_k)=\left( {x}_k, 0 \right) } \right., \\
		C_k & \coloneqq\frac{\partial \tilde{\psi}_k(\tilde{x}_k, u_k, w_k)}{\partial w_k}\left|{}_{(\tilde{x}_k, w_k)=\left( {x}_k, 0 \right) } \right..
	\end{align}
\end{subequations}
While the uncertainty dynamics \eqref{eq:dynamics_ellipsoid_lin} are exact for linear systems,
    for nonlinear systems, a truncation error is introduced.
This error can be bounded and  \eqref{eq:dynamics_ellipsoid_lin} can be modified to obtain a conservative over-approximation~\cite{Koller2018,Houska2011, Leeman2023}.
The discrete-time robust collision-avoidance \gls{ocp} can be stated as

\begin{mdframed}[
	linecolor=black,linewidth=0.8pt,
	frametitlerule=false,
	innertopmargin=-1pt
	]
	\begin{mini!}|s|
		{\substack{{x}_0, \Compactcdots, {x}_N,\\
				{u}_0, \Compactcdots, {u}_{N\!\shortminus \!1}, \\
				\Sigma_{0}, \Compactcdots, \Sigma_{N}}}
		{\sum_{k=0}^{N\!-\!1} {L}_k\left({x}_k, {u}_k\right) + {L}_N\left({x}_N\right)}
		{\label{eq:robust-ocp-def}}{}
		\addConstraint{{x}_{0}}{={\bar{x}}_{0},\ \Sigma_0= {\bar{\Sigma}}_{0},}{}
		\addConstraint{{x}_{k+1}}{=\psi_k({x}_k,\! {u}_k),\  k \!\in \! \mathbb{N}_{\left[0,  N\!\shortminus\! 1 \right]}, \label{eq:nominal-state-dynamics}}{}
		\addConstraint{\Sigma_{k+1}}{=\Phi_k({x}_k,\! {u}_k, \! \Sigma_k),\  k \!\in \!\mathbb{N}_{\left[0, N\!\shortminus\! 1 \right]}, \label{eq:Sigma-function-of-state-trajectory}}{}
		\addConstraint{0}{\geq h_{\mathcal{XU}}\!\left(x_k+\Sigma_k^{\frac{1}{2}}\gammaNoise, u_k \right), \nonumber}{}
		\addConstraint{}{\forall \gammaNoise \in \gammaSetSphere{n_x}, k\in\mathbb{N}_{\left[0,  N\!\shortminus\! 1 \right]}, \label{eq:robust-ocp-stage-constraints}}
		\addConstraint{0}{\geq h_{\mathcal{X}_N}\!\left(x_N \!+\! \Sigma_N^{\frac{1}{2}}\gammaNoise \right), \nonumber   }
		\addConstraint{}{\forall \gammaNoise \in \gammaSetSphere{n_x}, \label{eq:robust-ocp-teminal-constraints} }
		\addConstraint{0}{ \geq r_{\mathrm{shp}} - \Bigl\|  \rotMat\left({x}_k  \! + \! \Sigma_k^{\frac{1}{2}}\gammaNoise\right) \! \gammaShape  \nonumber }
		\addConstraint{}{ \quad  \quad +  p_{\mathrm{c}}\!\left({x}_k \! + \! \Sigma_k^{\frac{1}{2}}\gammaNoise\right) \!- \!p_{\mathrm{o}} \Bigl\|_2, \nonumber }
		\addConstraint{}{\forall \gammaShape \! \in \! \gammaSetShape,\
			\gammaNoise \in \gammaSetSphere{n_x}  , \nonumber}
		\addConstraint{}{p_{\mathrm{o}} \in\staticObsSet, \ k\in\mathbb{N}_{\left[1, N\right]}. \label{eq:robust-ocp-collision-constr}}
	\end{mini!}
\end{mdframed}

In addition to the challenges in solving the nominal \gls{ocp},
    the dimension of the ellipsoid shape matrices $\Sigma_k$ is quadratic in the dimension of the system states, resulting in a substantial increase in computational burden~\cite{Zanelli2021zoRO}.
Moreover, the state uncertainty $\gammaNoise \in \gammaSetSphere{n_x}$ introduces another dimension of constraint infiniteness.

%% file: sections/3-preliminary.tex
\section{Preliminaries on Numerical Methods}
\label{sec:numerical-methods}

The collision-avoidance constraints and the state uncertainties pose great challenges to numerical solvers.
In the following subsections,
we present the local reduction method for tackling infinitely many constraints, the external active-set method for handling a significant number of constraints, and the \gls{zoro} method for efficiently solving robust \glspl{ocp}.

\subsection{Local Reduction Method}
\label{sec:prelim-local-reduction}
\glsreset{sip}
A \gls{sip} problem is an optimization problem with a finite number of optimization variables $z\in\bbR^{n_z}$ and an infinite number of constraints:
\begin{mini!}|s|
	{z\in\bbR^{n_z}}{L(z)}{\label{eq:general-SIP-def}}{}
	\addConstraint{0}{\geq h(z, \gamma),\ \forall \gamma \in \Gamma \label{eq:general-SIP-constr}},
\end{mini!}
where $\gamma \in \bbR^{n_{\gamma}}$, $h:\bbR^{n_z} \times \bbR^{n_{\gamma}} \to \bbR$.
The set $\Gamma$ is an infinite index set and is \textit{independent} of the upper-level optimization variables $z$.

The authors of~\cite{Gramlich1995, Hettich1993, Lopez2007} show that the \gls{sip} problem~\eqref{eq:general-SIP-def}, under some regularity assumptions,
    can be locally reduced to a finite-dimensional programming problem.
Here we briefly describe the assumptions and refer to~\cite{Hettich1993} for details.
Suppose the infinite index set $\Gamma$ is a compact set and is defined via differentiable constraint functions $g_l:\bbR^{n_{\gamma}} \to \bbR$:
\begin{equation}
	\Gamma:=\left\lbrace \gamma \mid g_{\llConstrIndex}(\gamma) \leq 0, \ \llConstrIndex \in\bbN_{[1, n_g]}  \right\rbrace.
\end{equation}
Ensuring $h(z, \gamma) \leq 0$ for all $\gamma \in \Gamma$ is equivalent to ensuring that the maximum of $h(z, \gamma)$ over $\gamma \in \Gamma$ is not greater than zero.
The lower-level optimization problem that seeks to maximize the function $h(z, \gamma)$ with respect to $\gamma$ is given by
\begin{maxi!}|s|
	{\gamma \in \bbR^{n_{\gamma}} }{h({z}, \gamma)}{\label{eq:general-ll-prob-def}}{\mathrm{LL \hbox{-} P}({z}):}
	\addConstraint{ g_{\llConstrIndex}(\gamma)}{\leq 0,\ \llConstrIndex \in\bbN_{[1, n_g]}.}
\end{maxi!}
Let $\lambda_{\mathrm{LL}, \llConstrIndex} \in \bbR$ be the Lagrange multiplier of constraint ${g_\llConstrIndex(\gamma)\leq 0}$
and $\lambda_{\mathrm{LL}} := \begin{bmatrix}
	\lambda_{\mathrm{LL},1} & \lambda_{\mathrm{LL}, 2} & \cdots & \lambda_{\mathrm{LL}, {n_g}}
\end{bmatrix}^{\top}$ be the vector that concatenates all multipliers.
The Lagrangian of \mbox{LL-P$(z)$} is given by
\begin{equation}
	\mathcal{L}_{\mathrm{LL}}\left(\gamma, \lambda_{\mathrm{LL}}; z\right)  := h\left({z}, {\gamma}\right) - \sum_{\llConstrIndex=1}^{n_g}\lambda_{\mathrm{LL} , \llConstrIndex} g_{\llConstrIndex}(\gamma).
	\label{eq:general-Lagrangian}
\end{equation}

\begin{definition}[Nondegenerate local solution]
A solution pair $\left(\acute{\gamma}, \acute{\lambda}_{\mathrm{LL}} \right)$ is called a nondegenerate local solution of $\mathrm{LL \hbox{-} P}(\acute{z})$ if the following conditions are satisfied:
	\begin{enumerate}
		\item the KKT conditions,
		\item the linear independence constraint qualification (LICQ),
		\item the second order sufficient condition (SOSC),
		\item the strict complementary slackness (SCS).
	\end{enumerate}
\end{definition}

Let $\acute{\gamma}$ and the Lagrange multipliers $\acute{\lambda}_{\mathrm{LL}}$ be a nondegenerate local solution of $\mathrm{LL \hbox{-} P}(\acute{z})$.
Then within a neighborhood of $\acute{z}$ and $\acute{\lambda}$,
	there exists a continuously differentiable function ${\gamma}^*\!\left({z}\right)$ with ${\gamma}^*\!\left( \acute{z} \right) = \acute{\gamma}$ and ${\lambda}^*\!\left({z}\right)$ with ${\lambda}^*\!\left( \acute{z} \right) = \acute{\lambda}$.
Assuming that all local solutions are nondegenerate,
    the set of local solutions for a given $\acute{z}$ is finite given that the index set $\Gamma$ is a compact set.
Let $\iota$ denote the index of local solutions and $n_{\iota}$ represent the total number of local solutions.
Within a neighborhood of $\acute{z}$, the \gls{sip} problem~\eqref{eq:general-SIP-def} can be locally replaced by
\begin{mini!}|s|
	{z}{L(z)}{\label{eq:general-reduction}}{}
	\addConstraint{0\geq }{\, h\left(\acute{z}, \gamma^*_{\iota} \left(\acute{z}\right)  \right) + \left. \frac{\mathrm{d} h \! \left( {z}, \gamma_{\iota}^*(z)\right) } {\mathrm{d} z} \right\vert_{z=\acute{z}} \left( z \! - \! \acute{z} \right) , \nonumber }
	\addConstraint{}{ {\iota} \in \bbN_{[1, n_{\iota}]}\label{eq:general-reduction-total-derivatives} }.
\end{mini!}

The following lemma clarifies a property of the derivatives of the function $h\left( z, \gamma_{\iota}^*(z) \right)$.
\begin{lemma}
    The total derivatives $\frac{\mathrm{d} h}{\mathrm{d} z}\left( z, \gamma_{\iota}^*(z) \right)$ are equal to the partial derivatives $\frac{\partial h}{\partial z}\left( z, \gamma_{\iota}^*(z) \right)$.
    \label{lemma:partial-derivatives-equal-to-total}
\end{lemma}

\begin{proof}
For simplicity of notation,
    here we consider one single solution, i.e., one constraint in~\eqref{eq:general-reduction}, and drop the index $\iota$.
At local solutions of LL-P,
   the value of $\lambda_{\mathrm{LL} , \llConstrIndex}^*\left(z\right)  g_{\llConstrIndex}\!\left(\gamma^*\!\left(z\right)\right)$ is a constant zero for all values of $z$ due to the complementarity conditions of the KKT conditions,
   and therefore $\mathcal{L}_{\mathrm{LL}}\!\left( z, \gamma^*(z), \lambda_{\mathrm{LL}}^*(z) \right)\! =\! h\!\left( z, \gamma^*(z) \right)$.
Taking the derivatives with respect to~$z$, we obtain
\begin{subequations}
	\begin{align}
		&\frac{\mathrm{d} h}{\mathrm{d} z} \left( z, \gamma^*(z) \right) = \frac{\mathrm{d} \mathcal{L}_{\mathrm{LL}}}{\mathrm{d} z} \left( z, \gamma^*(z), \lambda_{\mathrm{LL}}^*(z) \right) \\
		% + \frac{\mathrm{d} 0}{\mathrm{d} z} \\
		= &\frac{\partial \mathcal{L}_{\mathrm{LL}}}{\partial z}
		+ \frac{\partial \mathcal{L}_{\mathrm{LL}}}{\partial \gamma} \frac{\partial \gamma^*}{\partial z}
		+ \frac{\partial \mathcal{L}_{\mathrm{LL}}}{\partial  \lambda_{\mathrm{LL}}} \frac{\partial  \lambda_{\mathrm{LL}}^*}{\partial z} \\
        = &\frac{\partial \mathcal{L}_{\mathrm{LL}}}{\partial z}
		+ \frac{\partial \mathcal{L}_{\mathrm{LL}}}{\partial \gamma} \frac{\partial \gamma^*}{\partial z}
		- \sum_{\llConstrIndex=1}^{n_g} g_{\llConstrIndex}\!\left( {\gamma}^*(z) \right) \frac{\partial  \lambda_{\mathrm{LL}, \llConstrIndex}^*}{\partial z}.
        \label{eq:general-sip-first-order-derivatives}
	\end{align}
\end{subequations}
We have $\frac{\partial \mathcal{L}_{\mathrm{LL}}}{\partial \gamma} = 0$ from the KKT conditions.
We have ${g_{\llConstrIndex}\!\left( {\gamma}^*(z)\right)=0}$ for active constraints and $\frac{\partial  \lambda_{\mathrm{LL}, \llConstrIndex}^*}{\partial z} = 0$ for strictly inactive constraints.
Therefore, we have
\begin{equation}
    \frac{\mathrm{d} h}{\mathrm{d} z} \left( z, \gamma^*(z) \right)
    = \frac{\partial \mathcal{L}_{\mathrm{LL}}}{\partial z} + 0 + 0\\
    = \frac{\partial h}{\partial z}\left( z, \gamma^*(z) \right),
\end{equation}
where the second equation is derived because the constraints on the infinite index set $\Gamma$ are not functions of $z$.
\end{proof}

\begin{remark}
	\label{rem:SCS-relaxation}
	The assumption of satisfaction of the SCS condition can be relaxed.
	If the strong second-order sufficient condition\mbox{\cite{Jittorntrum1984}} is satisfied, the lower-level solution map $\gamma^*(z)$ and its associated multiplier $\lambda_{\mathrm{LL}}^*(z)$ are Lipschitz continuous, and Lemma~\ref{lemma:partial-derivatives-equal-to-total} remains valid\mbox{\cite{Hettich1993}}.
\end{remark}

As the partial derivatives provide the exact gradient information,
the linearized constraint~\eqref{eq:general-reduction-total-derivatives} can be simplified to
\begin{equation}
	 0 \geq \check{h}_{\iota}(z;\acute{z}) := h\left(\acute{z}, \gamma^*_{\iota} \left(\acute{z}\right)\right) + \left. \frac{\partial h \! \left( {z}, \gamma_{\iota}^*(z)\right) } {\partial z} \right\vert_{z=\acute{z}} \left( z \! - \! \acute{z} \right).
\label{eq:general-linerized-constr-simplified}
\end{equation}
The SIP problem~\eqref{eq:general-SIP-def} can be solved by iteratively solving subproblems in which the constraints are locally reduced and linearized (see Algorithm~\ref{alg:conceptual-reduction}).
\begin{algorithm}[t]
	\caption{Local reduction for solving \gls{sip} problem~\eqref{eq:general-SIP-def}}
	\label{alg:conceptual-reduction}
	\begin{algorithmic}
		\Require $z^{(0)}$
		\For{ $ j = 0, \Compactcdots, \text{MAXITER} $}
		\State $\gamma^* \!\left(z^{(j)}\right) \!\gets\!$ solve lower-level problems
        \Comment{\eqref{eq:general-ll-prob-def}}
		\State $z^{(j+1)}\!\gets\!$ (partially) solve the reduced subprob. with the linearized constr. using \gls{nlp} solver
		\Comment{\eqref{eq:general-reduction}, \eqref{eq:general-linerized-constr-simplified}}
		\If{$ \left\| z^{(j+1)} - z^{(j)} \right\|_{\infty}  \leq \epsilon_{\text{cvg}}$}
		\State CONVERGED $\gets 1$, \Return $z^{(j+1)}$
		\EndIf
		\EndFor
		\State \Return $z^{(\text{MAXITER}+1)}$
	\end{algorithmic}
\end{algorithm}

\subsection{External Active-Set Method}
Consider the \gls{nlp} problem
\begin{mini!}|s|
	{z}{L(z)}{\label{eq:general-NLP-many-constr-def}}{}
	\addConstraint{0}{\geq h_{\finiteIndex}(z), \ \finiteIndex \in \mathcal{I}},
\end{mini!}
where $\finiteIndex$ denotes the index of the constraints and $\mathcal{I}$ is a finite (but very large) set.
    The optimization problems considered in this paper have a large number of constraints and a relatively low dimension of optimization variables, and typically only a few constraints are active at the solution.
    % the number of constraints that are active at the solution typically does not exceed the dimension of the optimization variable $z$.
While the active set method~\cite[Section 16.5]{Nocedal2006} is efficient for optimization problems with a moderate number of constraints,
    the expense of identifying the active set becomes high when the number of constraints is large.
In~\cite{Schittkowski1992, Chung2009}, an \textit{external} active-set method is proposed to reduce the computation time of solving such problems.

The authors of~\cite{Chung2009} iteratively solve \gls{nlp} subproblems that contain an increasingly larger subset of the constraints of the original \gls{nlp}:
\begin{mini!}|s|
	{z}{L(z)}{\label{eq:general-NLP-many-constr-subset}}{}
	\addConstraint{0}{\geq h_{\finiteIndex}(z), \ \finiteIndex \in \mathcal{I}}_{\mathrm{s}}^{(j)},
\end{mini!}
where $\mathcal{I}_{\mathrm{s}}^{(j)}$ denotes the index subset at iteration $j$.
Define the maximum constraint violation by
\begin{equation}
    H_{+}(z) := \max\left( 0, \max_{\finiteIndex \in \mathcal{I}} h_{\finiteIndex}(z) \right).
\end{equation}
At each iteration,
    the index of the constraints whose values are greater than $H_{+}(z) - \epsilon$ with $\epsilon > 0$ is added to the index subset:
\begin{equation}
	\mathcal{I}_{\mathrm{s}}^{(j)} = \mathcal{I}_{\mathrm{s}}^{(j\shortminus 1)} \cup \left\lbrace \finiteIndex \mid h_{\finiteIndex}(z) \geq H_{+}(z) - \epsilon  \right\rbrace.
	\label{eq:update-index-subset}
\end{equation}
The algorithm is detailed in Algorithm~\ref{alg:conceptual-external-active}.
\begin{algorithm}[t]
	\caption{External active-set method for solving~\eqref{eq:general-NLP-many-constr-def}}
	\label{alg:conceptual-external-active}
	\begin{algorithmic}[1]
		\Require $\mathcal{I}_{\mathrm{s}}^{(\shortminus 1)}$
		\For{ $ j = 0, \Compactcdots, \text{MAXITER} $}
		\State $\mathcal{I}_{\mathrm{s}}^{(j)} \!\gets\!$ update index subset $\mathcal{I}_{\mathrm{s}}^{(j \shortminus 1)}$
		\Comment{\eqref{eq:update-index-subset}}
		\State $z^{(j+1)}\!\gets\!$ (partially) solve the subprob.~\eqref{eq:general-NLP-many-constr-subset} \label{alg-line:external-active-solve-subproblem}
		\If{$ \left\| z^{(j+1)} - z^{(j)} \right\|_{\infty}  \leq \epsilon_{\text{cvg}}$}
		\State CONVERGED $\gets 1$, \Return $z^{(j+1)}$
		\EndIf
		\EndFor
		\State \Return $z^{(\text{MAXITER}+1)}$
	\end{algorithmic}
\end{algorithm}

\begin{remark}
	\label{rem:external-active-set-SIP-subproblem}
	{
	Note that each of the constraints in~\eqref{eq:general-NLP-many-constr-def} may in fact be infinite, i.e., of the form~\eqref{eq:general-SIP-constr}.
	In this case, the optimization problem~\eqref{eq:general-NLP-many-constr-def} and the subproblem~\eqref{eq:general-NLP-many-constr-subset} are \gls{sip} problems.
	In Line~\ref{alg-line:external-active-solve-subproblem} of Algorithm~\ref{alg:conceptual-external-active},
	    the subproblem~\eqref{eq:general-NLP-many-constr-subset} can be (partially) solved using \gls{sip} methods, e.g., Algorithm~\ref{alg:conceptual-reduction}.
	}
\end{remark}

\subsection{Zero-Order Robust Optimization (zoRO)}
\label{sec:zoRO-preliminary}

The \gls{zoro} method is proposed to relieve the computational burden of modeling the uncertainty dynamics in the robust \glspl{ocp}~\cite{Feng2020Adjoint, Zanelli2021zoRO}.
To be able to describe the method on a higher level of abstraction,
we illustrate with a compact notation in this subsection.
Let $\nominalOptVar$ and $\uncOptVar$ summarize the nominal input-state trajectory and the uncertainty shape matrices, respectively:
\begin{equation}
	\begin{split}
		\nominalOptVar := & \left( {u}_0, \Compactcdots, {u}_{N \shortminus 1}, {x}_0, \Compactcdots, {x}_N \right), \\
		\uncOptVar := & \left(
		\vecMat\left( \Sigma_0\right) ,  \vecMat\left( \Sigma_1\right), \Compactcdots, \vecMat\left( \Sigma_N\right) \right) .
	\end{split}
	\label{eq:barz-tildez-def}
\end{equation}
The robust \gls{ocp}~\eqref{eq:robust-ocp-def} in the compact notation is given by:
\begin{mini!}|s|
	{\nominalOptVar, \uncOptVar}{L(\nominalOptVar)}{\label{eq:robust-ocp-abstract-wo-collision-constr}}{}
	\addConstraint{0}{=\kappa_{\psi}(\nominalOptVar),}
	\addConstraint{0}{=\kappa_{\Phi}\left(\nominalOptVar,\uncOptVar\right), }
	\addConstraint{0}{\geq h_{i}\left( \nominalOptVar \!+ \! D(\uncOptVar)\gammaNoise \right), \nonumber}
	\addConstraint{}{\quad \quad \forall \gammaNoise \in \underbrace{\gammaSetSphere{n_x} \! \times \! \Compactcdots \! \times \gammaSetSphere{n_x}}_{N+1},\ i  \! \in \! \mathcal{I}, \label{eq:robust-prelim-affine-constraint}}
\end{mini!}
where the constraint $\kappa_{\psi}(\nominalOptVar) = 0$ contains the nominal system dynamics~\eqref{eq:nominal-state-dynamics}
and $\kappa_{\Phi}(\nominalOptVar, \uncOptVar) = 0$ the dynamics of the ellipsoidal uncertainty sets~\eqref{eq:Sigma-function-of-state-trajectory}.
The uncertainty $\gammaNoise$ is bounded by the Cartesian product of $N+1$ unit balls.
The index set $\mathcal{I}$ summarizes the constraint indices for constraints of all time indices~$k \in \mathbb{N}_{\left[0, N\right]}$.
The mapping $D(\uncOptVar)$ is given by
\begin{equation}
	\setlength\arraycolsep{1pt}
	D(\uncOptVar) \! := \! \begin{bmatrix}
		0^{Nn_u \times (N+1)n_x} \\
		\mathrm{blk\hbox{-}diag}\left(\Sigma_{0}^{\frac{1}{2}}, \Sigma_{1}^{\frac{1}{2}}, \Compactcdots, \Sigma_{N}^{\frac{1}{2}} \right) &
	\end{bmatrix} \!.
\end{equation}

The robust convex optimization approach described in~\cite{BenTal1998} provides a method for ensuring robust constraint satisfaction.
For an affine constraint $h_{\text{aff}}\!\left( \nominalOptVar \!+ \! D(\uncOptVar)\gammaNoise \right) \leq 0$,
    the maximum deviation of the constraint value from its nominal value can be derived analytically.
The infinitely many constraints can thereby be equivalently reformulated as a single constraint tightened by a backoff:
\begin{equation}
	0 \geq h_{\text{aff}}(z) + \left\|\frac{\partial h_{\text{aff}}(z)}{\partial z}  D(\Sigma) \right\|_2.
\end{equation}
When a robust \gls{ocp} is subject to nonaffine inequality constraints,
	a linearization-based approximation of the constraint backoff,
	which introduces a truncation error~\cite{Gao2023, Diehl2006c}, is commonly used:
\begin{equation}
	\beta_i\left(\nominalOptVar,  \uncOptVar \right):= \left\|\frac{\partial h_i(\nominalOptVar)}{\partial \nominalOptVar}  D(\Sigma) \right\|_2.
	\label{eq:compute-backoff}
\end{equation}
The resulting robust \gls{ocp} reformulation is given by
\begin{mini!}|s|
	{\nominalOptVar, \uncOptVar}{L(\nominalOptVar)}{\label{eq:robust-ocp-abstract-affine-reduced}}{}
	\addConstraint{0}{=\kappa_{\psi}(\nominalOptVar),}
	\addConstraint{0}{=\kappa_{\Phi}\left(\nominalOptVar,\uncOptVar\right), \label{eq:robust-unc-dynamics-affine-reduced}}
	\addConstraint{0}{\geq h_{i}\! \left( \nominalOptVar \right) \!+ \! \beta_i\left(\nominalOptVar,  \uncOptVar \right), \ i  \! \in \! \mathcal{I}, \label{eq:robust-ocp-tightened-constr-affine-reduced}}
\end{mini!}
which is a finitely-constrained optimization problem.
Later sections will present a method for robustifying nonaffine collision-avoidance constraints using a combination of the backoff reformulation and the local reduction method.

The computational cost of solving the finitely-constrained robust \gls{ocp}~\eqref{eq:robust-ocp-abstract-affine-reduced} is significantly higher than that of a nominal \gls{ocp}.
Consider the augmented state $\breve x_k = \left( x_k, \vecMat\left( \Sigma_k\right)\right) \in \bbR^{\breve n_x}$ with dimension $\breve n_x = n_x + \frac{1}{2}n_x(n_x + 1)$ corresponding to the state of \gls{ocp}~\eqref{eq:disturbed-system-discrete-time-dynamics}.
Using standard \gls{ocp} structure exploiting algorithms, the robust \gls{ocp}~\eqref{eq:robust-ocp-abstract-affine-reduced} incurs a computational cost of $\bigO(\breve n_x^3)$ per iteration~\cite{Rawlings2017}, i.e., $\bigO(n_x^6)$ with respect to the original state dimension.

The \gls{zoro} method~\cite{Zanelli2021zoRO, Feng2020Adjoint,Frey2024a} reduces the computational complexity of solving robust \glspl{ocp}.
The core idea of \gls{zoro} is to alternate between a forward simulation of the uncertainty dynamics and the solution of a nominal \gls{ocp} with fixed backoffs,
whose values are obtained using the current iterate of the nominal trajectory $\nominalOptVar^{(j)}$ and the uncertainty matrices $\uncOptVar^{(j)}$.
The \gls{ocp} with the fixed backoffs is given by
\begin{mini!}|s|
	{\nominalOptVar}{L\left(\nominalOptVar\right)}{\label{eq:robust-zoRO-subproblem}}{}
	\addConstraint{0}{=\kappa_{\psi}\left( \nominalOptVar\right), }
	\addConstraint{0}{\geq h_{i}\left( \nominalOptVar \right) + {\beta}_i\!\left(\nominalOptVar^{(j)}, \uncOptVar^{(j)} \right) , \ i  \! \in \! \mathcal{I}, \label{eq:robust-zoRO-subproblem-constr}}
\end{mini!}
which can be solved by standard \gls{ocp} solvers.
The solution of~\eqref{eq:robust-zoRO-subproblem} yields a new iterate of the nominal trajectory,
at which the uncertainty dynamics~\eqref{eq:robust-unc-dynamics-affine-reduced} is simulated again and
the backoff terms are recomputed.
The process is repeated until convergence (Algorithm~\ref{alg:robust-zoRO}).
As this algorithm neglects the sensitivity
$\frac{\partial \kappa_{\Phi}}{\partial \nominalOptVar}$, it does not converge exactly to a solution of~\eqref{eq:robust-ocp-abstract-affine-reduced}, but to a suboptimal, yet feasible, point in its neighborhood \cite{Zanelli2021zoRO}.
If desired, optimality can be achieved by adding an appropriate gradient correction term to the objective~\cite{Feng2020Adjoint}.

\begingroup
\begin{algorithm}[h]
	\caption{Zero-order robust optimization (zoRO)}
	\label{alg:robust-zoRO}
	\begin{algorithmic}
		\Require $\nominalOptVar^{(0)}$
		\For{ $ j = 0, \Compactcdots, \text{MAXITER} $}
		\State $\uncOptVar^{(j)} \gets $ propagate unc. dyn. based on $\nominalOptVar^{(j)}$
		\Comment{\eqref{eq:robust-unc-dynamics-affine-reduced}}
		\For{$i  \! \in \! \mathcal{I}$}
		\State  $\beta_i \!\left(\nominalOptVar^{(j)}, \uncOptVar^{(j)} \right) \gets $ compute backoff terms
		\Comment{\eqref{eq:compute-backoff}}
		\EndFor
		\State $\nominalOptVar^{(j+1)} \! \gets \! $ solve subprob.~\eqref{eq:robust-zoRO-subproblem}
		\If{$ \left\|\nominalOptVar^{(j+1)} - \nominalOptVar^{(j)}\right\|_{\infty}  \leq \epsilon_{\text{cvg}}$}
		\State CONVERGED $\gets 1$, \Return $\nominalOptVar^{(j+1)}$
		\EndIf
		\EndFor
		\State \Return $\nominalOptVar^{(\text{MAXITER}+1)}$
	\end{algorithmic}
\end{algorithm}
\vspace*{-5pt}
\endgroup

%% file: sections/4-methods.tex
\section{Numerical Method for Solving Nominal and Robust Collision-Avoidance OCPs}
\label{sec:methods-solve-robust-OCP}
In this section, the first two subsections focus on the collision avoidance for a single nominal state.
Next, we consider a nominal trajectory and present a numerical method for solving the nominal \gls{ocp}~\eqref{eq:nominal-ocp-def}.
The fourth subsection introduces state uncertainties and presents a tight approximation of the robust collision-avoidance constraints~\eqref{eq:robust-ocp-collision-constr}, allowing the constraint infiniteness from the uncertainties and shape parameterization to be handled within a single framework.
Finally, we present a numerical method for solving the robust~\gls{ocp}.

\subsection{Local Reduction for Infinite Constraints Due to Robot Shape Parameterization}
\label{sec:single-state-single-obstacle}
In order to isolate the key concepts,
in this subsection we present the algorithm for collision avoidance for one nominal state with respect to a given point obstacle ${p_{\mathrm{o}}\in\bbR^{\nWorld}}$.
An infinite number of constraints are imposed because of the robot shape parameterization.
The optimization problem is to minimize an objective function $L(x)$ over the robot state that does not incur collision:
\begin{mini!}|s|
	{x}{L(x) }{\label{eq:single-nominal-state-opt-def}}{}
	\addConstraint{0\geq }{ \,  r_{\mathrm{shp}} - \left\|p_{\mathrm{c}}\!\left({x} \right)  +  R\!\left({x}\right) \gammaShape - p_{\mathrm{o}} \right\|_2, \nonumber}
	\addConstraint{}{\forall \gammaShape \in \gammaSetShape.}
\end{mini!}

The lower-level optimization problem associated with the constraint in~\eqref{eq:single-nominal-state-opt-def} is to find the maximizer of the negative Euclidean distance between the polygon and the point $p_{\mathrm{o}}$:
\begin{alignat}{2}
	h_{\text{coll}}(x;p_{\mathrm{o}}) \!:= & \max_{\gammaShape} \quad && r_{\mathrm{shp}} \!- \! \left\| p_{\mathrm{c}}\!\left({x} \right) + R\!\left({x}\right) \gammaShape - p_{\mathrm{o}} \right\|_2 \nonumber\\
    & \ \mathrm{s.t.} && \gammaShape \in \gammaSetShape. \label{eq:simple-example-lower-level}
\end{alignat}
By omitting the constant terms and squaring the $l_2$ norm, the optimization problem~\eqref{eq:simple-example-lower-level} becomes a \gls{qp} problem,
    which can be efficiently solved by numerical solvers such as Clarabel~\cite{Goulart2024}.

\begin{remark}
	\label{rem:our-sip-degenerate-cases}
	The lower-level problem~\eqref{eq:simple-example-lower-level} has one strict local maximizer that varies continuously with the robot state.
	The infinitely many constraints can be reduced and Lemma~\ref{lemma:partial-derivatives-equal-to-total} can be applied except when the vector $R^{-1}\!\left({x}\right)\left(p_{\mathrm{o}} - p_{\mathrm{c}}\!\left({x} \right) \right)$ lies on the boundary of the set $\gammaSetShape$.
	As our formulation enforces a minimum separation distance of $r_{\text{shp}} > 0$ between the polygon and the obstacle,
	    this degenerate case will not occur in the neighborhood of a solution.
	In contrast, if we use the Euclidean distance transform or consider the full shape of the obstacle, e.g., polygons,
	we may encounter the situation that there exist an infinite number of distance minimizers or that the minimizer is not a continuous function of the robot state.
\end{remark}

Let $\gammaShape^*({x};p_{\mathrm{o}})$ denote the lower-level maximizer.
Following Algorithm~\ref{alg:conceptual-reduction},
the optimization problem~\eqref{eq:single-nominal-state-opt-def} can be solved by iteratively computing the lower-level maximizer, linearizing the constraint, and solving the reduced subproblem.
Denote
$ \eta \! \left({x}, \gammaShape \right) \!:= \!p_{\mathrm{c}}\left({x} \right) + R\left({x}\right) \gammaShape  - p_{\mathrm{o}}$.
Using Lemma~\ref{lemma:partial-derivatives-equal-to-total},
we can derive that the derivatives of the constraint function in~\eqref{eq:simple-example-lower-level}
\begin{equation}
	\frac{\mathrm{d} {h}_{\text{coll}}}{\mathrm{d} x} \!=\! - \frac{\eta \! \left({x}  , \gammaShape^*(x)\right)^{\top}}{ \left\| \eta \! \left({x}  , \gammaShape^*(x)\right) \right\|_2  } \left. \frac{\partial \eta (x, \gammaShape)}{\partial x}\right\vert_{\gammaShape=\gammaShape^*(x)}\!.
	\label{eq:coll-derivatives}
\end{equation}
{
In the 2D scenario,
the partial derivative is given by
\begin{equation}
	\setlength{\arraycolsep}{3pt}
	\frac{\partial \eta (x, \gammaShape)}{\partial x}  = \rotMat\!\left( x\right) J_{\mathrm{t}}\!\left( x\right) + \rotMat \!\left( x\right)\begin{bmatrix}
		0 & -1 \\ 1 & 0
	\end{bmatrix}\gammaShape J_{\mathrm{r}}\!\left( x\right) ,
\end{equation}
where $J_{\mathrm{t}}\!\left( x\right) \in \bbR^{2\times n_x} $ and $ J_{\mathrm{r}}\!\left( x\right) \in \bbR^{1\times n_x}$ denote the translational and rotational Jacobians, respectively,
    mapping the time derivatives of the state variable to the linear and angular velocities of the robot in its local frame.
The partial derivative in the 3D scenario is given by
\begin{equation}
	\frac{\partial \eta (x, \gammaShape)}{\partial x}  = \rotMat\!\left( x\right) J_{\mathrm{t}}\left( x\right) - \rotMat \!\left( x\right)\left[ \gammaShape\right]_{\times} J_{\mathrm{r}}\!\left( x\right) ,
\end{equation}
with $J_{\mathrm{t}}\!\left( x\right)  \in \bbR^{3\times n_x} $ and $ J_{\mathrm{r}}\!\left( x\right)\in \bbR^{3\times n_x}$.
We refer to~\cite[Section 3.1]{Siciliano2009} for a detailed explanation of the Jacobian matrices $J_{\mathrm{t}}$ and $J_{\mathrm{r}}$ and~\cite{Sola2018} for the derivatives on Lie groups.
}

At iteration $j$, the linearized function is given by
\begin{equation}
	\check{h}_{\text{coll}}\left({x};{x}^{(j)}, p_{\mathrm{o}}\right) := {h}_{\text{coll}}({x}^{(j)}; p_{\mathrm{o}})
	+ \frac{\mathrm{d} {h}_{\text{coll}}}{\mathrm{d} x} \left(x \!- \!x^{(j)} \right).
	\label{eq:single-nominal-linearized-constr}
\end{equation}
The iterative subproblem is in the form of
\begin{mini!}|s|
	{x}{L(x) }{\label{eq:single-nominal-iterative-subproblem}}{}
	\addConstraint{0 }{ \geq \check{h}_{\text{coll}}\left({x};{x}^{(j)}, p_{\mathrm{o}}\right).}
\end{mini!}
The convergence criterion is that the change in the robot state $\left\| x^{(j+1)} - x^{(j)} \right\|_{\infty}$ falls below a predefined threshold $\epsilon_{\text{cvg}}$.

\begin{remark}
	When the point obstacle $p_{\mathrm{o}}$ is inside the polygon,
	the value of the Euclidean distance is constantly zero.
	The gradient information of the Euclidean distance cannot let the polygon get out of the obstacle.
	In this case we will fall back to the gradient information provided by the signed distance.
\end{remark}

\subsection{External Active-Set Method for a Large Number of Obstacles}
\label{sec:many-obstacles}

\begin{figure}
	\center
	\includegraphics[width=\linewidth]{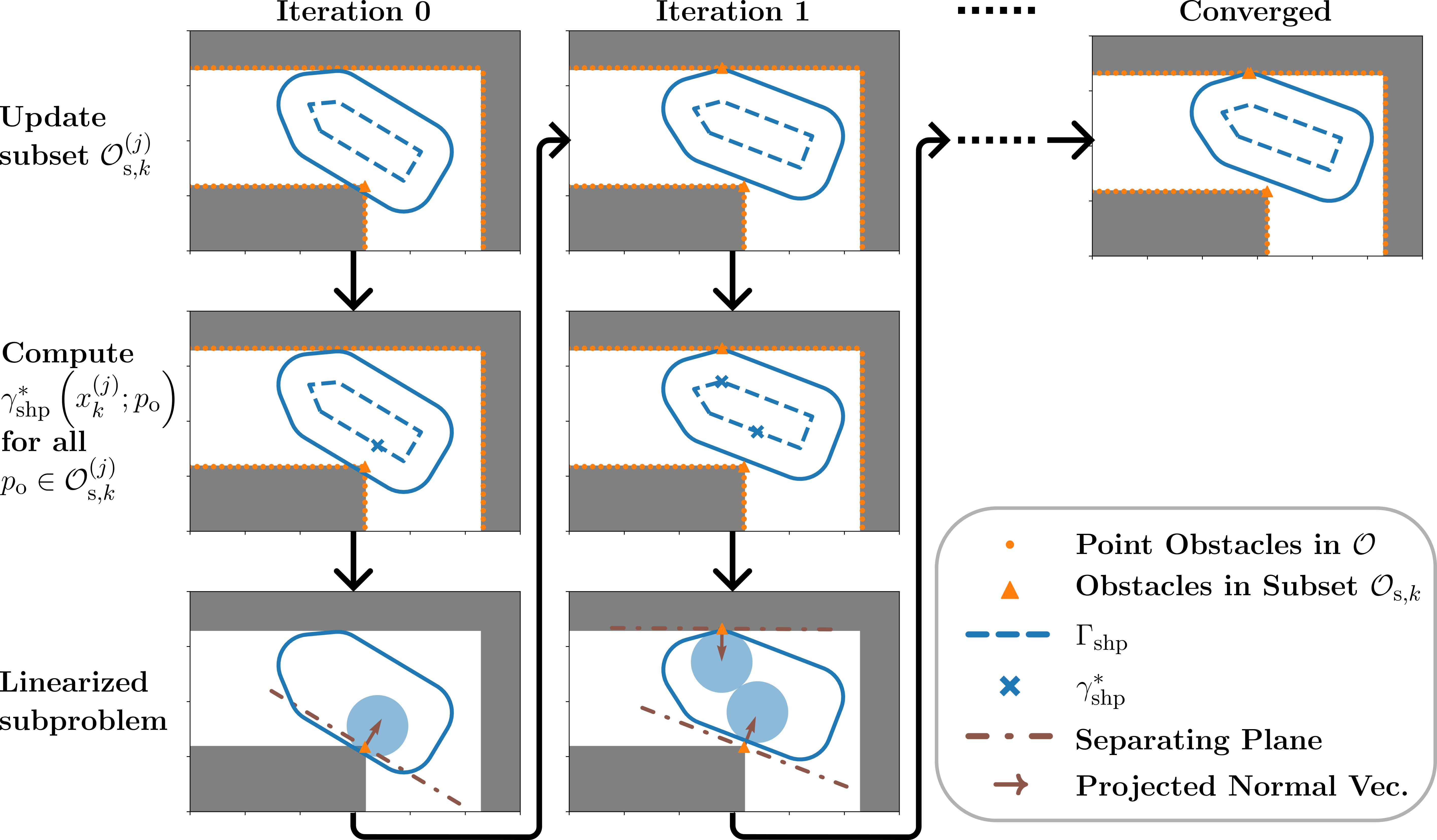}
	\caption{Illustration of Algorithm~\ref{alg:nominal-collision-free-trajectory}.
		For clarity, only the robot state and the obstacle subset at one time step $k$ are plotted.
		The closest obstacles are identified among all point obstacles (see the orange dots), and the obstacle subset $\mathcal{O}_{\mathrm{s}, k}^{(j)}$ is updated (see the orange triangles).
		For each obstacle $p_{\mathrm{o}} \in \mathcal{O}_{\mathrm{s}, k}^{(j)}$,
		the lower-level optimization problem~\eqref{eq:simple-example-lower-level} is solved,  yielding the lower-level maximizer $\gammaShape^{*}\!\left(x^{(j)}_k;p_{\mathrm{o}}\right)$ and one linearized constraint is imposed,
		which can be interpreted as a plane (see the brown dash-dotted lines and the brown arrows) to separate the obstacle and the circle corresponding to the maximizer.
		The procedure is repeated until convergence.}
	\label{fig:graphic_illustration}
\end{figure}

Here we look into the method for handling a set of point obstacles,
the cardinality of which can be very large.
The optimization problem is given by
\begin{mini!}|s|
	{{x}}{L\left({x} \right)} {\label{eq:multiple-obs-single-state-opt-def}}{}
	\addConstraint{0 }{ \geq r_{\mathrm{shp}} \!-\! \left\|p_{\mathrm{c}}\!\left({x} \right) \!+\!  R\!\left({x}\right)\!\gammaShape  \!-\! p_{\mathrm{o}}  \right\|_2 \! , \nonumber }
	\addConstraint{}{\forall \gammaShape \in \gammaSetShape, \ p_{\mathrm{o}} \in \mathcal{O}.}
\end{mini!}
{
The optimization problem~\eqref{eq:multiple-obs-single-state-opt-def} is a specific case of~\eqref{eq:general-NLP-many-constr-def} with each constraint $i$ being infinite.
Thus,~\eqref{eq:multiple-obs-single-state-opt-def} can be solved by Algorithm~\ref{alg:conceptual-external-active}, 
   with the subproblems being \gls{sip} problems solvable by Algorithm~\ref{alg:conceptual-reduction}, cf. Remark~\ref{rem:external-active-set-SIP-subproblem}.
The finite index set~$\mathcal{I}$ in~\eqref{eq:general-NLP-many-constr-def} specifically corresponds to the obstacle set~$\mathcal{O}$.}
The algorithm is to iteratively update an obstacle subset $\mathcal{O}_{\mathrm{s}}^{(j)} \subset \mathcal{O}$ and
solve subproblems in which each constraint is associated with an obstacle within the obstacle subset $\mathcal{O}_{\mathrm{s}}^{(j)}$:
\begin{mini!}|s|
	{{x}}{L\left({x} \right)} {\label{eq:multiple-obs-single-state-opt-obs-subset}}{}
	\addConstraint{0 }{ \geq r_{\mathrm{shp}} \!-\! \left\|p_{\mathrm{c}}\!\left({x} \right) \!+\!  R\!\left({x}\right)\!\gammaShape  \!-\! p_{\mathrm{o}}  \right\|_2 \! , \nonumber }
	\addConstraint{}{\forall \gammaShape \in \gammaSetShape, \ p_{\mathrm{o}} \in \mathcal{O}_{\mathrm{s}}^{(j)}.}
\end{mini!}

The subproblem~\eqref{eq:multiple-obs-single-state-opt-obs-subset} can then be solved using the local reduction method.
The subproblem of the subproblem~\eqref{eq:multiple-obs-single-state-opt-obs-subset} is of the form
\begin{mini!}|s|
	{{x}}{L({x})}{\label{eq:multiple-obs-single-state-lin-subprob}}{}
	\addConstraint{0\geq }{ \,  \check{h}_{\text{coll}}\left({x};{x}^{(j)}, p_{\mathrm{o}}\right)\!, \ p_{\mathrm{o}} \in \mathcal{O}_{\mathrm{s}}^{(j)} .}
\end{mini!}
In practice, the local-reduction iteration is performed only once after an update of the obstacle subset.
Otherwise, it would be computationally inefficient for the relatively time-intensive local-reduction iteration to converge while the outer external-active-set iteration is far from convergence.

The obstacle subset $\mathcal{O}_{\mathrm{s}}^{(j)}$ is updated by adding those obstacles which are the closest to the polygon,
if the obstacles are within a predefined neighborhood of the polygon.
For the converged solution, all obstacles of the set $\mathcal{O}$ satisfy the collision-avoidance constraints within the convergence criteria.
Consider a converged solution ${x}^{(j_{\text{cvg}})}$.
By contradiction we suppose that there exists an obstacle $p_{\text{coll}} \in \mathcal{O}$ that violates the collision-avoidance constraint.
Since ${x}^{(j_{\text{cvg}})}$ is a converged solution,
the solution ${x}^{(j_{\text{cvg}})}$ is sufficiently close to ${x}^{(j_{\text{cvg}}\shortminus 1)}$ and the obstacle $p_{\text{coll}}$ leads to constraint violation for ${x}^{(j_{\text{cvg}}\shortminus 1)}$ as well.
In consequence, $p_{\text{coll}}$ will be added to the subset,
which alters the solution.
If the change in the solution violates the convergence criteria,
    the algorithm did in fact not converge.

\vspace{-0.2\baselineskip}
\subsection{Nominal Collision-Avoidance Optimal Control}
\label{sec:collision-free-trajectory}
\vspace{-0.2\baselineskip}

A collision-free trajectory can be obtained by applying the previously presented method to every state within the prediction horizon.
A summary of the numerical method for solving the \gls{ocp}~\eqref{eq:nominal-ocp-def} without any approximation is presented in Algorithm~\ref{alg:nominal-collision-free-trajectory}.
The function \mbox{\textsc{UpdateObsSubset}} is implemented differently for the specific use cases presented in this paper.
The specific implementation for each case is provided in the corresponding section.
A graphical illustration of the algorithm is provided in Fig.~\ref{fig:graphic_illustration}.

\begin{algorithm}[t]
	\caption{Method for solving the nominal OCP~\eqref{eq:nominal-ocp-def}}
	\label{alg:nominal-collision-free-trajectory}
	\begin{algorithmic}
		\Require Initial guess ${u}^{(0)}_{0}, \Compactcdots, {u}^{(0)}_{N\shortminus 1}$,
		${x}^{(0)}_{0}, \Compactcdots, {x}^{(0)}_{N}$,
		and ${\mathcal{O}}^{(\shortminus 1)}_{\mathrm{s}, 1}, \Compactcdots, {\mathcal{O}}^{(\shortminus 1)}_{\mathrm{s}, N}$
		\For{$ j = 0, \Compactcdots, \text{MAXITER} $}
		\For{$k=1, \Compactcdots, N$}
		\State $\mathcal{O}_{\mathrm{s}, k}^{(j)} \gets$ \Call{UpdateObsSubset}{${x}_k^{(j)} , \mathcal{O}^{(j\shortminus 1)}_{\mathrm{s}, k}$}
		\For{$p_{\mathrm{o}} \in \mathcal{O}^{(j)}_{\mathrm{s}, k}$}
		\State $\gammaShape^*\left({x}^{(j)}_k; p_{\mathrm{o}}\right) \gets$ comp. lower-level maximizer
		\EndFor
		\EndFor
		\State ${u}^{(j+1)}_{0},\!\! \Compactcdots {u}^{(j+1)}_{N\shortminus 1}, {x}^{(j+1)}_{0}, \!\!\Compactcdots{x}^{(j+1)}_{N}  \! \gets \! $ solve subprob.~\eqref{eq:nominal-ocp-subproblem}
		\State $\Delta = \max(\|{u}^{(j+1)}_{0} - {u}^{(j)}_{0}\|_{\infty}, \Compactcdots,  \|{x}^{(j+1)}_{N} - {x}^{(j )}_{N} \|_{\infty})$
		\If{$ \Delta  \leq \epsilon_{\text{cvg}}$}
		\State CONVERGED $\gets 1$, \Return ${u}_0^{(j+1)},\Compactcdots, {u}^{(j+1)}_{N\shortminus 1}$
		\EndIf
		\EndFor
		\State \Return ${u}_0^{(\text{MAXITER}+1)},\Compactcdots, {u}^{(\text{MAXITER}+1)}_{N\shortminus 1}$
	\end{algorithmic}
\end{algorithm}

The subproblem solved in each iteration takes the form
\begin{mini!}|s|
	{\substack{{x}_0, \Compactcdots, {x}_N,\\
			{u}_0, \Compactcdots, {u}_{N\!\shortminus \!1}}}
	{\sum_{k=0}^{N\!-\!1} {L}_k\left({x}_k, {u}_k\right) + {L}_N\left({x}_N\right)}
	{\label{eq:nominal-ocp-subproblem}}{}
	\addConstraint{{x}_{0}}{={\bar{x}}_{0},}{}
	\addConstraint{{x}_{k+1}}{=\psi_k({x}_k,\! {u}_k), \ }{k \!\in \! \mathbb{N}_{\left[0, N\shortminus 1 \right]}, \label{eq:nominal-ocp-dynamic-constraints-nonlinearized}}
	\addConstraint{0}{\geq h_{\mathcal{XU}}\!\left(x_k, u_k \right),\ }{k \!\in \! \mathbb{N}_{\left[0, N\shortminus 1 \right]},}
	\addConstraint{0}{\geq h_{\mathcal{X}_N}\!\left(x_N \right), }
	\addConstraint{0 }{\geq \check{h}_{\text{coll}}\left({x}_k;{x}_k^{(j)}, p_{\mathrm{o}}\right),  \nonumber}
	\addConstraint{}{\text{for all } p_{\mathrm{o}} \in \mathcal{O}_{\mathrm{s}, k}^{(j)},\ }{ k \!\in \! \mathbb{N}_{\left[1, N \right]}}{\label{eq:nominal-ocp-subproblem-linearized-constr}}.
\end{mini!}
The constraints on the system dynamics~\eqref{eq:nominal-ocp-dynamic-constraints-nonlinearized} are not linearized.
The subproblem~\eqref{eq:nominal-ocp-subproblem} is solved by a \gls{sqp} solver, which is well suited for leveraging the solutions obtained in previous iterations, for a fixed number of \gls{qp} iterations.

\vspace{-0.3\baselineskip}
\begin{remark}
	\label{rem: discrete-time-coolision-avoidance}
	{
	This paper considers only discrete-time formulations of the collision-free~\gls{ocp}.
	Future research could investigate an approximation of continuous-time collision avoidance by enforcing collision avoidance in the shooting intervals,
	    e.g., at the collocation points of integration schemes,
	the temporal intervals of which are typically much smaller than the discretization intervals.
	The external active set method can be employed to simultaneously identify the closest obstacles and the worst-case collocation points.
    }
\end{remark}

\subsection{{Approximate Reformulation of the Robust Collision-Avoidance Constraints}}
\label{sec:reformulation-robust-constr}

Starting in this subsection, state uncertainties are taken into account.
This subsection presents an approximate reformulation of the robust collision-avoidance constraints~\eqref{eq:robust-ocp-collision-constr}.

For the infinite collision-avoidance constraints~\eqref{eq:robust-ocp-collision-constr},
    consider the worst-case translational and the worst-case rotational uncertainty independently.
Given one obstacle~$p_{\mathrm{o}}$ and one time step $k$,
    this corresponds to a constraint in the form of
\begin{equation}
	\begin{split}
	r_{\mathrm{shp}} \! \leq & \left\| p_{\mathrm{c}}\!\left({x}_k \! + \! \Sigma_k^{\frac{1}{2}}\gammaTransNoise\right) \! + \!  \rotMat \! \left({x}_k  \! + \! \Sigma_k^{\frac{1}{2}} \gammaRotNoise\right) \! \gammaShape \!- \!p_{\mathrm{o}} \right\|_2 \!, \\
	& \forall \gammaShape \! \in \! \gammaSetShape,
	\gammaTransNoise \in \mathbb{B}^{n_x}, \gammaRotNoise \in \mathbb{B}^{n_x}.
\end{split}
\label{eq:robust-coll-constr-separateRotPos}
\end{equation}
The enlarged infinite index set may lead to a smaller minimum value on the right of the inequality
   and thereby an inner approximation of the feasible set of the robust \gls{ocp}~\eqref{eq:robust-ocp-def}.

At a given state $x_k$, the constraint~\eqref{eq:robust-coll-constr-separateRotPos} can be approximated via linearization with respect to the uncertainties $\gammaTransNoise$ and $\gammaRotNoise$:
\begin{equation}
\begin{split}
	r_{\mathrm{shp}} & \leq \Big\| p_{\mathrm{c}}\!\left( x_k\right) + \rotMat\left( x_k \right) J_{\mathrm{t}}\!\left( x_k\right)\Sigma_k^{\frac{1}{2}}\gammaTransNoise\\
	& \quad  +  \rotMat \!\left( x_k\right)
	\so3Exp\left( \left[J_{\mathrm{r}}\!\left( x_k\right) \Sigma_k^{\frac{1}{2}}\gammaRotNoise\right]_{\times}\!\right) \gammaShape  - p_{\mathrm{o}} \Big\|_2,
\end{split}
\label{eq:robust-coll-constr-linearization}
\end{equation}
for all $\gammaShape \! \in \! \gammaSetShape$, $\gammaTransNoise
 \!\in \!\mathbb{B}^{n_x}$, and $\gammaRotNoise \!\in\! \mathbb{B}^{n_x}$,
    where $J_{\mathrm{t}}\!\left( x_k\right)$ and $J_{\mathrm{r}}\!\left( x_k\right)$
    denote the translational and rotational Jacobians, respectively,
    and $\so3Exp([\cdot]_{\times})$ denotes the exponential map.
The term $\rotMat \!\left( x_k\right)
\so3Exp\Big( \! \left[ J_{\mathrm{r}}\!\left( x_k\right) \Sigma_k^{\frac{1}{2}}\gammaRotNoise\right]_{\times}\!\Big) \gammaShape$,
    which consists of a matrix exponential and multiplication with $\gammaShape$,
    complicates the lower-level optimization problem of determining the worst-case $\gammaShape$, $\gammaTransNoise$, and $\gammaRotNoise$ with respect to the constraint~\eqref{eq:robust-coll-constr-linearization}.
    The proposed method handles the term $\rotMat \!\left( x_k\right)
    \so3Exp\Big( \! \left[ J_{\mathrm{r}}\!\left( x_k\right) \Sigma_k^{\frac{1}{2}}\gammaRotNoise\right]_{\times}\!\Big) \gammaShape$ by
	determining an upper bound of its possible impact on the constraint value and tightening the constraint accordingly.
The remaining constraint infiniteness in~\eqref{eq:robust-coll-constr-linearization} is tackled by the local reduction method.

In more detail, given the triangular inequality,
    the right side of the inequality in~\eqref{eq:robust-coll-constr-linearization} can be lower-bounded by
\begin{equation}
\begin{split}
	&\left\| p_{\mathrm{c}}\!\left( x_k\right) + \rotMat\left( x_k\right) J_{\mathrm{t}}\!\left( x_k\right)\Sigma_k^{\frac{1}{2}}\gammaTransNoise  + \rotMat \!\left( x_k\right)\gammaShape - p_{\mathrm{o}}\right\|_2 \\
	& - \!\left\| \rotMat\!\left( x_k\right) \so3Exp \!\left(\!\left[  J_{\mathrm{r}}\!\left( x_k\right) \Sigma_k^{\frac{1}{2}}\gammaRotNoise\right]_{\times}\!\right)\! \gammaShape \!- \!\rotMat \!\left( x_k\right)\!\gammaShape\right\|_2\!\!.
	\label{eq:robust-coll-triangular-lower-bound}
\end{split}
\end{equation}
Consider the 3D scenario.
Let $\theta\boldsymbol{\tau} = J_{\mathrm{r}}\!\left( x_k\right) \Sigma_k^{\frac{1}{2}}\gammaRotNoise$,
    where $\boldsymbol{\tau}$ is a unit vector, and~$\theta$ represents a rotation angle around $\boldsymbol{\tau}$, and is here considered nonnegative.
The second term in~\eqref{eq:robust-coll-triangular-lower-bound} can then be rewritten as
\begin{equation}
	\left\|\rotMat \!\left( x_k\right)\so3Exp\left( \left[ \theta\boldsymbol{\tau}\right]_{\times}\right) \gammaShape - \rotMat \!\left( x_k\right)\gammaShape \right\|_2.
	\label{eq:lemma-term-to-be-upper-bounded}
\end{equation}

\vspace{0.5\baselineskip}
\begin{lemma}
	For fixed values of $\gammaShape$ and $\gammaRotNoise$,
    \eqref{eq:lemma-term-to-be-upper-bounded} can be upper-bounded by $\theta \|\gammaShape\|_2 $.
\end{lemma}

\begin{proof}
	Since the vector norm remains unchanged under a rotation transform,
	    \eqref{eq:lemma-term-to-be-upper-bounded} is equal to $\left\|\so3Exp\left(\left[  \theta\boldsymbol{\tau}\right]_{\times}\right) \gammaShape - \gammaShape \right\|_2 $.
	The definition of the exponential map for rotation matrices of $\text{SO}\!\left(3\right)$ is given by
	$\so3Exp\left(\left[  \theta\boldsymbol{\tau}\right]_{\times}\right) \! = \! \mathbb{I} +  \sin\!\left( \theta\right) [\boldsymbol{\tau}]_{\times} \!+ \left( 1 \!- \! \cos\!\left( \theta\right)\right)[\boldsymbol{\tau}]_{\times}^2$,
	and we have
    \begin{subequations}
    	\allowdisplaybreaks
    	\begin{align}
    		& \left\|\so3Exp\left( \left[  \theta\boldsymbol{\tau}\right]_{\times}\right) \gammaShape - \gammaShape \right\|_2 \\
    		= &\left\| \sin\left( \theta\right) [\boldsymbol{\tau}]_{\times} \gammaShape + \left( 1- \cos\left( \theta\right)\right) [\boldsymbol{\tau}]_{\times}^2 \gammaShape\right\|_2
    		\label{eq:lemma-upper-bound-b}\\
    		= &\left( \sin^2\left( \theta\right) \left\| [\boldsymbol{\tau}]_{\times} \gammaShape\right\|_2^2 \!+\!\left( 1 \!-\! \cos\left( \theta\right)\right)^2  \left\| [\boldsymbol{\tau}]_{\times}^2 \gammaShape\right\|_2^2 \right)^{\frac{1}{2}}
    		\label{eq:lemma-upper-bound-c}
    		\\
    		\leq & \left(  \sin^2\left( \theta\right) + \left( 1- \cos\left( \theta\right)\right)^2\right)^{\frac{1}{2}} \|\gammaShape\|_2
    		\label{eq:lemma-upper-bound-d}\\
    		= & 2 \sin\left( \frac{\theta}{2}\right) \|\gammaShape\|_2 
    		\leq  \theta \|\gammaShape\|_2,
    	\end{align}
    \end{subequations}
    where~\eqref{eq:lemma-upper-bound-c} is obtained by using the fact that the two terms in the $l_2$ norm of~\eqref{eq:lemma-upper-bound-b} are orthogonal to each other,
    and the observation that the norm of the cross product between a vector and a unit vector is at most equal to the norm of the vector yields~\eqref{eq:lemma-upper-bound-d}.
\end{proof}

\begin{corollary}
	For all $\gammaShape \in \gammaSetShape$ and $\gammaRotNoise \in \mathbb{B}^{n_x}$,
	\eqref{eq:lemma-term-to-be-upper-bounded} can be upper-bounded by
	\begin{equation}
		\begin{split}
		&\beta_{\mathrm{coll}, k}\!\left(x_k, \Sigma_k\right):=\\
		& \lambda_{\max}^{\frac{1}{2}}\left(J_{\mathrm{r}} \!\left( x_k\right) \Sigma_k J_{\mathrm{r}} \!\left( x_k\right)^{\top} \right) \cdot \max_{\gammaShape \in \gammaSetShape}\! \|\gammaShape\|_2,
	\end{split}
		\label{eq:robust-coll-backoff-def}
	\end{equation}
where $\lambda_{\max}^{\frac{1}{2}}:\mathbb{S}^{n}_{+}\rightarrow \bbR$ denotes the square root of the maximum eigenvalue of a positive semi-definite matrix.
\end{corollary}
\begin{proof}
	From the definition of $\theta$ and $\boldsymbol{\tau}$, we have
	\begin{equation}
		\theta^2 =\left( J_{\mathrm{r}}\!\left( x_k\right) \Sigma_k^{\frac{1}{2}}\gammaRotNoise\right)^{\top} J_{\mathrm{r}}\!\left( x_k\right) \Sigma_k^{\frac{1}{2}}\gammaRotNoise.
	\end{equation}
	The corollary follows from the definition of the eigenvalue and that $\left( J_{\mathrm{r}}\!\left( x_k\right) \Sigma_k^{\frac{1}{2}}\right)^{\top} \! J_{\mathrm{r}}\!\left( x_k\right) \Sigma_k^{\frac{1}{2}}$ and $ J_{\mathrm{r}}\!\left( x_k\right) \Sigma_k^{\frac{1}{2}} \left( \! J_{\mathrm{r}}\!\left( x_k\right) \Sigma_k^{\frac{1}{2}}\right)^{\!\top}$ have the same nonzero eigenvalues.
\end{proof}

With this, a sufficient condition that guarantees the satisfaction of the constraint~\eqref{eq:robust-coll-constr-linearization} is given by
\begin{equation}
	\begin{split}
	0 & \geq r_{\mathrm{shp}} + \beta_{\mathrm{coll}, k}\!\left(x_k, \Sigma_k \right) - \Big\| p_{\mathrm{c}}\!\left( x_k\right)  + \rotMat \!\left( x_k\right)\gammaShape\\
	&  + \!\rotMat\!\left( x_k\right)\! J_{\mathrm{t}}\!\left( x_k\right)\!\Sigma_k^{\frac{1}{2}}\gammaTransNoise  \! - \! p_{\mathrm{o}}\Big\|_2, \forall \gammaShape \! \in \! \gammaSetShape, \gammaTransNoise \!\in \!\mathbb{B}^{n_x}\!.
\end{split}
\label{eq:robust-constraint-reformulation}
\end{equation}

The 2D counterpart of the backoff simplifies~\eqref{eq:robust-coll-backoff-def} to $\left( J_{\mathrm{r}} \!\left( x_k\right)  \Sigma_k J_{\mathrm{r}}\!\left( x_k\right)^{\top} \right)^{\frac{1}{2}}  \cdot \max_{\gammaShape \in \gammaSetShape} \|\gammaShape\|_2$.
In the application of mobile robot navigation,
    the system state often contains the robot position and the robot heading.
Thereby, the linearization-based approximation~\eqref{eq:robust-coll-constr-linearization} is not subject to a truncation error,
    and the feasible set defined by~\eqref{eq:robust-constraint-reformulation} provides an inner approximation of the feasible set defined by~\eqref{eq:robust-coll-constr-separateRotPos} and also that by~\eqref{eq:robust-ocp-collision-constr}.

For the constraints~\eqref{eq:robust-constraint-reformulation},
    the corresponding lower-level optimization problem is given by
\begin{subequations}
	\begin{alignat}{3}
		&\tilde{h}_{\text{coll}}\!\left( x_k, \Sigma_k; p_{\mathrm{o}}\right) := && \max_{\substack{\gammaShape,\\ \gammaTransNoise}} && \quad r_{\mathrm{shp}} \!-\! \Bigl\|p_{\mathrm{c}}\!\left( x_k\right)  \! + \!  \rotMat \!\left( x_k\right)\!\gammaShape \!-\! p_{\mathrm{o}}\nonumber \\
		& && &&\quad + \rotMat\!\left( x_k\right) J_{\mathrm{t}}\!\left(x_k\right) \Sigma_k^{\frac{1}{2}}\gammaTransNoise \Bigr\|_2 \\
		& &&  \;  \text{s.t.} && \quad \forall \gammaShape  \in  \gammaSetShape, \gammaTransNoise \in \mathbb{B}^{n_x}.
	\end{alignat}
\label{eq:robust-lower-level-qcqp-def}
\end{subequations}

Omitting the constant terms and squaring the $l_2$ norm,
   the lower-level optimization problem becomes a convex \gls{qcqp} problem and can be efficiently solved by numerical solvers such as Clarabel~\cite{Goulart2024}.
The lower-level maximizers are denoted by $\gammaShape^*({x}_k, \Sigma_{k};p_{\mathrm{o}})$ and $\gammaTransNoise^*({x}_k, \Sigma_{k};p_{\mathrm{o}})$, respectively.

With the approximate constraint reformulation,
    we obtain the robust \gls{ocp} formulation as follows:
\begin{mini!}|s|
	{\substack{{x}_0, \Compactcdots, {x}_N,\\
			{u}_0, \Compactcdots, {u}_{N\!\shortminus \!1}, \\
			\Sigma_{0}, \Compactcdots, \Sigma_{N}}}
	{\sum_{k=0}^{N\!-\!1} {L}_k\left({x}_k, {u}_k\right) + {L}_N\left({x}_N\right)}
	{\label{eq:robust-ocp-def-wReformulation}}{}
	\addConstraint{{x}_{0}}{={\bar{x}}_{0},\ \Sigma_0= {\bar{\Sigma}}_{0},}{}
	\addConstraint{{x}_{k+1}}{=\psi_k({x}_k,\! {u}_k),\  k \!\in \! \mathbb{N}_{\left[0,  N\!\shortminus\! 1 \right]}, }{}
	\addConstraint{\Sigma_{k+1}}{=\Phi_k({x}_k,\! {u}_k, \! \Sigma_k),\  k \!\in \!\mathbb{N}_{\left[0, N\!\shortminus\! 1 \right]}, \label{eq:robust-ocp-def-wReformulation-Sigma}}{}
	\addConstraint{}{\text{Stage and terminal constraints}~\eqref{eq:robust-ocp-stage-constraints}, \eqref{eq:robust-ocp-teminal-constraints} \nonumber}
	\addConstraint{0}{ \geq \tilde{h}_{\text{coll}}\!\left( x_k, \Sigma_k; p_{\mathrm{o}}\right)\!+ \!\beta_{\mathrm{coll}, k}\!\left(x_k,\Sigma_k \right)\! , \nonumber}
	\addConstraint{}{p_{\mathrm{o}} \in\staticObsSet, \ k\in\mathbb{N}_{\left[1, N\right]},}
\end{mini!}
which involves the following approximations:
\begin{enumerate}[label=(\alph*)]
	\item The infinite index set is over-approximated by independently considering the worst-case translational and the worst-case rotational uncertainties (see~\eqref{eq:robust-coll-constr-separateRotPos});
	\item The disturbed translation and rotation are approximated via linearization around the nominal state (see~\eqref{eq:robust-coll-constr-linearization}), which may result in loss of guarantees;
	\item The backoff for the rotational uncertainty~\eqref{eq:robust-coll-backoff-def} is not necessarily tight, contributing to conservativeness;
	\item Using the linearization-based model~\eqref{eq:dynamics_ellipsoid_lin} for the uncertainty dynamics~\eqref{eq:robust-ocp-def-wReformulation-Sigma} introduces a truncation error.
\end{enumerate}

%\vspace{-\baselineskip}
\subsection{{Robust Collision-Avoidance Optimal Control}}
\label{sec:solve-robust-ocp}
This subsection presents a numerical method for solving the robust \gls{ocp}~\eqref{eq:robust-ocp-def-wReformulation}.
We adopt the \gls{zoro} method (see Section~\ref{sec:zoRO-preliminary}) and update the uncertainty matrices in a zero-order fashion to mitigate the complexity introduced by modeling the uncertainty dynamics in the \gls{ocp}.
The constraints in the subproblems of the \gls{zoro} iteration take the form of
\begin{equation}
	0 \geq \tilde{h}_{\text{coll}}\!\left( x_k, \Sigma_k^{(j)}; p_{\mathrm{o}}\right)\!+ \!\beta_{\mathrm{coll}, k}\!\left(x_k^{(j)},\Sigma_k^{(j)} \right)\! , p_{\mathrm{o}} \in\staticObsSet,
\end{equation}
where ${x}^{(j)}_k$ and $\Sigma_{k}^{(j)}$ denote the state variables and the uncertainty shape matrices at the current iteration $j$, respectively.
Analogous to the standard \gls{zoro} method, our subproblem treats uncertainty matrices as parameters and uses fixed backoffs.
However, while the subproblem in the standard \gls{zoro} method is an \gls{nlp} problem,
our subproblem involves infinitely many constraints per obstacle and a large obstacle set.
To address this, we apply local reduction and external active-set methods as in the nominal case.
The locally-reduced, linearized robust collision-avoidance constraint is given by
\begin{equation}
	\begin{split}
		\check{\tilde{h}}_{\text{coll}}& \left({x}_k; {x}_k^{(j, m)},\Sigma_{k}^{(j)}, p_{\mathrm{o}}\right) \coloneqq \\
		& \tilde{h}_{\text{coll}}({x}^{(j, m)}_k, \Sigma_{k}^{(j)}; p_{\mathrm{o}}) + \frac{\mathrm{d} \tilde{h}_{\text{coll}}}{\mathrm{d} x_k} \left({x}_k \!- \!{x}_k^{(j, m)} \right), \\
		0 \geq & \, \check{\tilde{h}}_{\text{coll}}\left({x}_k; {x}_k^{(j, m)},\Sigma_{k}^{(j)}, p_{\mathrm{o}}\right) \!+ \!\beta_{\mathrm{coll}, k}\!\left(x_k^{(j)},\Sigma_k^{(j)} \right). \\
	\end{split}
	\label{eq:robust-linearized-constr}
\end{equation}
where ${x}^{(j, m)}_k$ denote the state variables at the outer \gls{zoro} iteration $j$ and the inner iteration $m$.
As in solving the nominal \gls{ocp}, the derivatives can be straightforwardly computed without computing the sensitivities of the lower-level maximizers $\gammaShape^*$ and $\gammaTransNoise^*$ with respect to $x_k$.

The algorithm for solving the robust \gls{ocp}~\eqref{eq:robust-ocp-def-wReformulation} by approximation is summarized in Algorithm~\ref{alg:robust-collision-free-trajectory}.
Note that the inner iteration for solving the \gls{sip} subproblem (Line~\ref{alg-line:robust-ocp-sip-subproblem-start} to~\ref{alg-line:robust-ocp-sip-subproblem-end}) is performed only once after every uncertainty propagation (Line~\ref{alg-line:robust-ocp-propagation-start} and~\ref{alg-line:robust-ocp-propagation-end}).
Although it is possible to solve for more iterations,
it is computationally inefficient to let the inner, relatively time-intensive, iteration converge when the outer \gls{zoro} iteration is still far from convergence.
The \gls{nlp} subproblem,
which considers the uncertainty matrices as parameters, focuses on a subset of critical obstacles, and includes the locally-reduced constraints, is given by
\begin{mini!}|s|
	{\substack{{x}_0, \Compactcdots, {x}_N,\\
			{u}_0, \Compactcdots, {u}_{N\!\shortminus \!1}}}
	{\sum_{k=0}^{N\!-\!1} {L}_k\left({x}_k, {u}_k\right) + {L}_N\left({x}_N\right)}
	{\label{eq:robust-ocp-subproblem}}{}
	\addConstraint{{x}_{0}}{={\bar{x}}_{0},}{}
	\addConstraint{{x}_{k+1}}{=\psi_k({x}_k,\! {u}_k),\  k \!\in \! \mathbb{N}_{\left[0, N\shortminus 1 \right]},}{}
	\addConstraint{0 }{\geq h_{\mathcal{XU}}({x}_k, u_k) \! + \! \beta_{\mathcal{XU}} \!\left({x}_k^{(j)},{u}_k^{(j)}, \Sigma_k^{(j)} \right)\!, \nonumber }{}
	\addConstraint{}{\quad \; k \!\in \! \mathbb{N}_{\left[0, N\shortminus 1 \right]},}
	\addConstraint{0}{ \geq  h_{\mathcal{X}_N}({x}_N) + \beta_{\mathcal{X}_N}\!\left({x}_N^{(j)}, \Sigma_N^{(j)} \right), }{}
	\addConstraint{0\geq }{ \,   \check{\tilde{h}}_{\text{coll}}\!\left({x}_k; {x}_k^{(j)},\Sigma_{k}^{(j)}, p_{\mathrm{o}}\right) \!+ \!\beta_{\mathrm{coll}, k}\!\left(\! x_k^{(j)},\Sigma_k^{(j)} \right)\!, \nonumber}
	\addConstraint{}{ \, p_{\mathrm{o}} \in \mathcal{O}_{\mathrm{s}, k}^{(j)}, \ k \!\in \! \mathbb{N}_{\left[1, N \right]}}{\label{eq:robust-ocp-subproblem-linearized-constr}},
\end{mini!}
where the index $m$ for the inner iteration is dropped for compactness.
State-input and terminal constraints are robustified via a backoff reformulation (see~\eqref{eq:robust-ocp-tightened-constr-affine-reduced}) with backoff denoted by $\beta_{\mathcal{XU}}\!\left(\cdot\right)$ and $\beta_{\mathcal{X}_N}\!\left(\cdot\right)$, respectively.
Many alternative approaches for constraint robustification exist beyond the backoff reformulation~\cite{BenTal1998}, but a detailed discussion of these methods falls outside the scope of this paper.

\begin{algorithm}[t]
	\caption{Method for solving the robust \gls{ocp}~\eqref{eq:robust-ocp-def-wReformulation}}
	\label{alg:robust-collision-free-trajectory}
	\begin{algorithmic}[1]
		\Require Initial guess ${u}^{(0)}_{0}, \Compactcdots, {u}^{(0)}_{N\shortminus 1}$,
		${x}^{(0)}_{0}, \Compactcdots, {x}^{(0)}_{N}$,
		and ${\mathcal{O}}^{(\shortminus 1)}_{\mathrm{s}, 1}, \Compactcdots, {\mathcal{O}}^{(\shortminus 1)}_{\mathrm{s}, N}$
		\For{$ j = 0, \Compactcdots, \text{MAXITER} $}
		\BeginBox[draw=gray, dashed, line width=0.8pt, inner sep=1.5pt, outer sep=2pt]
		\For{$k=0, \Compactcdots, N$} \label{alg-line:robust-ocp-propagation-start}
		\State $\Sigma_{k+1} \gets \Phi_k({x}_k,\! {u}_k, \! \Sigma_k) $ \label{alg-line:robust-ocp-propagation-end}
		\EndFor
		\EndBox
		\For{$k=1, \Compactcdots, N$} \label{alg-line:robust-ocp-sip-subproblem-start}
		\State $\mathcal{O}_{\mathrm{s}, k}^{(j)} \gets$ \Call{UpdateObsSubset}{${x}_k^{(j)} , \mathcal{O}^{(j\shortminus 1)}_{\mathrm{s}, k}$}
		\BeginBox[draw=gray, dashed, line width=0.8pt, inner sep=1.5pt, outer sep=2pt]
		\State $\beta_{\mathcal{XU}} \gets$ comp. state-input constr. backoff \Comment{\eqref{eq:compute-backoff}}
		\State $\beta_{\mathrm{coll}, k} \gets$ comp. coll. avoid. constr. backoff
		\Comment{\eqref{eq:robust-coll-backoff-def}}
		\For{$p_{\mathrm{o}} \in \mathcal{O}^{(j)}_{\mathrm{s}, k}$}
		\State $\gammaShape^*\left(\cdot\right), \gammaTransNoise^*\left(\cdot \right) \gets$ solve lower-level problem~\eqref{eq:robust-lower-level-qcqp-def}
		\EndFor
		\EndBox
		\EndFor
		\State ${u}^{(j+1)}_{0},\! \Compactcdots, {u}^{(j+1)}_{N\shortminus 1}, {x}^{(j+1)}_{0},\! \Compactcdots, {x}^{(j+1)}_{N}  \! \gets \! $ solve~\eqref{eq:robust-ocp-subproblem}  \label{alg-line:robust-ocp-sip-subproblem-end}
		\State $\Delta = \max(\|{u}^{(j+1)}_{0} - {u}^{(j)}_{0}\|_{\infty}, \Compactcdots,  \|{x}^{(j+1)}_{N} - {x}^{(j )}_{N} \|_{\infty})$
		\If{$ \Delta  \leq \epsilon_{\text{cvg}}$}
		\State CONVERGED $\gets 1$, \Return ${u}_0^{(j+1)},\Compactcdots, {u}^{(j+1)}_{N\shortminus 1}$
		\EndIf
		\EndFor
		\State \Return ${u}_0^{(\text{MAXITER}+1)},\Compactcdots, {u}^{(\text{MAXITER}+1)}_{N\shortminus 1}$
		\State \LeftComment{The gray boxes highlight the additional or differing elements compared to the algorithm for the nominal \gls{ocp}.}
	\end{algorithmic}
\end{algorithm}
The approximations introduced in Algorithm~\ref{alg:robust-collision-free-trajectory} for solving the robust \gls{ocp}~\eqref{eq:robust-ocp-def-wReformulation} arise from the following aspects:
\begin{enumerate}[label=(\alph*)]
	\item The zero-order update of the uncertainty trajectory approximates certain gradients to be zero (see Sec.~\ref{sec:zoRO-preliminary}). While this impacts optimality, it does not affect feasibility and can be remedied by a gradient correction~\cite{Feng2020Adjoint};
	\item The obstacle subset update does not take state uncertainties into account,
	i.e., it identifies obstacles closest to the nominal robot rather than those nearest to the space potentially occupied by the uncertain robot.
	This choice compromises the constraint satisfaction guarantee.
\end{enumerate}

%% file: sections/5-mobileRobot.tex
\section{Real-Time-Feasible Optimal and Model Predictive Control of a Mobile Robot in 2D}
\label{sec:mobile-robot-ocp-mpc}

In this section, we first define the objective function of the \gls{ocp}.
We then detail the method for identifying the closest obstacles and updating the obstacle subsets.
Subsequently, the model of the mobile robot in the \gls{ocp}, which includes an identified model of the robot dynamics, is described.
Finally, we present several software implementation details, including techniques that accelerate computation when solving the \glspl{ocp} in an \gls{mpc} scheme.

\subsection{Reference Trajectory Tracking}
\label{sec:ref-traj-tracking}
In this robot navigation task,
    the \gls{ocp} is formulated as a reference trajectory tracking problem.
Note that the reference trajectories are only used in the objective function.
The proposed collision-avoidance method is independent of the reference trajectories.
The reference trajectory is generated by first fitting a spline to a series of waypoints and then computing a kinematically time-optimal trajectory that accounts for the path geometry and robot actuator constraints~\cite{Bobrow1985}.
Note that the reference trajectories may lead to collisions,
    as the spline fitting step ignores obstacle information and path-finding algorithms that generate the sequence of waypoints often assume a simplified circular shape that under-approximates the actual robot shape for improved tractability and reduced conservativeness.

The controller tracks the reference trajectory while taking into account the state-input and collision-avoidance constraints.
The driven path of the robot and the velocity by which the robot follows the path are modified when necessary.
To track the reference trajectories, the stage and terminal costs of the OCP are defined as weighted squared tracking errors.
The deviations in the robot kinematic state and the velocities, as well as the control efforts, are penalized.

\subsection{Closest Obstacle Identification and Obstacle Subset Update}
To efficiently identify the closest obstacles and update the obstacle subset,
    the Euclidean \textit{signed} distance transform is employed.
This transform can be efficiently computed using standard algorithms such as the dead reckoning method~\cite{Grevera2004}.
In addition to the signed distance,
    it also provides a mapping from each grid coordinate to the location of the corresponding closest obstacle.
The key idea of the closest-obstacle identification is to find the set of $\gammaShape$ that has the minimum value of the signed distance and then obtain the corresponding closest obstacles using the closest-obstacle map.

For the definition of the Euclidean {signed} distance,
    let us consider the actual shape of an obstacle, denoted by $\mathcal{O}_{\text{AS}}$, alongside the point obstacle set $\mathcal{O}$.
The definition is given by
\begin{equation}
	\signedDist(p; \mathcal{O}, \mathcal{O}_{\text{AS}}) :=
	\left\lbrace
	\begin{matrix*}[l]
		-\!\min_{p_{\mathrm{o}} \in \mathcal{O}}\left\| p-p_{\mathrm{o}} \right\|_2, & \text{if}\; p \in \mathcal{O}_{\text{AS}},\\
		\min_{p_{\mathrm{o}} \in \mathcal{O} }\left\| p-p_{\mathrm{o}} \right\|_2,  & \text{otherwise}.
	\end{matrix*}
	\right.
\end{equation}
The closest obstacle to the point $p\in \bbR^2$ is defined by
\begin{equation}
	\closestObs(p; \mathcal{O} ) := \arg\min_{p_{\mathrm{o}} \in \mathcal{O} } \left\| p-p_{\mathrm{o} } \right\|_2.
\end{equation}

The algorithm for identifying the closest obstacle and updating the obstacle subset is outlined in Algorithm~\ref{alg:identify-min-dist-obs}.
For simplicity of notation, the time index $k$ for the state and the obstacle set is omitted in the remainder of this subsection.
Let
\begin{equation}
	\mapRToW(x, \gammaShape) := p_{\mathrm{c}} \!\left({x}\right) + R\left( {x} \right) \gammaShape.
\end{equation}
The first step of the algorithm is to perform a grid search along the boundary of the set $\gammaSetShape$:
\begin{equation}
	\underline{\text{sd}}_{\mathrm{GS}}(x) := \min_{\gammaShape \in \discSet \left( \bdSet\gammaSetShape \right)  }
	\signedDist \left( \mapRToW(x, \gammaShape) \right),
	\label{eq:min-sd-grid-search}
\end{equation}
where $\discSet $ and $\bdSet$ denote the boundary of a set and a discretization of a set, respectively.
The values of $\gammaShape$
whose corresponding distances are no greater than $\underline{\text{sd}}_{\mathrm{GS}}(x)$ plus a small positive value $\epsilon_{\text{GS}} > 0$ are collected in a set:
\begin{equation}
	\begin{split}
		{\Gamma}_{\mathrm{shp}, \mathrm{GS}}^*(x) \! := \!
		\bigl\lbrace  & \gammaShape  \mid \gammaShape \!\in\! \discSet  \left( \bdSet  \gammaSetShape \right) , \\
		& \; \signedDist \left(\mapRToW(x, \gamma_{\mathrm{shp}}) \right) \leq \underline{\text{sd}}_{\mathrm{GS}}(x) +\epsilon_{\text{GS}} \bigr\rbrace.
	\end{split}
	\label{eq:worst-gamma-grid-search}
\end{equation}

\begin{algorithm}[t]
	\caption{Method for updating the obstacle subset}
	\label{alg:identify-min-dist-obs}
	\begin{algorithmic}
		\Require Predefined threshold $\epsilon_{\text{inside}} \geq 0$
		\Function{UpdateObsSubset}{${x}, \mathcal{O}^{(j\shortminus 1)}_{\mathrm{s}}$}
		\State $\underline{\text{sd}}_{\mathrm{GS}}(x), {\Gamma}_{\mathrm{shp}, \mathrm{GS}}^*(x) \gets$ grid search \Comment{\eqref{eq:min-sd-grid-search}, \eqref{eq:worst-gamma-grid-search}}
		\If{$\underline{\text{sd}}_{\mathrm{GS}}(x) \geq \epsilon_{\text{inside}}$}
		% \State OBSINPOLYGON $\gets 0$
		\State $\mathcal{O}^{(j)}_{\mathrm{s}} \gets$ add close-by obstacles
		\Comment{\eqref{eq:min-dist-obs-fine-search}, \eqref{eq:update-obstacle-subset}}
		\State \Return $\mathcal{O}^{(j)}_{\mathrm{s}}$
		\Else
		\State impose linearized signed distance constraints
		\State \Return $ \mathcal{O}^{(j\shortminus 1)}_{\mathrm{s}}$
		\EndIf
		\EndFunction
		\State \LeftComment{see Table~\ref{table:mobile-robot-params-split} for the value of the parameters used.}
	\end{algorithmic}
\end{algorithm}

Under the assumption that no obstacle is entirely contained in the polygon,
   the value of $\underline{\text{sd}}_{\mathrm{GS}}(x)$ being positive indicates that all point obstacles are outside of the polygon,
   and the minimum value of $\signedDist( \mapRToW(x, \gammaShape))$ for $\gammaShape \in \gammaSetShape$ is achieved on the polygon edges.
Performing the grid search over the boundary of the polygon eliminates the need of the gradient computation of the distance function and also avoids the risk of getting stuck at local minima.
With the set ${\Gamma}_{\mathrm{shp}, \mathrm{GS}}^*$ obtained, we query the closest obstacles:
\begin{equation}
	{\Gamma}_{\mathrm{shp}}^{*}(x) \! :=
	\argmin_{\gammaShape \in {\Gamma}_{\mathrm{shp}, \mathrm{GS}}^*}
	\left\Vert \mapRToW\left( x, \gamma_{\mathrm{shp}}\right)
	\!-\!  \closestObs \! \left(  \mapRToW\left( x, \gamma_{\mathrm{shp}}\right) \right) \right\Vert_2\!.
	\label{eq:min-dist-obs-fine-search}
\end{equation}
The obstacle subset is then updated as follows:
\begin{equation}
	\begin{split}
		\mathcal{O}^{(j)}_{\mathrm{s}} \! = \!
		\mathcal{O}^{(j\shortminus 1)}_{\mathrm{s}} \cup \Bigl\{ \closestObs \left(  \mapRToW \!\left({x}, \gammaShape\right)\right) \mid \gammaShape \! \in \! {\Gamma}_{\mathrm{shp}}^{*}(x), \\
		\left\|\mapRToW\left( x, \gamma_{\mathrm{shp}}\right)
		\!-\! \closestObs \left(  \mapRToW\left( x, \gamma_{\mathrm{shp}}\right) \right) \right\|_2 \leq \epsilon_{\text{cl}}  \Bigr\},
	\end{split}
	\label{eq:update-obstacle-subset}
\end{equation}
where $\epsilon_{\text{cl}} > r_{\mathrm{shp}}$ is a predefined threshold for a point obstacle to be considered close to the polygon.

When the value of $\underline{\text{sd}}_{\mathrm{GS}}(x)$ is negative,
the polygon and the obstacles intersect.
In this case, the imposed constraints are that, for each $\gammaShape \in {\Gamma}_{\mathrm{shp}, \mathrm{GS}}^*$, the function $\signedDist(p; \mathcal{O}, \mathcal{O}_{\text{AS}})$ linearized at $p=\mapRToW\left( x, \gamma_{\mathrm{shp}}\right)$ should be no smaller than a small positive value, e.g., five times the point-obstacle resolution in our implementation.
Note that this does not affect local convergence properties since feasibility of the OCP implies that the Euclidean signed distance evaluates positively along the entire robot trajectory.

\begin{remark}
	\label{rem:inner-grid-search-complete-penetration}
   {For environments containing small or thin obstacles that might be completely contained within the polygon shape,
       the search space of the grid search can be extended to include grid points inside the polygon.
   While boundary grids must be fine enough to detect the grid with the minimum signed distance and thereby identify the closest obstacle,
      interior grids can be coarser.
   For interior grid intervals smaller than the obstacle size,
      detection of at least one point of the penetrating obstacle is guaranteed, enabling separation between the robot polygon and the obstacle.}
\end{remark}

\subsection{System Dynamics of a Mobile Robot}
\label{sec:system-dynamics-model}
For control of mobile robots,
    modeling system dynamics within the~\gls{ocp}, including for example dynamics with overshoot,
    facilitates fast and safe trajectory execution.
Consider a linear system model for the robot dynamics
    with its state denoted by $\nu \in \bbR^{n_{\nu}}$.
The forward and angular command velocities, denoted by $v_{\mathrm{cmd}}$ and $\omega_{\mathrm{cmd}}$, serve as the system inputs.
The nominal system dynamics is described as
\begin{equation}
	\dot{{\nu}} = A_{\nu} {\nu} + B_{\nu} \begin{bmatrix}
		v_{\mathrm{cmd}} \\ \omega_{\mathrm{cmd}}
	\end{bmatrix},
\label{eq:nu-system-dynamics}
\end{equation}
    where $A_{\nu} \in \bbR^{n_\nu \times n_\nu}$ and $B_{\nu} \in \bbR^{n_\nu \times 2}$.
The system measurements consist of the forward and angular velocities disturbed by measurement noise:
\begin{equation*}
    \begin{bmatrix}
			v_{\mathrm{real}} \\ \omega_{\mathrm{real}}
	\end{bmatrix}
	\! = \! C_{\nu} {\nu} + D_{\nu} \begin{bmatrix}
			v_{\mathrm{cmd}} \\ \omega_{\mathrm{cmd}}
	\end{bmatrix},\
	\begin{bmatrix}
		v_{\mathrm{meas}} \\ \omega_{\mathrm{meas}}
	\end{bmatrix} \! = \! \begin{bmatrix}
			v_{\mathrm{real}} \\ \omega_{\mathrm{real}}
	\end{bmatrix} + w_{\mathrm{meas}}.
\end{equation*}
The system matrices $A_{\nu}, B_{\nu}, C_{\nu}, D_{\nu}$ can be identified offline using, for instance, the N4SID subspace algorithm~\cite{VanOverschee1994}.

The whole system state consists of the robot kinematic state $x_{\text{kin}}$, the command velocity, and the robot dynamic state~$\nu$:
\begin{equation}
	x = \begin{bmatrix}
		x_{\text{kin}}^{\top} & v_{\mathrm{cmd}}& \omega_{\mathrm{cmd}}& {\nu}^{\top}
	\end{bmatrix}^{\top}.
\end{equation}
The system dynamics is given by
\begin{equation}
	\frac{\mathrm{d}x(t)}{\mathrm{d}t} = 
	\begin{bmatrix}
		f_{\text{kin}}\left( x_{\text{kin}}, C_{\nu}{\nu}  + D_{\nu} \begin{bmatrix}
			v_{\mathrm{cmd}} \\ \omega_{\mathrm{cmd}}
		\end{bmatrix}\right) \\ a \\ \alpha \\
		A_{\nu} \nu + B_{\nu} \begin{bmatrix}
			v_{\mathrm{cmd}} \\ \omega_{\mathrm{cmd}}
		\end{bmatrix}
	\end{bmatrix},
\end{equation}
where $a$ and $\alpha$ are the forward acceleration and angular acceleration, respectively,
and are considered as the system input, namely,
$ u = \begin{bmatrix}
		a & \alpha
	\end{bmatrix} ^{\top}$.

\subsection{Details of Software Implementation}
\label{sec:details-software-implementation}

In this subsection, we present several software implementation details.
Our implementation solves the \gls{nlp} subproblems~\eqref{eq:nominal-ocp-subproblem} and~\eqref{eq:robust-ocp-subproblem}  with an \gls{sqp}-type solver in acados~\cite{Verschueren2021} and uses HPIPM~\cite{Frison2020} as the QP solver.
For real-world experiments, the \gls{mpc} is implemented as a Nav2 controller plugin~\cite{Macenski2020} written in C++.

\paragraph{Warm starting of obstacle subset}
When computing the solution of the first \gls{ocp},
the obstacle set is initially empty.
For subsequent steps, we shift forward the part of the obstacle subset~$\mathcal{O}_{\mathrm{s}, k}$ whose constraints are active.

\paragraph{Termination at early iterates}
When solving \glspl{ocp} recursively in an \gls{mpc} scheme, the \gls{sqp} optimization process is often terminated early in the \gls{qp} iterations to reduce computation time~\cite{Diehl2005}.
Similarly, for real-time capability, the proposed controller limits the number of iterations, i.e., MAXITER in Algorithm~\ref{alg:nominal-collision-free-trajectory} and Algorithm~\ref{alg:robust-collision-free-trajectory}, to a relatively small number in real-world experiments,
see Table~\ref{table:mobile-robot-params-split} for the values used in practice.

\paragraph{Constraints with slack variables and exact penalty}
In practice, large disturbances can occasionally lead to infeasibility of state-dependent constraints of \glspl{ocp}.
Additionally, the linearization of collision-avoidance constraints (see~\eqref{eq:nominal-ocp-subproblem-linearized-constr} and~\eqref{eq:robust-ocp-subproblem-linearized-constr}) may lead to an infeasible subproblem when the \gls{ocp} is only barely feasible.
To ensure the feasibility of the subproblems, the collision-avoidance constraints are implemented as constraints with slack variables.
For example, the constraints in the nominal \gls{ocp} are given by
\begin{equation}
	s_i \geq  \check{h}_{\text{coll}}\left({x}_k;{x}_k^{(j)}, p_{\mathrm{o}}\right)
	\ \text{and} \ s_i \geq 0.
\end{equation}
Weighted $l_1$ and $l_2$ penalties are added to the objective function.
Note that such constraints are intrinsically different from smooth penalty functions.
When the penalties of the slack variables are sufficiently large and the problem is feasible,
    the solution of the problem with slackened constraints corresponds to that with hard constraints~\cite{Han1979}.
In the infeasible case, it will return a solution that minimizes constraint violation, which can be used as a soft recovery when the constraint violation is acceptable.

\begin{figure*}[th]
	\captionsetup{skip=-2pt}
	\centering
	\includegraphics[width=0.98\linewidth]{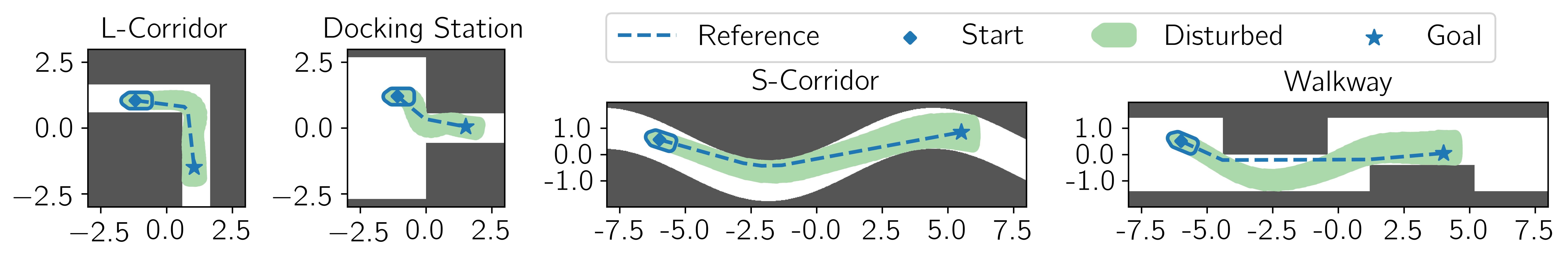}
	\caption{Four world maps used for numerical evaluation.
		The dashed lines are the reference paths.
		The green tube represents the union of the occupied space across all discrete time steps, as given by the uncertainty model in~\eqref{eq:robust-coll-constr-separateRotPos} and evaluated using a converged OCP solution.
		The resolution of the point obstacles is 0.02\,m.
		See Table~\ref{table:mobile-robot-config-split} for the number of points in each environment.}
	\label{fig:test_world_maps}
\end{figure*}

\paragraph{Initial State Uncertainty and Additive Disturbance}
The initial state uncertainty encompasses both the uncertainty in the kinematic states $x_{\text{kin}}$ and the robot dynamic state $\nu$.
The uncertainty of the kinematic state $x_{\text{kin}}$ is obtained from the localization module.
The robot dynamic state is estimated using a Kalman filter, from which we extract the state covariance matrix.
For the additive disturbance, we consider a model-plant mismatch in the system dynamics~\eqref{eq:nu-system-dynamics} as the primary source.
The corresponding covariance values are derived from the error covariance matrix of the system identification.
To achieve a desired level of robustness in constraint robustification, these covariance matrices can be scaled by a predefined factor.

%% file: sections/6-results.tex
\section{Numerical Evaluation and Physical Experiments with Mobile Robot}
\label{sec:results}

This section presents the results of the numerical evaluation and real-world experiments of the mobile robot control problem presented in the previous section.
For the experiments, we consider a differential-drive robot:
\begin{equation} \label{eq:robot_system_dynamics}
	\begin{split}
	x = &\Big[
			\underbrace{p_x \quad p_y \quad \theta}_{x_{\text{kin}}} \quad v_{\mathrm{cmd}}\quad \omega_{\mathrm{cmd}}\quad {\nu}^{\top}
		\Big]^{\top},\\
    \frac{\mathrm{d}x_{\text{kin}}(t)}{\mathrm{d}t}
    = &\begin{bmatrix}
    	\cos\theta & 0 \\ \sin\theta & 0 \\ 0 & 1
    \end{bmatrix}\left( C_{\nu}{\nu}  + D_{\nu} \begin{bmatrix}
    v_{\mathrm{cmd}} \\ \omega_{\mathrm{cmd}}
\end{bmatrix}\right) .
\end{split}
\end{equation}
The continuous-time system dynamics~\eqref{eq:robot_system_dynamics} are discretized using an explicit fourth-order Runge-Kutta integrator.
The configuration of the robot shape is listed in Table~\ref{table:mobile-robot-config-split}.

\begin{figure*}[t]
	\begin{minipage}[t]{0.325\linewidth}
		\captionsetup{skip=-4pt}
		\centering
		\includegraphics[width=\linewidth]{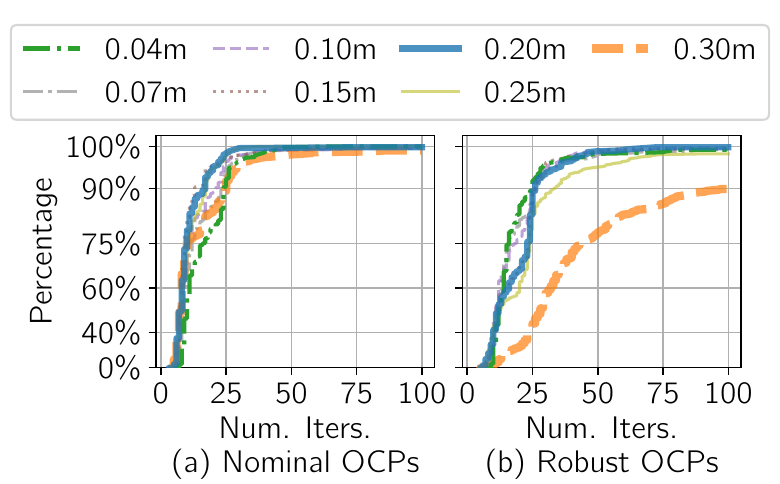}
		\caption{Optimal control of a mobile robot: The percentages of the OCPs that converge within a certain number of iterations.
			The radius of the padding circle $r_{\mathrm{shp}}$ is varied.
			The reference trajectories are carefully chosen such that the collision-avoidance constraints are active for different $r_{\mathrm{shp}}$.}
		\label{fig:ocp_number_iterations_varying_radius}
	\end{minipage}
    \hfill
    \begin{minipage}[t]{0.325\linewidth}
    	\captionsetup{skip=-4pt}
    	\centering
    	\includegraphics[width=\linewidth]{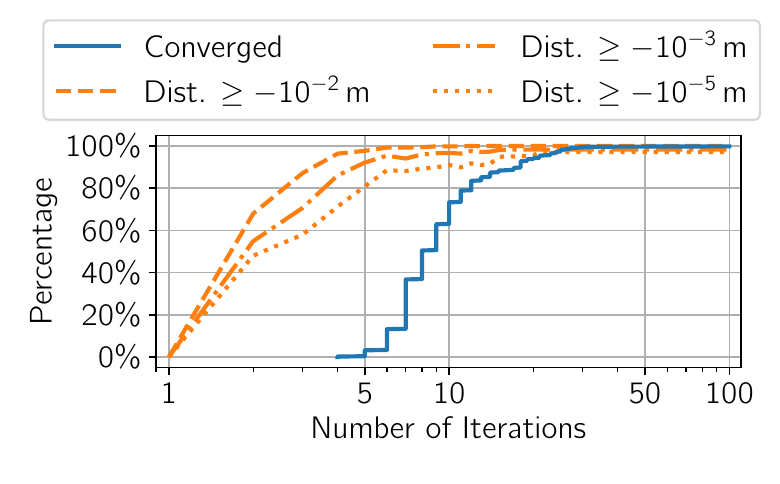}
    	\caption{The orange lines show the percentages of the nominal \glspl{ocp}
    		of which for all the states of the current iterate the signed distances from the robot to all obstacles are no small than certain thresholds.
    		The blue line plots the percentage of the converged \glspl{ocp} within a certain number of iterations.
    	}
    	\label{fig:nominal-ocp-num-iterations}
    \end{minipage}
    \hfill
    \begin{minipage}[t]{0.325\linewidth}
    	\captionsetup{skip=-4pt}
    	\centering
    	\includegraphics[width=\linewidth]{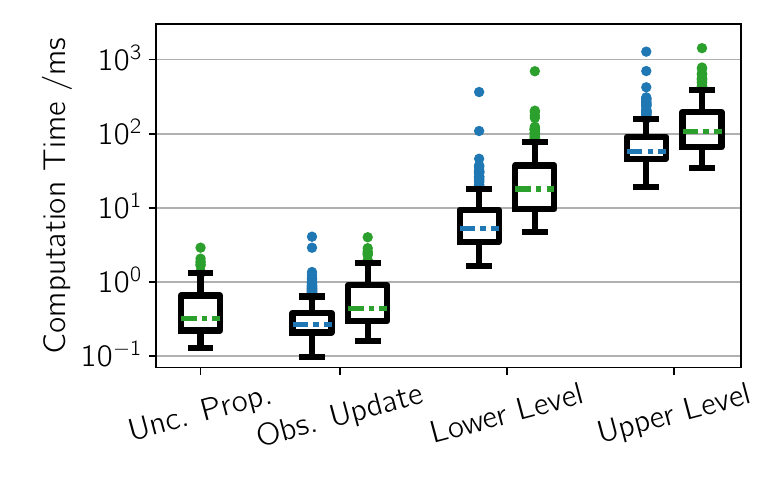}
    	\caption{Computation time of solving the OCPs for the mobile robot.
    		The blue colors correspond to the nominal OCPs and the green colors represent the robust OCPs.
    		The dashdotted line is the median.
    		`Unc. Prop.' refers to the uncertainty propagation.
    	}
    	\label{fig:dilated-polygon-ocp-computation-time}
    \end{minipage}
\end{figure*}

\subsection{Numerical Evaluation of Solving \glspl{ocp}}
\label{sec:results-numerical-evaluation-ocp}
The numerical evaluation is conducted on a laptop with an Intel i7-11850H processor and 32GB of RAM.
We use four maps for evaluation: `L-Corridor', `Docking Station', `S-Corridor', and `Walkway' (see Fig.~\ref{fig:test_world_maps}).
For each map, we have 135 test cases with different starting poses and velocities with the reference paths generated using the Theta* algorithm~\cite{Daniel2010}.
Examples of the starting poses and reference paths are plotted in Fig.~\ref{fig:test_world_maps}.
The padding radius is 0.2\,m unless otherwise specified.
The initialization of each \gls{ocp} is determined by the reference trajectory.

\paragraph{Iterations required for convergence with varying padding radius}
\label{sec:results-varying-padding-radius}
We evaluate the number of iterations required for convergence (see Fig.~\ref{fig:ocp_number_iterations_varying_radius}).
The convergence criterion is set as the $l_{\infty}$ norm of the change in the state-input trajectory getting below $10^{-6}$.
Seven different values of the padding radii ranging from 0.04\,m to 0.30\,m are tested with the polygon shape and the free space of the maps remaining constant.
The resolution of the point obstacles is 0.02\,m.
The reference trajectories are carefully chosen to be sufficiently close to the obstacles such that the collision-avoidance constraints are active for different values of the padding radius.

Overall, the robust \glspl{ocp} require more iterations to converge than the nominal \glspl{ocp}.
While the nominal \glspl{ocp} with a small padding radius of 0.04\,m shows an increase in the number of iterations,
    this is not evident in the robust \glspl{ocp}.
This difference can be attributed to the backoff in the robust \glspl{ocp}, which compensates for the rotational uncertainties and implicitly enlarges the padding.
A noticeable degradation in convergence is observed for a padding radius of 0.30\,m,
    especially for the robust \glspl{ocp}.
This decline in convergence performance occurs because the available free space becomes too narrow for the robot of that size to admit a robustly collision-free trajectory.

\paragraph{Progress of constraint satisfaction over the iterations}
For the nominal \glspl{ocp}, we evaluate the violation of collision-avoidance constraints with respect to all obstacles throughout the iterations.
Let $\zeta$ represent the right of the inequality constraint~\eqref{eq:nominal-ocp-collision-constraints} minus the padding radius $r_{\mathrm{shp}}$:
\begin{equation}
	\zeta({x}, \gammaShape, p_{\mathrm{o}}) := - r_{\mathrm{shp}} +
		\left\| p_{\mathrm{c}}({x}) + R\left({x}\right) \gammaShape - p_{\mathrm{o}} \right\|_2 \!.
\end{equation}
Given a robot state $x$, the minimum signed distance between the robot and all point obstacles is given by:
\begin{equation}
	\min_{\gammaShape \in \gammaSetShape, p_{\mathrm{o}} \in \staticObsSet}
	\zeta({x}, \gammaShape, p_{\mathrm{o}}).
	\label{eq:eval-nominal-constr-violation}
\end{equation}
At each iteration $j$, we compute the minimum value of~\eqref{eq:eval-nominal-constr-violation} for all states in the state trajectory of the current iterate.
The percentages of the \glspl{ocp} whose minimum values are no smaller than certain thresholds are plotted in Fig.~\ref{fig:nominal-ocp-num-iterations}.
After five iterations, approximately 92\% of the test cases achieve a constraint violation of at most $10^{\shortminus 3}$\,m.

\begin{remark}
In early terminations in \gls{mpc} schemes,
    a smaller degree of constraint violation can typically be expected compared to the results reported here, due to improved initialization of the~\glspl{ocp} using the (approximate) solutions at previous time instants.
\end{remark}

\begin{figure}[t]
	\captionsetup{skip=-3pt}
	\centering
	\includegraphics[width=0.9\linewidth]{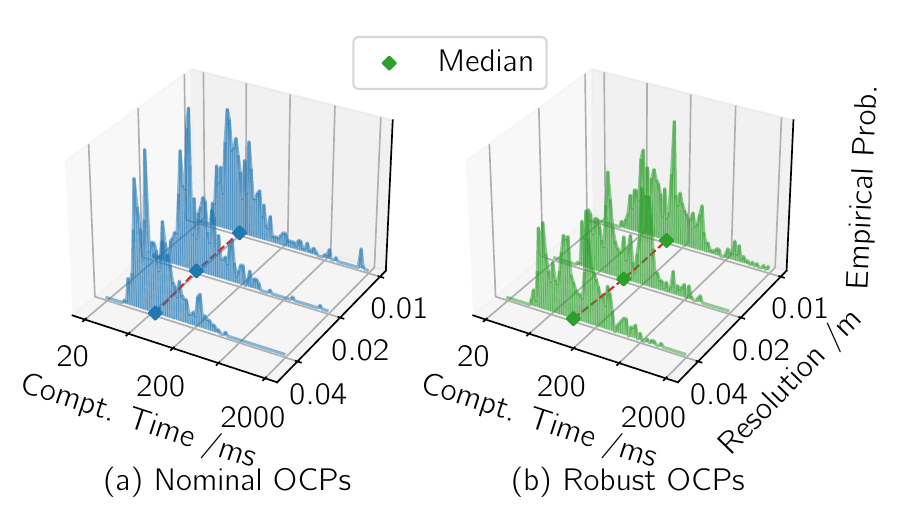}
	\caption{
		Empirical distribution of the total computation time for solving one \gls{ocp}
		    with different point-obstacle resolutions: 0.01\,m, 0.02\,m, and 0.04\,m.
	}
	\label{fig:ocp_computation_time_varying_resolution}
\end{figure}

\begin{figure}[t]
	\captionsetup{skip=-3pt}
	\centering
	\includegraphics[trim={0 0cm 0 0.45cm}, clip, width=0.9\linewidth]{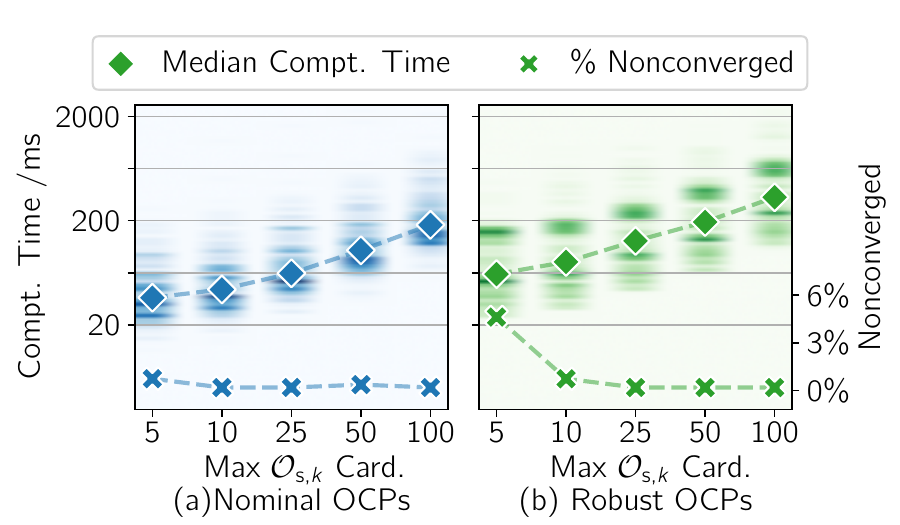}
	\caption{
		Empirical distribution of the computation time for solving one \gls{ocp} (heatmaps: darker shades indicate higher empirical probability)
		    and percentage of nonconverged cases (crosses)
		    for different maximum cardinalities of $\mathcal{O}_{\mathrm{s}, k}$.
	}
	\label{fig:ocp_computation_time_varying_num_constr}
\end{figure}

\vspace{-0.3\baselineskip}
\paragraph{Computation time}
The computation times taken to solve the \glspl{ocp} are plotted in Fig.~\ref{fig:dilated-polygon-ocp-computation-time}.
The upper-level optimization refers to solving the locally-reduced \gls{nlp} subproblems.
Across all four environments,
the lower level takes approximately 5\,ms at the median and 13\,ms at the 90th percentile for the nominal \glspl{ocp},
and 18\,ms at the median and 49\,ms at the 90th percentile for the robust \glspl{ocp}.
The upper level takes approximately 58\,ms at the median and up to 156\,ms at the 90th percentile for the nominal \glspl{ocp},
    and around 107\,ms at the median for the robust \glspl{ocp},
    with the 90th percentile computation time being around 253\,ms.
The computation time for updating the obstacle subsets and for disturbance propagation is negligible,
    each below  0.5\,ms at the median.

Solving the robust \glspl{ocp} requires more computation time compared to the nominal \glspl{ocp}
    as the upper level requires more iterations to converge,
    as shown in Fig.~\ref{fig:ocp_number_iterations_varying_radius}.
At the lower level,
    the increase is also due to the fact that solving a \gls{qcqp} problem is more time-consuming than solving a \gls{qp} problem.

{To evaluate the sensitivity of the computation time with respect to the point-obstacle resolution,
    we vary the resolution among 0.01\,m, 0.02\,m, and 0.04\,m, see Fig.~\ref{fig:ocp_computation_time_varying_resolution}.
For the nominal \glspl{ocp},
    the median computation time remains stable across the tested resolutions.
An increase in the occurrence of computation times exceeding 500\,ms can be observed at the tested finest resolution, i.e., 0.01\,m.
Regarding the robust \glspl{ocp},
    the median computation time increases from approximately 114\,ms to 129\,ms when the resolution changes from 0.04\,m to 0.02\,m.
The empirical probability of long computation times also becomes higher at the 0.01\,m resolution.}

We also examine how the maximum cardinality of the subset $\mathcal{O}_{\mathrm{s}, k}$ influences the solver performance.
The NLP subproblems contain a fixed number of collision-avoidance constraints equal to this limit, irrespective of the actual cardinality of the subsets.
As shown in Fig.~\ref{fig:ocp_computation_time_varying_num_constr}, raising this limit increases the solution time.
For the nominal \glspl{ocp}, increasing the maximum cardinality (for every time step $k$) from 25 to 50 raises the median computation time by roughly 66\%, and a further increase to 100 adds an additional 76\% overhead.
For the robust \glspl{ocp}, the corresponding increases are about 52\% and 72\%, respectively.
Lowering the maximum cardinality leads to a higher percentage of nonconverged cases.
The effect is especially pronounced for the robust \glspl{ocp} when the limit is reduced to as few as five point obstacles.

\subsection{Real-World MPC Experiments}
\label{sec:real-world-mpc}

\begin{figure*}[t]
	\begin{minipage}[b][][b]{0.325\linewidth}
    	\captionsetup{skip=-4pt}
		\captionsetup[subfloat]{captionskip=2pt}
		\centering
		\hfil
	\subfloat[L-corridor.
	]{\includegraphics[trim={0 0.28cm 0.0cm 0.0cm}, clip, height=2.5cm]{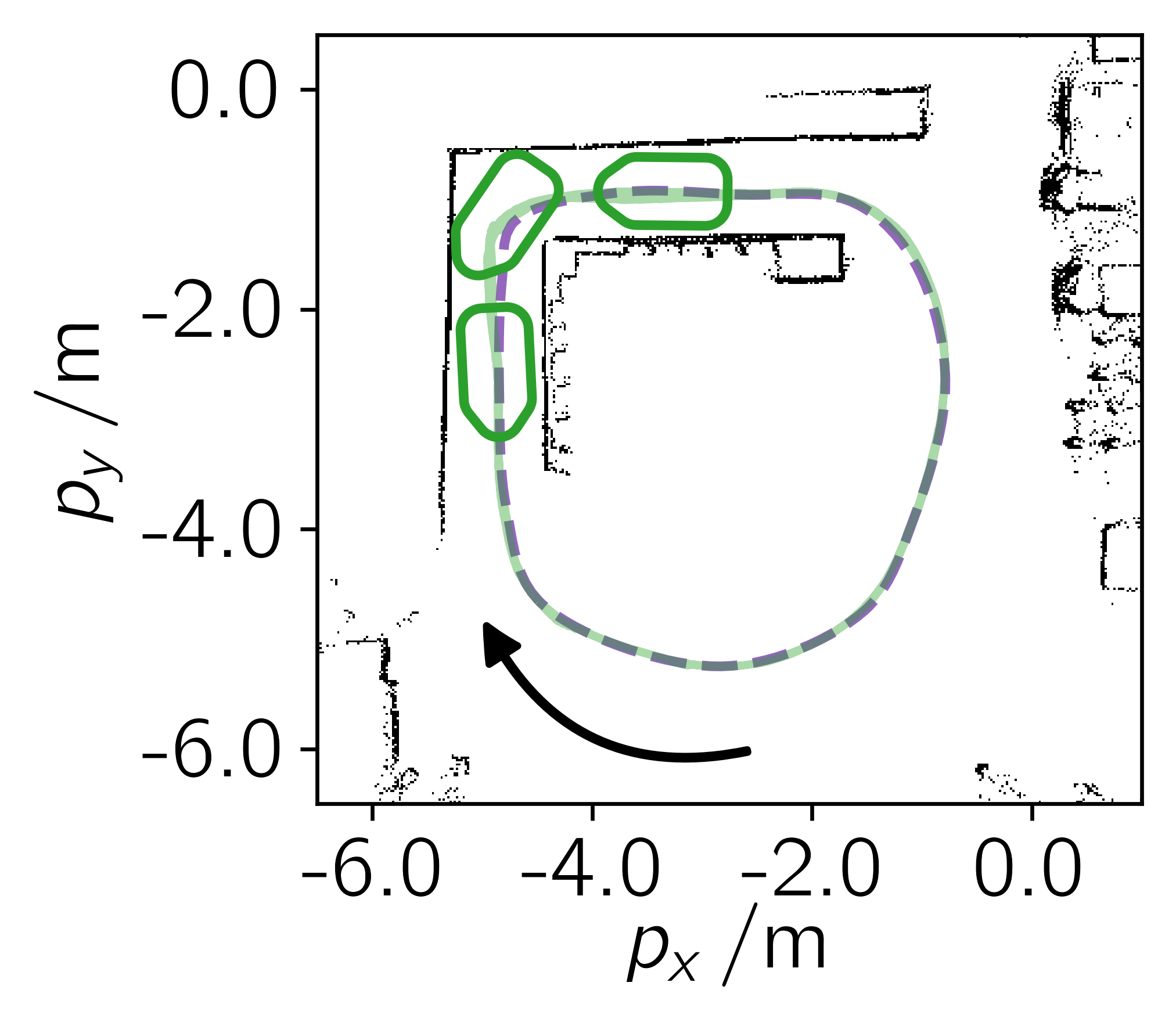}}
	\hfil
	\subfloat[S-corridor.
	]{\includegraphics[trim={1.0cm 0.25cm 0.0cm 0.0cm}, clip, height=2.5cm]{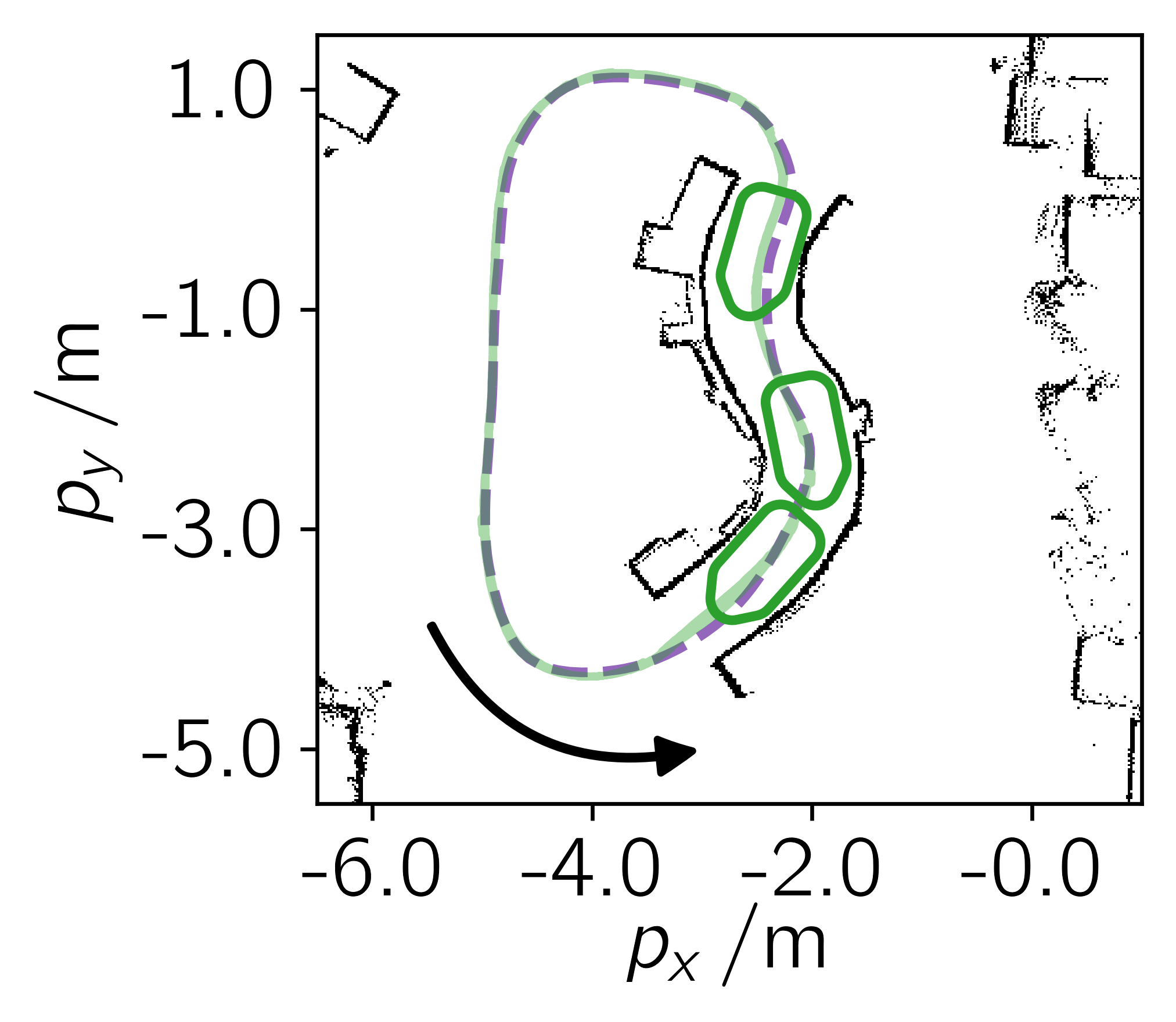}}
    \hspace{4ex}
	\caption{Real-world \gls{mpc} experiments.
    The robot is controlled by the robust controller.
    The dash lines are the reference paths and the semi-transparent lines are the robot trajectories over a four-minute experiment.
    Black arrows indicate the direction of motion.
	}
	\label{fig:real-robot-traj}
	\end{minipage}
    \hfill
    \begin{minipage}[b][][b]{0.325\linewidth}
    	\captionsetup{skip=-4pt}
    	\centering
	\includegraphics[width=0.98\linewidth]{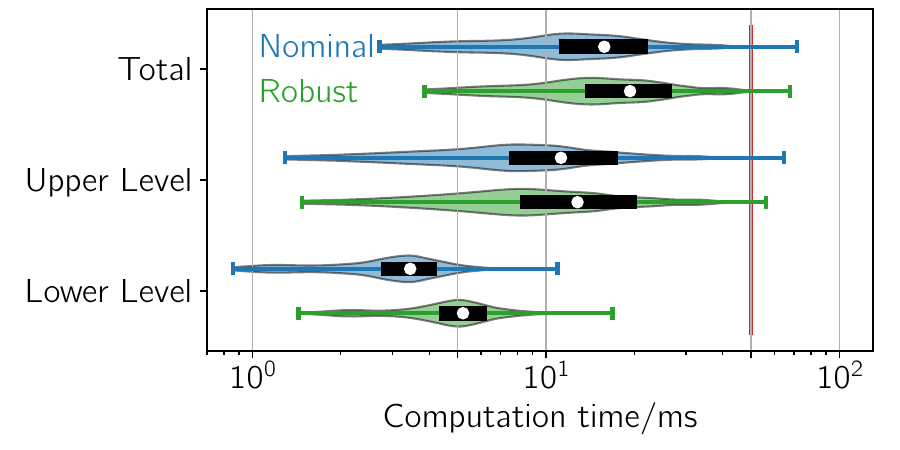}
	\caption{Real-world \gls{mpc} experiments: Computation time.
		The results include computation time of two environments and three velocity configurations.
		The white circle is the median. The black bar goes from the lower to the upper quartile.
		The sampling rate of the \gls{mpc} controllers is 50\,ms.}
	\label{fig:real_robot_nominal_and_robust_timings}
    \end{minipage}
	\hfill
    \begin{minipage}[b][][b]{0.325\linewidth}
    	\captionsetup{skip=-4pt}
    	\centering
	\includegraphics[width=0.98\linewidth]{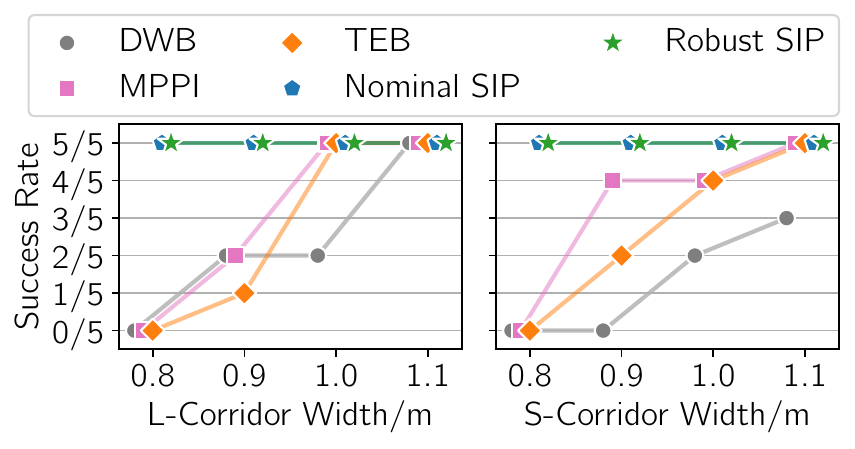}
	\caption{Real-world \gls{mpc} experiments: Success rates (out of five trials) for traversing corridors of varying widths.
	The plot compares our \gls{sip}-based nominal and robust \gls{mpc} controller with three state-of-the-art baselines, including the DWB, MPPI, and TEB controllers.}
	\label{fig:benchmark_success_passes}
    \end{minipage}
\end{figure*}

\begin{figure*}[t]
	\captionsetup{skip=-2pt}
	\includegraphics[width=0.98\linewidth]{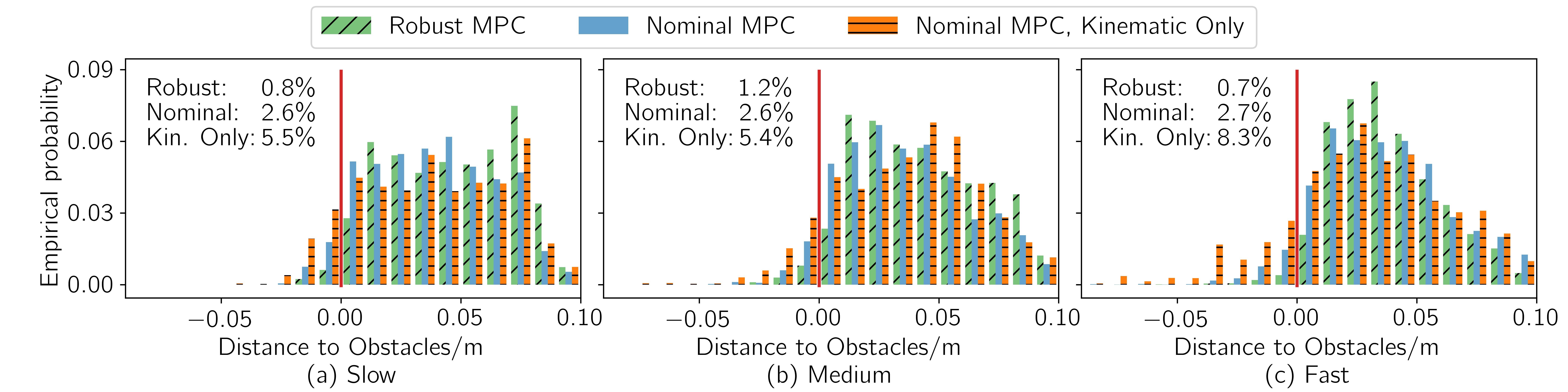}
	\caption{Real-world \gls{mpc} experiments: Distances to obstacles.
		The values shown in the upper left correspond to the percentages of negative distances.
		`Kinematic Only' refers to that system model does not include the system dynamics~\eqref{eq:nu-system-dynamics}.
		The empirical distribution of distances greater than 0.1\,m is truncated.
		See Table~\ref{table:varying-vel-acc-limits} for the detailed values of the velocity and acceleration bounds in the three settings: slow, medium, and fast.
	}
	\label{fig:min_dist_comparison}
\end{figure*}

% \begin{figure}[t]
% 	\captionsetup{skip=-6pt}
% 	\centering
	
% \end{figure}

We deploy the nominal and robust controllers within an \gls{mpc} framework in the real-world experiments.
The experiments are carried out on a Neobotix MP-500 differential-drive robot.
Its onboard computer is equipped with an Intel i7-7820EQ processor and 16GB of RAM.
We test the SIP-based MPC controllers in real-world L-corridor and S-corridor environments (see Fig.~\ref{fig:real-robot-traj}) and in three different settings: slow, medium, and fast settings (see Table~\ref{table:varying-vel-acc-limits} for the velocity and acceleration bounds).

\paragraph{Reference tracking and collision avoidance}
The robot tracks the reference paths and navigates through the narrow corridors\footnote{A video demonstrating the real-world experiment results is available as multimedia material.}.
At each solver call, the minimum signed distance to all point obstacles, as formulated in~\eqref{eq:eval-nominal-constr-violation}, is evaluated.
Fig.~\ref{fig:min_dist_comparison} shows the empirical distributions of the minimum signed distances for the nominal and robust controllers, each based on over four minutes of data in the two environments.
To the left of the red vertical line,
    the robust MPC demonstrates fewer and less severe constraint violations compared to the nominal MPC.
In the first bin to the right of the line, the robust MPC shows a notably lower occurrence frequency compared to the nominal MPC, which results from the additional safety margin gained through the constraint robustification.

\paragraph{Computation time}
The maximum number of iterations is set to six for the nominal controller and five for the robust controller.
The results reported in Fig.~\ref{fig:real_robot_nominal_and_robust_timings} include the computation times for the two environments and three speed settings.
The median of the total computation time is approximately 13\,ms for the nominal and 19\,ms for the robust controller.
The majority of the computation times are below 50\,ms, which is the sampling rate of the controller.
If the computation time exceeds 50\,ms, the controller will wait until a solution is computed,
    which is the default behavior in Nav2.
Alternatively, it is possible to manually implement alternatives such as utilizing the solutions of input trajectories obtained at previous time instants combined with multi-threading,
    which is more suitable for the robust setting.

\paragraph{Comparison to state-of-the-art controller}
We use the Nav2 implementation of the DWB controller\footnote{github.com/ros-navigation/navigation2/tree/main/nav2\_dwb\_controller} and the \gls{mppi} controller\footnote{github.com/ros-navigation/navigation2/tree/main/nav2\_mppi\_controller}, and the TEB controller\footnote{github.com/rst-tu-dortmund/teb\_local\_planner} as benchmark\footnote{The benchmarking controller parameters used in these experiments are archived at [https://doi.org/10.5281/zenodo.19116921] (located under sipoc\_mobile\_robot/sipoc\_mr\_nav2\_benchmark/params)}.
The DWB controller is an enhanced version of the \gls{dwa} and is the default option in the Nav2 stack.
The \gls{teb} controller~\cite{Roesmann2017} is an optimization-based method and uses quadratic penalties for constraint violations.
To evaluate the ability of different controllers to pass through narrow corridors,
    we vary the width of the corridor and test the number of successful passes.
A total of five tests are conducted for each setting.
The number of successful passes is reported in Fig.~\ref{fig:benchmark_success_passes}.
Although the DWB, \gls{mppi}, and \gls{teb} methods succeed in most test cases when the corridor is wide, they fail to pass through narrower corridors.
In contrast, the proposed method consistently manages to pass through the corridors of different widths.

% \begin{table}[h]
% 	\setlength{\tabcolsep}{8pt}
% 	\centering
% 	\caption{}
% 	\label{table:real-robot-state-of-the-art}
% 	\vspace{-5pt}
% 	\begin{tabular}{@{}l | l l  l l @{}}
%                 \hline \hline
% 		& \multicolumn{3}{l}{Width of the Corridor} \\
% 		\hline \hline
% 		L-Corridor & 1.1\,m & 1.0\,m& 0.9\,m& 0.8\,m \\
% 		\hline
% 		DWB &  5/5\,(\ding{52}) & 2/5 & 2/5  &  0/5\,(\ding{55})  \\
% 		MPPI & 5/5\,(\ding{52}) & 5/5\,(\ding{52}) & 2/5 & 0/5\,(\ding{55}) \\
% 		TEB & 5/5\,(\ding{52}) & 5/5\,(\ding{52}) & 1/5 & 0/5\,(\ding{55})\\
% 		Nominal SIP & 5/5\,(\ding{52}) &  5/5\,(\ding{52})& 5/5\,(\ding{52}) & 5/5\,(\ding{52})\\
% 		Robust SIP& 5/5\,(\ding{52}) &  5/5\,(\ding{52})& 5/5\,(\ding{52}) & 5/5\,(\ding{52})\\
% 		\hline \hline
% 		S-Corridor & 1.1\,m & 1.0\,m& 0.9\,m& 0.8\,m\\
% 		\hline
% 		DWB & 3/5 & 2/5 & 0/5\,(\ding{55}) &  0/5\,(\ding{55})  \\
% 		MPPI & 5/5\,(\ding{52}) & 4/5 & 4/5 & 0/5\,(\ding{55}) \\
% 		TEB & 5/5\,(\ding{52})  & 4/5 & 2/5 & 0/5\,(\ding{55}) \\
% 		Nominal SIP &5/5\,(\ding{52}) & 5/5\,(\ding{52})& 5/5\,(\ding{52})& 5/5\,(\ding{52})\\
% 		Robust SIP & 5/5\,(\ding{52})& 5/5\,(\ding{52})& 5/5\,(\ding{52})& 5/5\,(\ding{52})\\
% 		\hline \hline
% 	\end{tabular}
% \end{table}

\paragraph{Impact of incorporating the system dynamics}
The impact is evaluated by comparing two nominal controllers: one with and one without modeling the system dynamics in~\eqref{eq:nu-system-dynamics}, in both environments.
The latter case corresponds to setting $C_{\nu}$ and $D_{\nu}$ in~\eqref{eq:robot_system_dynamics} to a zero matrix and an identity matrix, respectively.
At every solver call,
    the minimum signed distance between the robot at the current time and all point obstacles, as formulated in~\eqref{eq:eval-nominal-constr-violation}, is evaluated.
As shown in Fig.~\ref{fig:min_dist_comparison},
    while the controller without the dynamics modeling shows slightly worse constraint violations at low speeds, its performance worsens significantly at higher velocities.
Not only do the constraint violations become more frequent, but the severity also increases, with maximum violations exceeding 0.08\,m.

%% file: sections/7-carSeat.tex
\section{Car Seat Placement in 3D}
\label{sec:car-seat-placement}

This section demonstrates the proposed method on a car seat placement task in Tesseract\footnote{github.com/tesseract-robotics/tesseract\_planning/blob/master/\\tesseract\_examples/src/car\_seat\_example.cpp}, focusing on \textit{nominal} collision avoidance at discrete time grids.
 We first consider an open-loop trajectory planning problem and compare the numerical performance with benchmark methods.
We then switch to an \gls{mpc} setting and show that the proposed method achieves a control sampling frequency of 10\,Hz.

A Yaskawa SIA20D robot arm (7-DOF) on a carriage rail is used to place a seat into a car, maintaining collision-free motion through the narrow car door (see Fig.~\ref{fig:car_seat_rviz}).
The seat is modeled as a union of ten polyhedra, each containing hundreds of faces.
For our experiment, we scale the polyhedra and pad each scaled polyhedron with a sphere of radius 0.08\,m.
The original seat geometry is completely contained within its scaled and padded version.
The car environment is represented by a point cloud sampled from the car mesh.

\subsection{Open-Loop Trajectory Planning}
\label{sec:car-seat-open-loop-planning}
Here we consider an open-loop trajectory planning problem.
Let $q \in \bbR^{8}$ denote the joint configuration, which concatenates the prismatic joint to the carriage rail and the revolute joints of the robot arm.
The system state is the joint configuration~$q$,
    and the system input is the velocity~$\dot{q}$.

\begin{figure}[t]
	\centering
	\captionsetup{skip=-2pt}
	\includegraphics[width=0.7\linewidth]{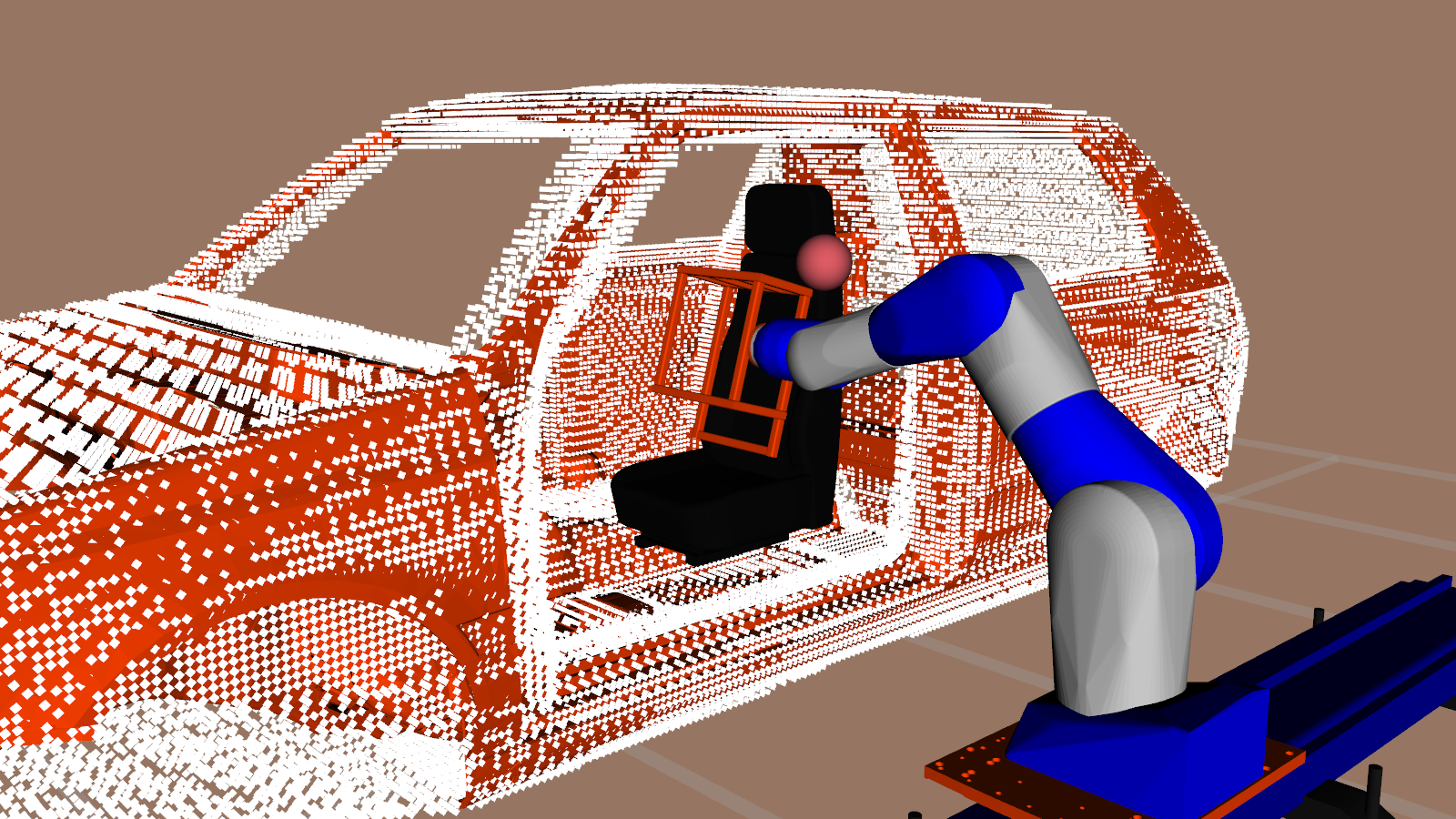}
    \caption{Car seat placement task in 3D.
		The car environment is represented as a point cloud (see the white dots),
		and the car seat is modeled as a union of ten padded polyhedra.
		The pink sphere depicts one padding sphere.
	}
	\label{fig:car_seat_rviz}
\end{figure}

\paragraph{Problem formulation}
Given start and target joint configurations, the goal is to find a trajectory that minimizes the sum of squared joint velocities.
In addition to this cost, we include in the objective function a small cost on the difference between the joint configuration at each time step $k$ and the target configuration to improve numerical stability.
Let $q_{\mathrm{tgt}}$ denote the target configuration.
The stage cost is given by
\begin{equation}
    L_k(x_k, u_k) := \left\|q_k - q_{\mathrm{tgt}}\right\|_{Q}^2 +  \left\|\dot{q}_k\right\|_{R}^2,
\end{equation}
with $Q, R \in \mathbb{S}^{8}_{+}$ being positive semi-definite.
The terminal equality constraint $q_N = q_{\mathrm{tgt}}$ ensures that the target is reached.
Collision avoidance between the seat and the car is ensured by imposing an infinite number of collision-avoidance constraints~\eqref{eq:nominal-ocp-collision-constraints} per polyhedron of the seat.
Affine constraints that limit the joint velocities are also enforced.

\begin{figure*}[t]
	\centering
	\captionsetup{skip=-6pt}
	\begin{minipage}[t]{.32\textwidth}
		\centering
		\includegraphics[width=\linewidth, height=0.15\textheight]{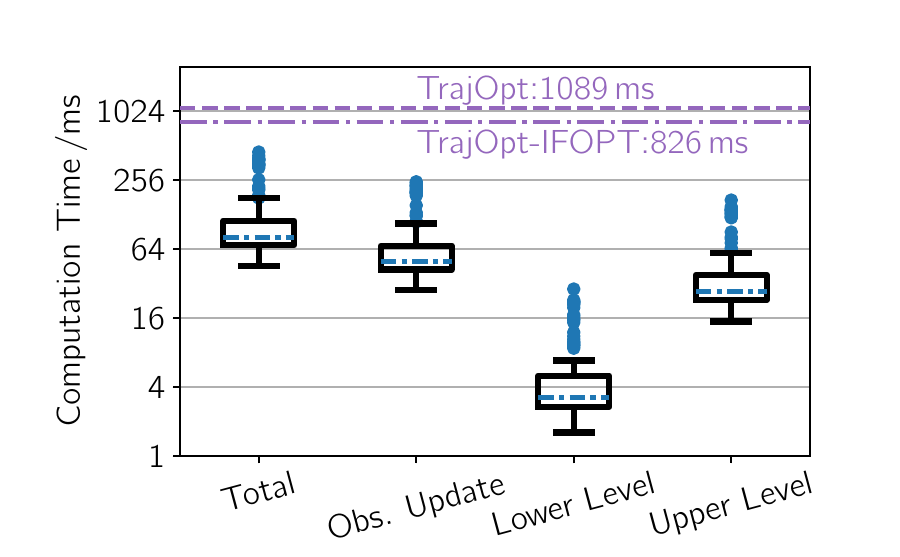}
		\caption{Open-loop planning for the car-seat placement: Computation time. The blue dashdotted lines are the median.
		}
		\label{fig:sia20d_car_seat_planner_computation_time}
	\end{minipage}
    \hfill
	\begin{minipage}[t]{.32\textwidth}
		\centering
		\includegraphics[width=\linewidth, height=0.15\textheight]{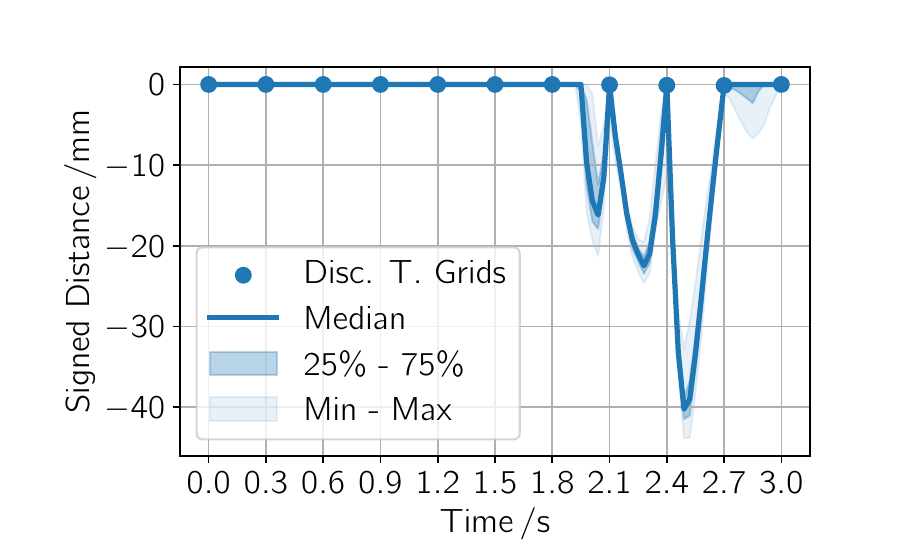}
		\caption{Open-loop planning: Signed distance of the seat to the car environment.
		Distances greater than zero are not computed and are plotted as zeros.}
		\label{fig:sia20d_car_seat_planner_signed_distance}
	\end{minipage}
    \hfill
    \begin{minipage}[t]{.32\textwidth}
    	\centering
    	\includegraphics[width=\linewidth]{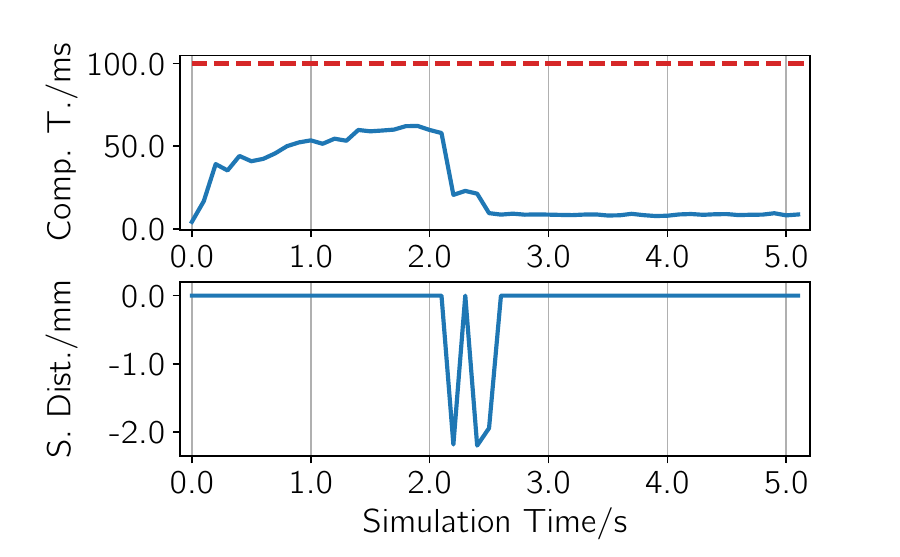}
    	\caption{MPC for the car-seat placement: computation time and signed distance to the car (evaluated at 10\,Hz as the control sampling frequency).}
    	\label{fig:sia20d_car_seat_mpc}
    \end{minipage}
\end{figure*}

\paragraph{Software implementation}
The Coal library~\cite{coalweb} is used for \mbox{\textsc{UpdateObsSubset}} in Algorithm~\ref{alg:nominal-collision-free-trajectory}.
It detects the point obstacles whose distances to the polyhedra are not greater than the padding radius.
For each polyhedron, the point obstacle with the minimum signed distance to the polyhedron is added to the obstacle subset $\mathcal{O}_{\mathrm{s}, k}$.
The cardinality of the subset is limited to 30.
If this limit is exceeded,  the obstacle corresponding to the minimum value of the linearized constraint $\check{h}_{\text{coll}}\left(\cdot \right)$ is removed.
Pinocchio~\cite{Carpentier2019} computes the forward kinematics and Jacobians.
The point cloud of the car is managed by OctoMap~\cite{Hornung2013} with a resolution of 0.015\,m.
The upper-level problem is solved with acados~\cite{Verschueren2021}, initialized by linear interpolation between start and target configurations.

\paragraph{Computation time and comparison to TrajOpt}
For evaluation of computation time, we vary the target seat pose and compute the corresponding joint configurations via inverse kinematics.
A total of 194 test cases are evaluated,
    and the results are reported in Fig.~\ref{fig:sia20d_car_seat_planner_computation_time}.
The median computation time is approximately 81\,ms, comprising 50\,ms for updating the obstacle subsets, 3\,ms for finding the lower-level maximizers, and 27\,ms for solving the upper-level \gls{nlp} subproblems.
To benchmark numerical performance, we compare the proposed method with TrajOpt~\cite{Schulman2014}.
For a fair comparison, we use the same number of discretization intervals,
    consider only the discrete-time collision-avoidance constraints,
    and include only the collision avoidance between the car and the seat,
    i.e., collision concerns between the robot joints and between the seat and the joints are neglected.
For the test case provided in the open-source code, the computation times for TrajOpt and TrajOpt-IFOPT are 1089\,ms and 826\,ms, respectively.

\paragraph{Collision-avoidance satisfaction at finer grids}
We evaluate the satisfaction of the collision-avoidance conditions on a grid that is ten times finer than the grid on which the constraints are enforced.
As shown in Fig.~\ref{fig:sia20d_car_seat_planner_signed_distance}, the maximum constraint violation is 48\,mm.
Distances greater than zero are not computed and are plotted as zeros.

\subsection{Model Predictive Control}
While the open-loop planning provides a trajectory for the robot arm to follow,
    real-world discrepancies (e.g., car pose or car configuration mismatches) require online adaptation.
The robot can use LiDAR to sense the environment and adapt its control inputs in an \gls{mpc} framework.
Compared to open-loop planning,
    the \gls{mpc} focuses on a shorter prediction horizon and employs a finer time grid to achieve improved collision-avoidance performance.

\paragraph{Problem formulation}
The system state of the~\gls{ocp} consists of the robot joint configuration~$q$ and the joint velocity~$\dot{q}$,
and the system input is the acceleration~$\ddot{q}$:
\begin{equation}
    x = \begin{bmatrix}
        q \\ \dot{q}
    \end{bmatrix} \in \bbR^{16},
    u = \begin{bmatrix}
        \ddot{q}
    \end{bmatrix} \in \bbR^{8},
    \dot{x} = \begin{bmatrix}
        \dot{q} \\ \ddot{q}
    \end{bmatrix}.
\end{equation}
The objective function of the \gls{ocp} is to track a reference trajectory defined in the joint space.
Denote the reference state trajectory by
    $x_{\mathrm{ref}, 0}, \cdots, x_{\mathrm{ref}, N}$.
The stage and terminal costs of the \gls{ocp} are given by
\begin{equation}
    \begin{split}
        L_k(x_k, u_k) &:= \left\|x_k - x_{\mathrm{ref}, k}\right\|_{Q}^2 +  \left\|u_k - u_{\mathrm{ref}, k}\right\|_{R}^2, \\
        L_N(x_N) &:= \left\|x_N - x_{\mathrm{ref}, N}\right\|_{Q_N}^2,
    \end{split}
\end{equation}
with $Q\in\mathbb{S}^{16}_{+}$, $R\in\mathbb{S}^{8}_{+}$, and $Q_N\in\mathbb{S}^{16}_{+}$ being positive semi-definite.
In addition to the collision-avoidance and velocity constraints enforced in the open-loop planning problem,
    affine constraints on the joint accelerations are imposed.

\paragraph{Implementation details}
The software implementation follows that of the open-loop planning problem.
To reduce computation time, we crop the point cloud to exclude the front and trunk of the car.
This also reflects the real-world scenario in which a LiDAR sensor mounted on the robot end effector is unlikely to perceive the entire environment.
The maximum number of iterations is set to eight, and both the control sampling and discretization intervals are 100\,ms.

\paragraph{Results}
The top of Fig.~\ref{fig:sia20d_car_seat_mpc} shows that the computation time for solving the \glspl{ocp} remains below the control sampling interval.
The computation time gradually increases and reaches its peak at approximately 62\,ms when the seat passes through the car door.
Similar to the open-loop planning, signed distances greater than zero are not computed and are plotted as zeros.
The signed distance, evaluated only at the control sampling frequency, has a minimum value of \mbox{-2.2\,mm}.

%% file: sections/appendix-parameters.tex
\vspace{-0.5\baselineskip}
\appendix[Parameter Settings and Robot Configurations]
\newcolumntype{Y}{>{\raggedleft\arraybackslash}X}

\vspace{-\baselineskip}
\begin{table}[h]
	\centering
	\caption{Configurations of the experiments with the mobile robot}
	\label{table:mobile-robot-config-split}
	\vspace{-3pt}
	\begin{tabularx}{\linewidth}{@{} >{\hsize=1.9\hsize}X >{\hsize=.35\hsize}Y >{\hsize=0.6\hsize}Y >{\hsize=1.15\hsize}Y @{}}
		\hline \hline
		Name & Unit & Symbol & Value \\
		\hline \hline
		\multicolumn{3}{@{}l}{\textbf{Robot Configuration}} &\\
		\hline
		polygon vertices & m & & (-0.18, $\pm$ 0.11)  \\
		&  & & (0.45, $\pm$ 0.11)\\
		% &  & & (0.45,  0.11)\\
		% &  & & (-0.18,  0.11)\\
		&  & & (-0.33,  0.00)\\
		polygon padding & m & $r_{\mathrm{shp}}$ & 0.20\\
		\multicolumn{3}{@{}l}{max. eigval. of $W_k$ for a 50\,ms disc. intvl.} & 2.5$\times 10^{\shortminus 4}$\\
		\multicolumn{3}{@{}l}{scale of cov. of kinematic state for $k=0$} & 4$\times10^{\shortminus 5}$\\
		\hline \hline
		\multicolumn{3}{@{}l}{\textbf{Environment Configuration}} &\\
		\hline
		point cloud resolution & m & & 0.02\\
		\multicolumn{3}{@{}l}{\# point obs., L-Corridor} & 821\\
		\multicolumn{3}{@{}l}{\# point obs., Docking Station} & 806 \\
		\multicolumn{3}{@{}l}{\# point obs., S-Corridor} & 1594 \\
		\multicolumn{3}{@{}l}{\# point obs., Walkway} & 1830 \\
		\hline \hline
	\end{tabularx}
\end{table}

\noindent
\begin{minipage}{\linewidth}
	\centering

	\setlength{\tabcolsep}{3.85pt}
	\captionof{table}{Configurations of the robot velocity and acceleration limits}
	\label{table:varying-vel-acc-limits}
	\vspace{-3pt}
	\begin{tabular}{@{}l r r r r @{}}
		\hline \hline
		& max.\ $v$\,/ms${}^{\shortminus1}$ & max.\ $\omega$\,/s${}^{\shortminus1}$ & max.\ $a$\,/ms${}^{\shortminus2}$ & max.\ $\alpha$\,/s${}^{\shortminus2}$ \\
		\hline
		Slow & 0.8 & 0.8 & 1.2 & 1.2 \\
		Med. & 1.2 & 1.2 & 1.6 & 1.6 \\
		Fast & 1.6 & 1.6 & 2.0 & 2.0 \\
		\hline \hline
	\end{tabular}

	\vspace{1.5\baselineskip}

	\captionof{table}{Parameters used in the experiments with the mobile robot}
	\label{table:mobile-robot-params-split}
	\begin{tabularx}{\linewidth}{@{} >{\hsize=1.9\hsize}X >{\hsize=.35\hsize}Y >{\hsize=0.6\hsize}Y >{\hsize=1.15\hsize}Y @{}}
		\hline \hline
		Name & Unit & Symbol & Value \\
		\hline \hline
		$\epsilon_{\text{inside}}$ in Algorithm~\ref{alg:identify-min-dist-obs} & m & $\epsilon_{\text{inside}}$ & 0.03\\
		$\discSet$ grid size in~\eqref{eq:min-sd-grid-search} & m & &0.016\\
		\multicolumn{3}{@{}l}{scaling of unc. matrices for robustification} & 1.0 \\
		\hline \hline
		\multicolumn{3}{@{}l}{\textbf{Numerical Evaluation}} & \\
		\hline
		\# disc. intervals & & N & 30 \\
		disc. interval & ms & & 200 \\
		\multicolumn{3}{@{}l}{max. cardinality of $\mathcal{O}_{\mathrm{s}, k}$} & 25 \\
		\multicolumn{3}{@{}l}{max. \# upper-level iters.} & 100 \\
		cvrg. criteria & & $\epsilon_{\text{cvg}}$ & $10^{\shortminus 6}$ \\
		\hline \hline
		\multicolumn{3}{@{}l}{\textbf{Real-robot MPC Experiments}} & \\
		\hline
		\multicolumn{2}{@{}l}{\# disc. intervals (non-uniform)} & N & 20 \\
		sampling rate & ms &  & 50 \\
		prediction horizon & ms & & 2400 \\
		\multicolumn{3}{@{}l}{max. cardinality of $\mathcal{O}_{\mathrm{s}, k}$} & 6 \\
		\multicolumn{3}{@{}l}{max. \# upper-level iters. (robust)} & 5 \\
		\multicolumn{3}{@{}l}{max. \# upper-level iters. (nominal)} & 6 \\
		cvrg. criteria & & $\epsilon_{\text{cvg}}$ & $10^{\shortminus 5}$ \\
		\hline \hline
	\end{tabularx}

	\vspace{1.5\baselineskip}

	\captionof{table}{Parameters used in the car seat placement task}
	\label{table:car-seat-params}
	\begin{tabularx}{\linewidth}{@{} >{\hsize=1.9\hsize}X >{\hsize=.35\hsize}Y >{\hsize=0.6\hsize}Y >{\hsize=1.15\hsize}Y @{}}
		\hline \hline
		Name & Unit & Symbol & Value \\
		\hline \hline
		\textbf{Open-Loop Planning} & & &\\
		\hline
		point cloud resolution & m && 0.015 \\
		\# octomap nodes &&& 491171\\
		\# disc. intervals & & N & 10 \\
		disc. interval & ms & & 300 \\
		max. cardinality of $\mathcal{O}_{\mathrm{s}, k}$ & & & 30 \\
		max. \# upper-level iters. & & & 100 \\
		cvrg. criteria & & $\epsilon_{\text{cvg}}$ & $10^{\shortminus 6}$ \\
		\hline \hline
		\textbf{MPC in Simulation} & & &\\
		\hline
		point cloud resolution & m && 0.015 \\
		\# octomap nodes &&& 311442\\
		\# disc. intervals & & N & 20 \\
		sampling rate & ms & & 100 \\
		max. cardinality of $\mathcal{O}_{\mathrm{s}, k}$ & & & 30 \\
		max. \# upper-level iters. & & & 10 \\
		cvrg. criteria & & $\epsilon_{\text{cvg}}$ & $10^{\shortminus 5}$ \\
		\hline \hline
	\end{tabularx}
\end{minipage}